\documentclass[12pt]{article}

\oddsidemargin
-0.2in \topmargin -.60in
\usepackage{microtype}
\usepackage{graphicx}
\usepackage{booktabs} % for professional tables

\usepackage{hyperref}

\usepackage{amsmath}
\usepackage{amssymb}
\usepackage{mathtools}
\usepackage{amsthm}
\usepackage[sort&compress,numbers]{natbib}

\usepackage[capitalize,noabbrev]{cleveref}

\theoremstyle{plain}
\newtheorem{theorem}{Theorem}[section]

\newtheorem{lemma}[theorem]{Lemma}

\theoremstyle{definition}
\newtheorem{definition}[theorem]{Definition}

\theoremstyle{remark}

\usepackage{caption}
\usepackage{subcaption}

\usepackage{algorithm}
\usepackage{algorithmic}

\usepackage{authblk}

\usepackage{enumitem}
\newlength{\myvspace}
\usepackage{makecell,diagbox,rotating}
\usepackage{threeparttable}
\usepackage{threeparttablex}
\usepackage{colortbl} % For coloring rows

\usepackage{pifont}
\newcommand{\cmark}{\ding{51}}  % ✓
\newcommand{\xmark}{\ding{55}}  % ✗

\usepackage[table]{xcolor}

\definecolor{tableblue}{HTML}{D0E1F9}        % deep
\definecolor{tablemidblue}{HTML}{E0EDFF}     % mid (added)
\definecolor{tablelightblue}{HTML}{F0F8FF}   % light

\definecolor{tablegreen}{HTML}{C8E6C9}       % deep
\definecolor{tablemidgreen}{HTML}{D9F0DA}    % mid (added)
\definecolor{tablelightgreen}{HTML}{E8F5E9}  % light

\definecolor{tablered}{HTML}{FFCDD2}         % deep
\definecolor{tablemidred}{HTML}{FFE0E3}      % mid (added)
\definecolor{tablelightred}{HTML}{FFEBEE}    % light

\definecolor{lightgraytext}{gray}{0.55} 

\usepackage{multirow}

\begin{document}

\title{Adaptively Point-weighting Curriculum Learning}

\author[1]{Wensheng Li}
\author[1]{Yichao Tian}
\author[2]{Hao Wang}
\author[1]{Ruifeng Zhou}
\author[1]{Hanting Guan}
\author[1]{Chao Zhang\thanks{Corresponding author.}}
\author[3]{Dacheng Tao}

\affil[1]{School of Mathematical Sciences, Dalian University of Technology, Dalian 116024, China}
\affil[ ]{\texttt{lws@mail.dlut.edu.cn, 13223541137@mail.dlut.edu.cn, zhouruifeng@mail.dlut.edu.cn, miracie@mail.dlut.edu.cn, chao.zhang@dlut.edu.cn}}
\affil[2]{School of Environmental Science and Technology, Dalian University of Technology, Dalian 116024, China}
\affil[ ]{\texttt{wanghao\_dut@mail.dlut.edu.cn}}
\affil[3]{College of Computing $\&$ Data Science, Nanyang Technological University, 639798, Singapore}
\affil[ ]{\texttt{dacheng.tao@gmail.com}}

\maketitle

\begin{abstract}
Curriculum learning (CL) mimics human learning, in which easy samples are learned first, followed by harder samples, and has become an effective method for training deep networks. However, many existing automatic CL methods maintain a preference for easy samples during the entire training process regardless of the constantly evolving training state. This is just like a human curriculum that fails to provide individualized instruction, which can delay learning progress. To address this issue, we propose an adaptively point-weighting (APW) curriculum learning method that assigns a weight to each training sample based on its training loss. The weighting strategy of APW follows the easy-to-hard training paradigm, guided by the current training state of the network. We present a theoretical analysis of APW, including training effectiveness, training stability, and generalization performance. Experimental results validate these theoretical findings and demonstrate the superiority of the proposed APW method.
\end{abstract}

\section{Introduction}
The concept of curriculum learning (CL) is inspired by the human education process, which starts with simple concepts and then gradually introduces more complicated concepts as the learner's ability improves. A desired curriculum should be designed according to the easy-to-hard training paradigm, which starts with easy samples to form the current training data subset and then gradually adds hard samples into the subset until all training samples are included~\cite{bengio2009curriculum,elman1993learning}.

As flexible plug-and-play submodules, CL methods have been widely applied to train deep networks in various real-world scenarios, such as image classification~\cite{weinshall2018curriculum,croitoru2024learning}, natural language processing (NLP)~\cite{xu2020curriculum,mejeh2024taking}, and graph machine learning~\cite{wang2021curgraph,li2023curriculum}. The recent CL benchmark (CurBench)~\cite{zhou2024curbench} evaluates the performance of CL methods across these research directions and has pointed out that a well-designed curriculum can accelerate training convergence and improve the generalization performance of the networks~\cite{gong2016curriculum,weinshall2018curriculum}. However, most of the existing automatic CL methods are designed with a fixed weighting preference, {\it i.e.,} they favor either easy samples or hard samples throughout training and thus do not implement the easy-to-hard training paradigm in an adaptive manner.

In general, there are two main components in developing CL methods~\cite{wang2021survey}: 1) the difficulty measurer that evaluates the difficulty of training samples, for example, it introduces an error threshold to decide whether a training sample is easy or hard: a sample is considered easy if its loss is no larger than the threshold, and hard otherwise; 2) the training scheduler designs the schedule for learning samples from easy to hard. These two components allow a CL method to adjust sample weights to guide the network to focus more on easy samples in the early training phase and then gradually shift attention toward hard samples in the later training phase.

The recent curriculum learning library, CurML~\cite{zhou2022curml}, implements most of the existing CL methods and evaluates their performance on image classification tasks. CurBench~\cite{zhou2024curbench} further evaluates these methods on a broader range of classification tasks beyond image classification. The existing CL methods can be broadly categorized into two groups, as summarized in CurBench, based on whether the sample weights are 1) hard weights, such as self-paced learning and its variants~\cite{kumar2010self,jiang2015self} and teacher transferring~\cite{weinshall2018curriculum,matiisen2019teacher}, or 2) soft weights, such as meta-learning~\cite{ren2018learning,shu2019meta,wang2020optimizing} and confidence-aware losses~\cite{saxena2019data,castells2020superloss}. Beyond these, other CL methods design curricula from different perspectives, such as model-level curricula~\citet{sinha2020curriculum} and intra-sample curricula~\cite{wang2023efficienttrain,jarca2024cbm}.

We focus on the CL methods that assign soft sample weights, which are referred to as loss-reweighting methods in CurBench~\cite{zhou2024curbench}. Such methods assign a soft weight to each training sample and replace the average loss with a reweighted loss (see Figure~\ref{fig:APW_structure}).

\begin{figure}[htbp]
\centering
\includegraphics[width=0.65\textwidth]{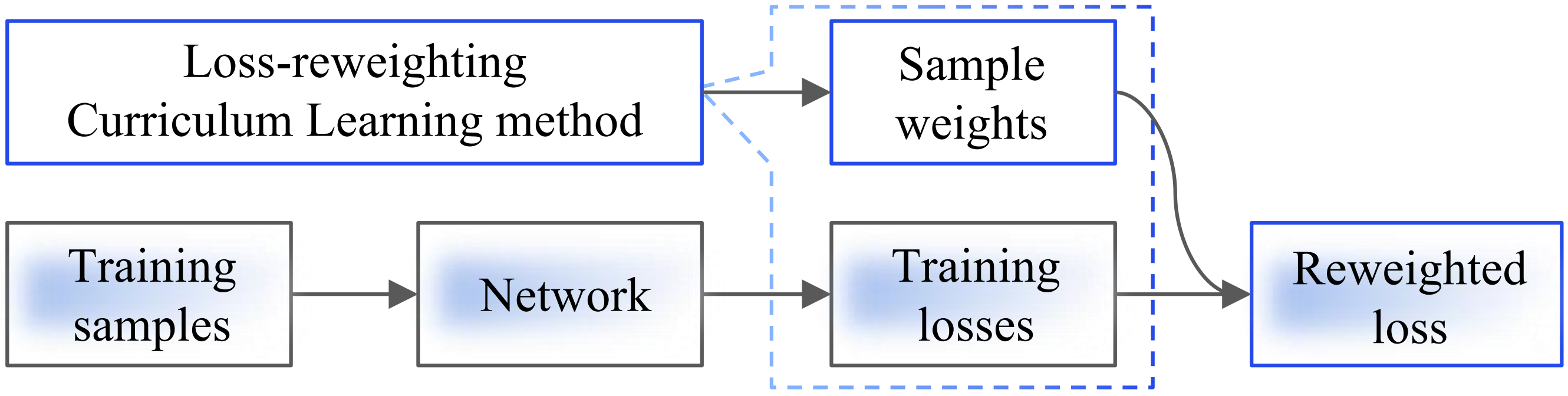}
\caption{An illustration of the loss-reweighting CL methods.}
\label{fig:APW_structure}
\end{figure}

There are two limitations of the existing loss-reweighting methods. 
\begin{enumerate}[label=(Lim-\arabic*),ref=Lim-\arabic*, leftmargin=*, labelindent=0pt]
\item \label{lim:L1} Their difficulty measurers usually only consider training samples in the current mini-batch, without taking the entire training set into account. This issue limits their applicability; for example, the sample weights cannot be directly used as sampling probabilities over the entire dataset.

\item \label{lim:L2} The existing works commonly adopt a fixed weighting preference throughout training, namely, assigning larger weights to easy samples than to hard ones or the opposite, but do not adapt the training scheduler to the constantly evolving training state of the network ({\it e.g.,}~\cite{kim2018screenernet, shu2019meta,castells2020superloss}). Consequently, the training scheduler does not fully follow the easy-to-hard training paradigm.
\end{enumerate}

In this study, we propose an adaptively point-weighting (APW) curriculum learning method, which is an automatic CL algorithm and assigns a soft weight to each training sample based on its per-sample loss and the current training state of the network. Thus, APW addresses Limitation~\ref{lim:L1}. Inspired by the latest evidence of the learning dynamics of networks~\cite{zhou2025new}, APW can adaptively detect the two training phases and assign sample weights following the easy-to-hard training paradigm. Thus, APW provides a solution to Limitation~\ref{lim:L2}. Specifically, in the early training phase, the APW method gradually increases the weights of easy samples to mitigate the negative impact of hard samples; in the later training phase, it gradually increases the weights of hard samples to further improve predictive performance. We primarily focus on image classification, and we also demonstrate that APW can be applied to NLP and graph machine learning.

The proposed APW method makes the following contributions. 1) APW is a novel and self-contained loss-reweighting method. To the best of our knowledge, it is the first loss-reweighting method to adaptively enforce the easy-to-hard training paradigm in CL. Its weighting strategy accounts for the entire training set, and thus APW is broadly applicable in practice. Accordingly, we present two APW variants to demonstrate its applicability. 2) APW involves only two hyperparameters, which are among the fewest in CL methods, and we further show that these hyperparameters can be set in an interpretable way. The computation of sample weights in APW relies only on elementary operations (addition, subtraction, multiplication, division, logarithm, and exponentiation) and introduces no additional learnable parameters. 3) The theoretical analysis suggests that APW improves prediction confidence on easy samples in the training set, and then improves prediction confidence on easy samples in the test set, which is crucial for real-world datasets~\cite{guo2017calibration,angelopoulos2022conformal} and has not been explicitly analyzed in prior CL work.

\section{Adaptively Point-weighting (APW) Curriculum Learning}\label{section:APW}

In this section, we introduce APW under the two-component framework for automatic CL methods~\cite{wang2021survey}, then elaborate on its practical implementations, and finally conclude with a discussion of its capabilities.

\subsection{The Difficulty Measurer}\label{section:APW:workflow:difficulty_measurer}

Denote the $n$-th training sample by $\mathbf{z}^{(n)} = (\mathbf{x}^{(n)}, y^{(n)})$ and its loss by $\mathrm{L}^{(n)}$ ({\it e.g.,} the cross-entropy loss in a classification task). Let $\mathcal{Z} = \{\mathbf{z}^{(n)}\}_{n=1}^{N}$ denote the entire training set, and initialize the sample-weight vector as $\mathbf{w}_0 = \{w_{0}^{(n)}\}_{n=1}^{N}$ with $w_{0}^{(n)} = \frac{1}{N}$. At the beginning of the $k$-th epoch of applying APW, the difficulty level of every training sample is evaluated in the following way:
\vspace{-0.05cm}
\begin{equation}\label{eq:beta_k}
\beta_k(\mathbf{z}^{(n)};e)=\begin{cases}+1,&\mathrm{L}_{k}^{(n)}\leq e;\\-1,
&\mathrm{L}_{k}^{(n)}>e.\end{cases}
\vspace{-0.05cm}
\end{equation}
where $\mathrm{L}_{k}^{(n)}$ is the training loss evaluated at the current network parameters (weights and biases) $\Theta_{k}$. Given an error threshold $e$, the sample $\mathbf{z}^{(n)}$ with $\beta_k(\mathbf{z}^{(n)};e)=+1$ (resp. $\beta_k(\mathbf{z}^{(n)};e)=-1$) is marked as the easy (resp. hard) sample. Many loss-reweighting methods ({\it e.g.,}~\cite{kim2018screenernet,castells2020superloss}) also employ a binary criterion similar to Eq.~\eqref{eq:beta_k}. Next, define 
\vspace{-0.05cm}
\begin{equation}\label{eq:rho_k}
\rho_{k}=\sum_{\{n:\beta_k(\mathbf{z}^{(n)};e)=-1\}}w_{k-1}^{(n)}
\vspace{-0.05cm}
\end{equation}
as the sum of hard sample weights at the $(k-1)$-th epoch. The value $\rho_k$ serves as an indicator of the remaining training difficulty for the current network. Namely, if the network can efficiently learn the training set, the value of $\rho_k$ should exhibit a decreasing trend. For a network trained with APW, according to whether $\rho_k$ exceeds $\frac{1}{2}$, the training process is split into two phases: the epochs with $\rho_{k} > \frac{1}{2}$ are said to be in the early training phase, whereas the epochs with $\rho_{k} < \frac{1}{2}$ are said to be in the later training phase.

\subsection{The Training Scheduler}\label{section:APW:workflow:training_scheduler}

To follow the easy-to-hard training paradigm, we design the training scheduler such that, in the early training phase, easy samples are assigned larger weights than hard samples to avoid disrupting initial learning; in the later training phase, the weights of easy samples gradually decrease while those of hard samples gradually increase to further improve performance. To this end, we devise a weight-change value $\alpha_k$:
\begin{equation}\label{eq:alpha_k}
\alpha_k = \frac{1}{q} \log\left(\frac{(1 - \rho_k)}{\rho_k}\right),
\end{equation}
where $q \geq 2$ is chosen to be large to stabilize the process of updating sample weights. This formula draws inspiration from AdaBoost~\cite{freund1997decision}. In practice, we clip $\rho_k$ to $[10^{-4}, 1-10^{-4}]$ so that $\alpha_k$ remains bounded and the magnitude of weight updates is controllable. At the $k$-th epoch, the weight of the sample $\mathbf{z}^{(n)}$ is updated as follows:
\begin{equation}\label{eq:weight_update}
w_k^{(n)}=\frac{w_{k-1}^{(n)}\exp\left(-\alpha_k\,\beta_k\left(\mathbf{z}^{(n)};e\right)\right)}{Z_k},
\end{equation}
where $w_{k-1}^{(n)}$ comes from the $(k-1)$-th epoch, and
\begin{equation*}
Z_k=\sum_{n=1}^Nw_{k-1}^{(n)}\exp(-\alpha_k\,\beta_k(\mathbf{z}^{(n)};e))
\end{equation*}
ensures that the sample-weight vector of APW sums to one. Finally, the network parameters are updated by minimizing the reweighted loss function:
\begin{equation}\label{eq:APW_loss}
\mathrm{L}_k^\mathrm{APW}=\sum_{n=1}^{N}w_k^{(n)}\mathrm{L}_k^{(n)}.
\end{equation}

We summarize the relationship between $w_k^{(n)}$ and $w_{k-1}^{(n)}$ in the early and later training phases in Table~\ref{tab:relation}. In the difficulty measurer of APW, given an error threshold $e$, the transition between the two training phases is determined by comparing $\rho_k$ with $\frac{1}{2}$. The threshold $e$ is set in an interpretable manner based on the estimated label-noise rate of the real-world dataset ({\it cf.} Section~\ref{section:experiments}) and we use $\rho_k$ relative to $\frac{1}{2}$ to develop the theoretical analysis ({\it cf.} Section~\ref{section:theoretical_analysis}). As further shown in Appendix~\ref{appendix:APW}, extending the default value $\frac{1}{2}$ to other choices typically leads to worse performance.

\begin{table}[htbp]
\centering   
\caption{Weight update in the early and later training phases.}\label{tab:relation}
\resizebox{0.7\linewidth}{!}{
\begin{tabular}{c|c|c}
\toprule[1.2pt]
Training phase & Early training phase ($\rho_k > \frac{1}{2}$) & Later training phase ($\rho_k < \frac{1}{2}$) \\
\midrule
Sign of $\alpha_k$ & $\alpha_k < 0$ & $\alpha_k > 0$ \\
\midrule
\makecell{Weight update \\ (unnormalized)} 
& \makecell{$\beta_{k}(\mathbf{z}^{(n)};e) = +1$: $w_k^{(n)} > w_{k-1}^{(n)}$ \\ $\beta_{k}(\mathbf{z}^{(n)};e) = -1$: $w_k^{(n)} < w_{k-1}^{(n)}$} 
& \makecell{ $\beta_{k}(\mathbf{z}^{(n)};e) = +1$: $w_k^{(n)} < w_{k-1}^{(n)}$ \\  $\beta_{k}(\mathbf{z}^{(n)};e) = -1$: $w_k^{(n)} > w_{k-1}^{(n)}$} \\
\bottomrule[1.2pt]
\end{tabular}
}
\end{table}

The APW method belongs to the group of loss-reweighting CL methods summarized in CurBench~\cite{zhou2024curbench}. In Appendix~\ref{appendix:APW}, we provide an analysis of the weighting behavior of APW and visualize the weighting dynamics.

\subsection{Implementations of APW}\label{section:APW:implementations}

Since the APW method assigns sample weights based on the entire training set, we introduce two variants of APW: sampling-APW (S-APW) and mixup-APW (M-APW). S-APW takes the APW-assigned weight of every training sample as its sampling probability. As training progresses, the sampled set becomes increasingly difficult, thus implementing an easy-to-hard curriculum. M-APW incorporates mixup~\cite{zhang2017mixup} into APW by replacing the standard mixing coefficient with the APW-assigned weights of the two samples (normalized to sum to one). In this manner, M-APW combines APW's CL weighting strategy with mixup's robustness to label noise. The details of these two variants can be found in Appendices~\ref{appendix:implementation:S-APW} and~\ref{appendix:implementation:M-APW}.

Furthermore, we consider three weighting approaches, epoch-level weighting (E-weighting), iteration-level weighting (I-weighting), and combined epoch- and iteration-level weighting (EI-weighting), whose detailed algorithms are provided in Appendix~\ref{appendix:implementation:approaches}. In E-weighting, our default weighting approach, all sample weights are updated at the beginning of each epoch ({\it cf.} Section~\ref{section:APW:workflow:training_scheduler}). In I-weighting, only the weights of the samples in the current mini-batch are updated at each iteration. EI-weighting combines both E-weighting and I-weighting. In particular, S-APW only uses E-weighting during curriculum sampling when gradually increasing the difficulty of the current training subset. After all training samples have been included, S-APW can then adopt one of the three weighting approaches.

\subsection{Discussion of APW}\label{section:APW:discussion}

A desired CL method should have the following two characteristics~\cite{wang2021survey}: one is the guidance, {\it i.e.,} to encourage the network to learn from the simple sub-tasks to the overall task~\cite{matiisen2019teacher,shu2019transferable}, and the other is the denoising, {\it i.e.,} to avoid overfitting to label noise~\cite{geng2023denoising,kim2024denoising}. The proposed APW method meets the two requirements. 

{\bf The guidance.} The weighting strategy of APW uses $\rho_{k}$ to indicate the training state of the network and follows the easy-to-hard training paradigm. Specifically, in the early training phase, since $\rho_{k}$ is large, the easy samples have relatively large weights and thus the network pays more attention to simple sub-tasks. In the later training phase, since $\rho_{k}$ gradually becomes small, the weights of hard samples are gradually increased to balance the network's attention among all training samples, rather than focusing only on the easy samples.

{\bf The denoising.} A common assumption in CL is that hard samples are more likely to contain label noise, and this assumption has been widely corroborated~\cite{wang2021survey}. Moreover, the latest evidence on neural network learning dynamics shows that the training process can be divided into two phases~\cite{zhou2025new}. In the early training phase, the network is highly sensitive, and thus hard samples can cause irreversible damage to the network parameters~\cite{achille2017critical}. In contrast, the later training phase is more stable and exhibits strong robustness to hard samples~\cite{zhou2025new}. Therefore, to mitigate the premature influence of noisy information on model performance, in the early training phase, APW places more emphasis on the easy samples, and in the later training phase, APW gradually shifts the emphasis toward the hard samples.

The weighting strategy of APW is simple and lightweight, as it introduces no extra modules, does not require an unbiased meta-data set, and does not introduce any additional regularization terms. At the end of each epoch, per-sample training losses are typically computed via a forward evaluation pass for logging and checkpoint selection, and APW simply reuses these losses to mark a training sample as easy or hard. Moreover, the weighting strategy of APW requires only a few elementary operations for each training sample. As a result, the additional computational overhead is negligible and in fact smaller than that of a tiny fully connected network with a single hidden layer of 100 neurons (1D input and 1D output)~\cite{shu2019meta}. Therefore, the APW method has strong potential for practical applications.

\section{Theoretical Analysis of APW}\label{section:theoretical_analysis}

In this section, we provide a theoretical analysis of APW under the default E-weighting approach. We first introduce several necessary preliminaries to assess the training behavior of the network trained with APW. Building upon these preliminaries, we then present theoretical results showing that APW increases E-Prop on the training set, maintains the easy status of training samples, and leads to an improvement in E-Prop on the test set, where E-Prop is defined as the proportion of easy samples in the considered dataset.

\subsection{Preliminaries}

The difficulty measurer of APW introduces an error threshold $e$ and treats $\rho_k$ as a criterion for the remaining training difficulty at the $k$-th epoch. Since in practice $\rho_k$ is clipped, we have $0<\rho_k<1$ for all $k$ in our theoretical analysis. Compared to empirical risk minimization, using the reweighted loss of APW as the training objective can be interpreted as completing an $e$-level task, with progress quantified by E-Prop on the training set. If the network can reliably accomplish the $e$-level task, we state the following definition to characterize a well-behaved convergence phase toward the end of training.

\begin{definition}\label{definition:K_convergence}
Starting from the $(K+1)$-th epoch, if $\rho_{k} < \frac{1}{2}$ holds for all $k = K+1,\dots,K+\Delta_K$, then the phase from the $(K+1)$-th to the $(K+\Delta_K)$-th epoch is said to be a $\Delta_K$-convergence phase.
\end{definition}

Since $\alpha_k$ is a strictly decreasing function of $\rho_k$ ({\it cf.} Eq.~\eqref{eq:alpha_k}), we introduce $A_K := \sum_{k=1}^K \alpha_k$ to summarize the cumulative training state. $A_K>0$ indicates that the network predominantly stays in the later training phase up to the $K$-th epoch, and a larger $A_K$ indicates that the network can handle the $e$-level task. 

\begin{definition}
According to Definition~\ref{definition:K_convergence}, if $A_K>0$, then the subsequent $\Delta_K$-convergence phase is said to be a $\Delta_K^{+}$-convergence phase.
\end{definition}

Since the weighting strategy of APW is iterative ({\it cf.} Eq.~\eqref{eq:weight_update}), it is reasonable to examine the cumulative easy status of a training sample ({\it cf.} Section~\ref{section:APW:discussion}). For a training sample $\mathbf{z}^{(n)}$, we define
\begin{equation}\label{eq:g_K}
g_K(\mathbf{z}^{(n)}) := \sum_{k=1}^K\alpha_k\,\beta_k(\mathbf{z}^{(n)};e),
\end{equation}
as its cumulative training performance from the $1$-st to the $K$-th epoch. 
\begin{definition}
If $g_{K}(\mathbf{z}^{(n)})>0$, we call $\mathbf{z}^{(n)}$ a $K$-retentive sample. We further define $m_K(\mathbf{z}^{(n)}) := g_{K}(\mathbf{z}^{(n)}) / A_K$ as the relative cumulative training performance.
\end{definition}
Next, we present the following lemma to relate the network training state to the easy status of training samples.
\begin{lemma}\label{lemma:ineq:m_K}
During a $\Delta_K^{+}$-convergence phase, if a $K$-retentive sample $\mathbf{z}^{(n)}$ satisfies $m_{K+\Delta_K}(\mathbf{z}^{(n)})>m_{K}(\mathbf{z}^{(n)})>0$, then it holds that
\begin{equation}\label{ineq:m_K}
\frac{g_{K+\Delta_K}(\mathbf{z}^{(n)})-g_K(\mathbf{z}^{(n)})}{A_{K+\Delta_K}-A_K}>\frac{g_K(\mathbf{z}^{(n)})}{A_K}.
\end{equation}
\end{lemma}
Consequently, a larger increment $g_{K+\Delta_K}(\mathbf{z}^{(n)})-g_K(\mathbf{z}^{(n)})$ indicates that $\mathbf{z}^{(n)}$ is marked as easy more frequently during this training phase.

\subsection{Theoretical Analysis}

Based on the aforementioned preliminaries, we analyze the theoretical properties of APW. We first investigate whether the training difficulty $\rho_k$ tends to decrease as training progresses.
\begin{theorem}\label{theorem:rho_k}
For any $e>0$, it holds that
\begin{equation*}
\sum_{\{n:\beta_k(\mathbf{z}^{(n)};e)=-1\}} w_k^{(n)} < \frac{1}{e} \mathrm{L}_k^{\mathrm{APW}}.
\end{equation*}
Especially, if the hard samples remain the same from the $k$-th epoch to the $(k+1)$-th epoch, then $\rho_{k+1} < \frac{1}{e} \mathrm{L}_k^{\mathrm{APW}}$.
\end{theorem}
This result provides an upper bound on the training difficulty and suggests a tendency for $\rho_k$ to decrease as $\mathrm{L}_k^{\mathrm{APW}}$ is minimized. Namely, the APW method can facilitate learning the $e$-level task. Next, we analyze how APW encourages more training samples to become the $K$-retentive samples.
\begin{theorem}[Training Effectiveness]\label{theorem:training_effectiveness}
Let $q \geq  2$ and $\gamma_k = \frac{1}{2} - \rho_k$. Then,
\begin{align*}
\frac{1}{N} \sum_{n=1}^N \mathbf{I}(g_K(\mathbf{z}^{(n)}) \leq 0) 
&\leq \frac{1}{N} \sum_{n=1}^N \exp(-g_K(\mathbf{z}^{(n)})) \\
&= \prod_{k=1}^K Z_k 
\leq \exp\left(-\frac{4}{q} \sum_{k=1}^K \gamma_k^2\right),
\end{align*}
where $\mathbf{I}(\cdot)$ is the indicator function.
\end{theorem}
This theorem suggests that the training samples are converted into $K$-retentive samples at an exponential rate in $\sum_{k=1}^K \gamma_k^2/q$ as training progresses. The choice of $q$ is important for the convergence rate and in practice, we set $q$ to be equal to the number of training epochs by default. Moreover, if there exists a subsequent $\Delta_K$-convergence phase, the $K$-retentive samples tend to be marked as easy more frequently during this training phase. If this phase is further a $\Delta_K^{+}$-convergence phase (with $A_K>0$), then building on the conclusion of Theorem~\ref{theorem:training_effectiveness}, we provide a probabilistic analysis of how APW increases E-Prop on the training set.
\begin{theorem}[Training Stability]\label{theorem:training_stability}
Let $q \geq 2$ and $\gamma_k = \frac{1}{2} - \rho_k$. Then, under the condition of $ A_K > 0 $, for any $\theta>0$, it holds that
\begin{equation}\label{eq:P_m}
\mathbb{P}_{\mathbf{z} \sim \mathcal{Z}} \left[ m_K(\mathbf{z}) > \theta \right] \geq 1 - \prod_{k=1}^K \delta(\theta,\gamma_k),
\end{equation}
where $\delta(\theta,\gamma_k) := (1 - 2\gamma_k)^{\frac{1 - \theta}{q}} (1 + 2\gamma_k)^{\frac{1 + \theta}{q}} > 0$. Here, $\mathbf{z} \sim \mathcal{Z}$ denotes sampling uniformly at random from the training set $\mathcal{Z}$.
\end{theorem}
During the $\Delta_K^{+}$-convergence phase, Eq.~\eqref{eq:P_m} provides a lower bound on $\mathbb{P}_{\mathbf{z} \sim \mathcal{Z}} \left[ m_K(\mathbf{z})>\theta \right]$. From the $K$-th epoch to the $(K+\Delta_K)$-th epoch, the product term in Eq.~\eqref{eq:P_m} is updated by the factor
\begin{equation}\label{eq:delta_k}
\eta(\Delta_K) := \prod_{k=K+1}^{K+\Delta_K} \delta(\theta,\gamma_k).
\end{equation}
Under the setup of Theorem~\ref{theorem:training_stability} and the definition of a $\Delta_K^{+}$-convergence phase, we further restrict $\theta$ and provide the following lemma.
\begin{lemma}\label{lemma:theta}
For any $0<\theta \leq \gamma_k<\frac{1}{2}$ and any $k \in \{K+1,\dots,K+\Delta_K\}$, it holds that
\begin{equation*}
\delta(\theta, \gamma_k) \leq \delta(\gamma_k, \gamma_k) \leq \delta(\gamma_{\min}, \gamma_{\min}) < 1,
\end{equation*}
where $\gamma_{\min} := \min\{ \gamma_{K+1},\cdots,\gamma_{K +\Delta_K }\}$.
\end{lemma}
Thus, the factor $\eta(\Delta_K)$ is upper-bounded by an exponential term $\eta(\Delta_K) \leq {\delta(\gamma_{\min}, \gamma_{\min})}^{\Delta_K}$. Moreover, this lemma indicates that a larger $\gamma_{\min}$ permits a more reasonable $\theta \leq \gamma_{\min}$, and thus makes the lower bound in Eq.~\eqref{eq:P_m} more effective. In this case, the proportion of training samples with a large $m_{K+\Delta_K}(\mathbf{z})$ is expected to increase over this training phase, and thus Eq.~\eqref{ineq:m_K} is more likely to hold for more training samples, which indicates that the easy status is well maintained.

Finally, this naturally motivates the following question: for a network trained using the APW method, if most training samples are $K$-retentive and have large values of $m_{K+\Delta_K}(\mathbf{z}^{(n)})$ after a $\Delta_K^{+}$-convergence phase, we investigate whether these properties lead to a large E-Prop on the test set. We present the following theorem.
\begin{theorem}[Generalization Performance]\label{theorem:generalization_error}
Consider a $\Delta_K$-convergence phase. Let $\mathcal{Z}$ be a set of $N$ i.i.d. training samples $\mathbf{z}$ taken from the distribution $\mathcal{D}$ over the sample space. Define a function class $\mathcal{B}$ as:
\begin{equation*}
\mathcal{B} = \left\{ b_{\Delta_K} \mid b_{\Delta_K}(\mathbf{z}) = \frac{g_{K+\Delta_K}(\mathbf{z})-g_K(\mathbf{z})}{A_{K+\Delta_K}-A_K}\right\}, 
\end{equation*}
where the functions $\beta_k$ are determined by the network parameters. Then, for any $\varepsilon > 0$ and any $\theta > 0$, with probability at least
\begin{equation*}
1-2\mathcal{N}\left(\mathcal{B},\frac{\theta}{2},\frac{\varepsilon}{8},2N\right)
\exp\left(-\frac{\varepsilon^2N}{32}\right),
\end{equation*}
it holds that for any $b_{\Delta_K} \in \mathcal{B}$,
\begin{equation*}
\mathbb{P}_{\mathbf{z}\sim\mathcal{D}}\left[b_{\Delta_K}(\mathbf{z})\leq0\right]\leq\mathbb{P}_{\mathbf{z}\sim \mathcal{Z}}\left[b_{\Delta_K}(\mathbf{z})\leq\theta\right]+\varepsilon,
\end{equation*}
where $\mathcal{N}\left(\mathcal{B},{\theta/2},{\varepsilon/8},2N\right)$ is the $\theta/2$-covering number of $\mathcal{B}$ over all training sets of size $2N$ with $\varepsilon/8$-sloppy~\cite{661502}.
\end{theorem}
This theorem is adapted from \cite{bartlett1998boosting} and specialized to the APW formulation. Based on the discussion of Lemma~\ref{lemma:ineq:m_K} and Theorems~\ref{theorem:training_effectiveness} and \ref{theorem:training_stability}, APW can increase E-Prop on the training set and maintain this status well. Consequently, Theorem~\ref{theorem:generalization_error} shows that E-Prop on the test set is expected to increase accordingly. Reliable confidence estimates and effective risk/loss control are crucial for deployment in real-world applications~\cite{guo2017calibration,angelopoulos2022conformal}. The APW method increases E-Prop on the test set and thus offers strong practical value. In Appendix~\ref{appendix:APW}, we further provide intuitive examples to illustrate the core idea of APW and the theoretical findings.

\section{Experiments}\label{section:experiments}

In this section, we consider multiple datasets spanning image classification, NLP, and graph machine learning. Specifically, we use CIFAR-10/100~\cite{krizhevsky2009learning}, CIFAR-10N/100N~\cite{wei2021learning}, WebVision and Mini-WebVision~\cite{li2017WebVision} for image classification, RTE~\cite{wang2018glue} for NLP, and NCI1~\cite{morris2020tudataset} for graph machine learning. The detailed description of these datasets (including their label-noise ratios $p_{\text{noise}}$), as well as the network architectures and optimizer settings used in our experiments, can be found in Appendix~\ref{appendix:descriptions}.

{\bf Hyperparameter $q$ in APW.} 
The hyperparameters in the APW method are designed to be interpretable. Throughout this section, we set $q$ equal to the number of training epochs $K^\dagger$ (see the discussion of Theorem~\ref{theorem:training_effectiveness}). We further examine the sensitivity of APW to $q$ in Section~\ref{section:experiments:benchmark}.

{\bf Hyperparameter $e$ in APW.} 
Our default choice of $e$ is guided by the estimated label-noise rate $p_{\text{noise}}$ and is set in an interpretable manner. 1) For clean datasets ($p_{\text{noise}}=0$), we set $e = \log(2)$ and mark a training sample as easy only if the predicted probability of the correct class is at least 50\%. 2) For datasets that inherently contain label noise, we set $e = \log(2) + \log(1 - p_{\text{noise}})$, since fewer training samples can be considered easy samples in this case. 3) For datasets with synthetic label noise at rate $p_{\text{noise}}$, we set $e = \log(2) - \log(1 - p_{\text{noise}})$, because such synthetic noise increases the overall training loss and thus requires a larger error threshold $e$. In addition, we evaluate the parameter dissimilarity (PD) method proposed in~\cite{zhou2025new} to automatically determine the value of $e$ in Section~\ref{section:experiments:benchmark}, with implementation details provided in Appendix~\ref{appendix:implementation:PD}.

{\bf Evaluation Metrics.}
During training, we use the dataset-dependent default threshold $e$ defined above. In the evaluation phase, we compute E-Prop on the evaluation set (the test set when available; otherwise the validation set) using a fixed threshold $\log(2)$ (denoted as E-Prop$^\dagger$). Namely, E-Prop$^\dagger$ we report is the proportion of samples with $\mathrm{L}^{(n)}\leq \log(2)$. Since $\mathrm{L}^{(n)}\leq \log(2)$ corresponds to assigning at least 50\% probability to the true class, it indicates a confident correct prediction. We denote the accuracy on the evaluation set by ``T-ACC'' and report both T-ACC and E-Prop$^\dagger$ averaged over 5 runs.

{\bf Notation for APW Variants.} 
We denote the vanilla baseline that minimizes the standard mini-batch average loss (i.e., uniform sample weights) as ``Vanilla''. For each of APW, S-APW, and M-APW, we implement three weighting approaches, namely E-weighting, I-weighting, and EI-weighting, and denote the resulting variants as APW-\{E, I, EI\}, S-APW-\{E, I, EI\}, and M-APW-\{E, I, EI\}. In addition, S-APW-A denotes that after sampling is completed, we train the model with the vanilla average loss. Among the three weighting approaches, we recommend E-weighting as the default for simplicity and efficiency. For brevity, we refer to APW-\{E, I, EI\} as APWs, S-APW-\{A, E, I, EI\} as S-APWs, and M-APW-\{E, I, EI\} as M-APWs.

\subsection{Alignment with Fair Benchmark}\label{section:experiments:benchmark}

CurML establishes a fair benchmark~\cite{zhou2022curml}, which evaluates most existing CL methods on clean CIFAR-10 and CIFAR-10 with 40\% synthetic label noise. These CL methods include methods assigning binary sample weights ({\it e.g.,}~\cite{jiang2015self,zhou2018minimax,matiisen2019teacher,kong2021adaptive}) and methods assigning soft sample weights ({\it e.g.,}~\cite{kim2018screenernet,shu2019meta,castells2020superloss}). Following the experimental setup in this fair comparison, we evaluate the proposed APW and S-APW on CIFAR-10 and CIFAR-100, with and without 40\% synthetic label noise. We use (N-40) to denote the dataset with 40\% synthetic label noise. The network architecture and optimizer settings are kept consistent with CurML, and we refer to this network as Simple-CNN. 

\begin{figure}[htbp]
\centering
\includegraphics[width=0.7\textwidth]{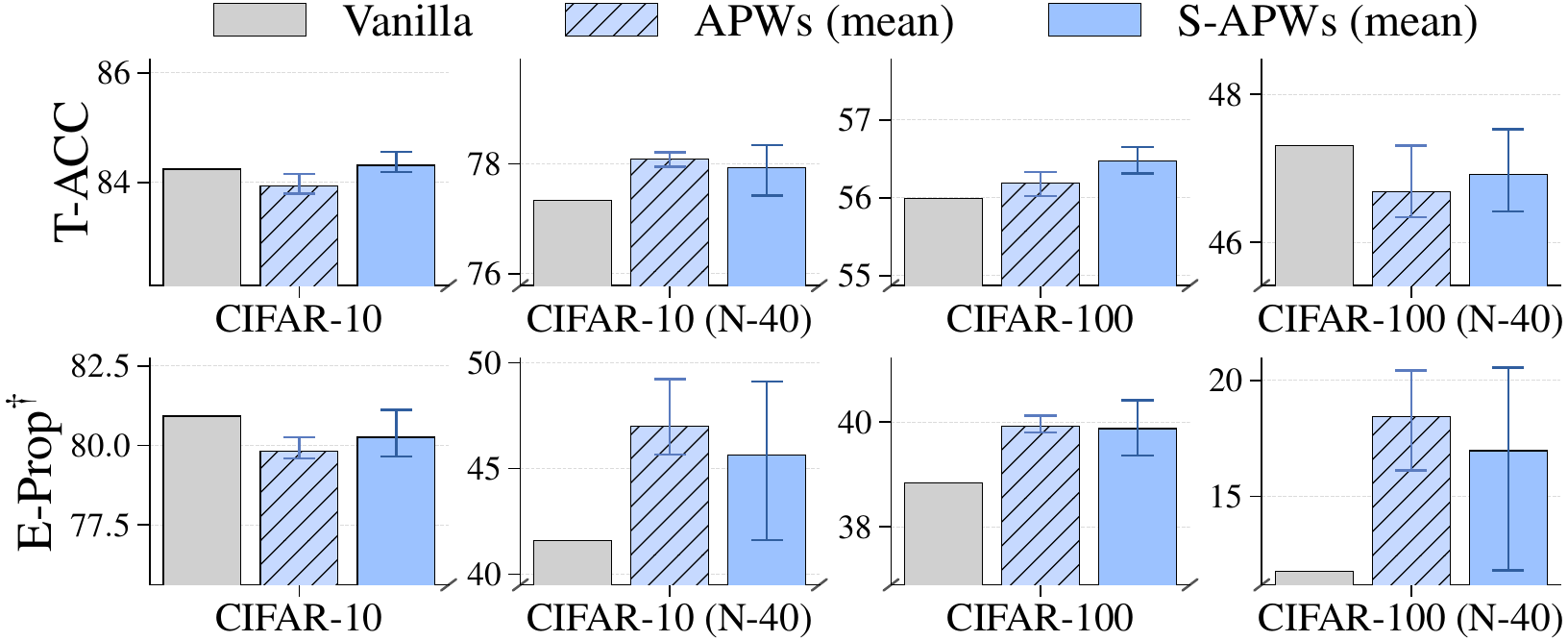}
\caption{Experimental results on CIFAR-10/100 and CIFAR-10/100 (N-40) trained using Simple-CNN. For APWs and S-APWs, bars denote mean performance over the three weighting approaches (E-weighting, I-weighting, and EI-weighting) and error bars span the minimum and maximum across these strategies.}
\label{fig:CIFAR_10_100}
\end{figure}

Figure~\ref{fig:CIFAR_10_100} shows the experimental results. On clean CIFAR-10, APW and S-APW perform slightly worse than Vanilla under most of the three weighting approaches, in line with CurML's observations for other CL methods (e.g.,~\cite{zhou2018minimax,kong2021adaptive}). Since Simple-CNN already performs very well on clean CIFAR-10, this easy task leaves limited headroom for CL methods and may even lead to slightly worse performance than the vanilla baseline. In contrast, on CIFAR-10 (N-40) and clean CIFAR-100, both APW and S-APW substantially outperform Vanilla. However, on CIFAR-100 (N-40), APW and S-APW achieve poor T-ACC while obtaining a significant improvement in E-Prop$^\dagger$. We attribute this behavior to an overly small $q$ ({\it cf.} the discussion of Theorem~\ref{theorem:training_effectiveness}), which makes the APW method downweight the hard samples too aggressively and consequently weakens its guidance effect ({\it cf.}~Section~\ref{section:APW:discussion}).

\begin{figure}[htbp]
\centering
\includegraphics[width=0.75\textwidth]{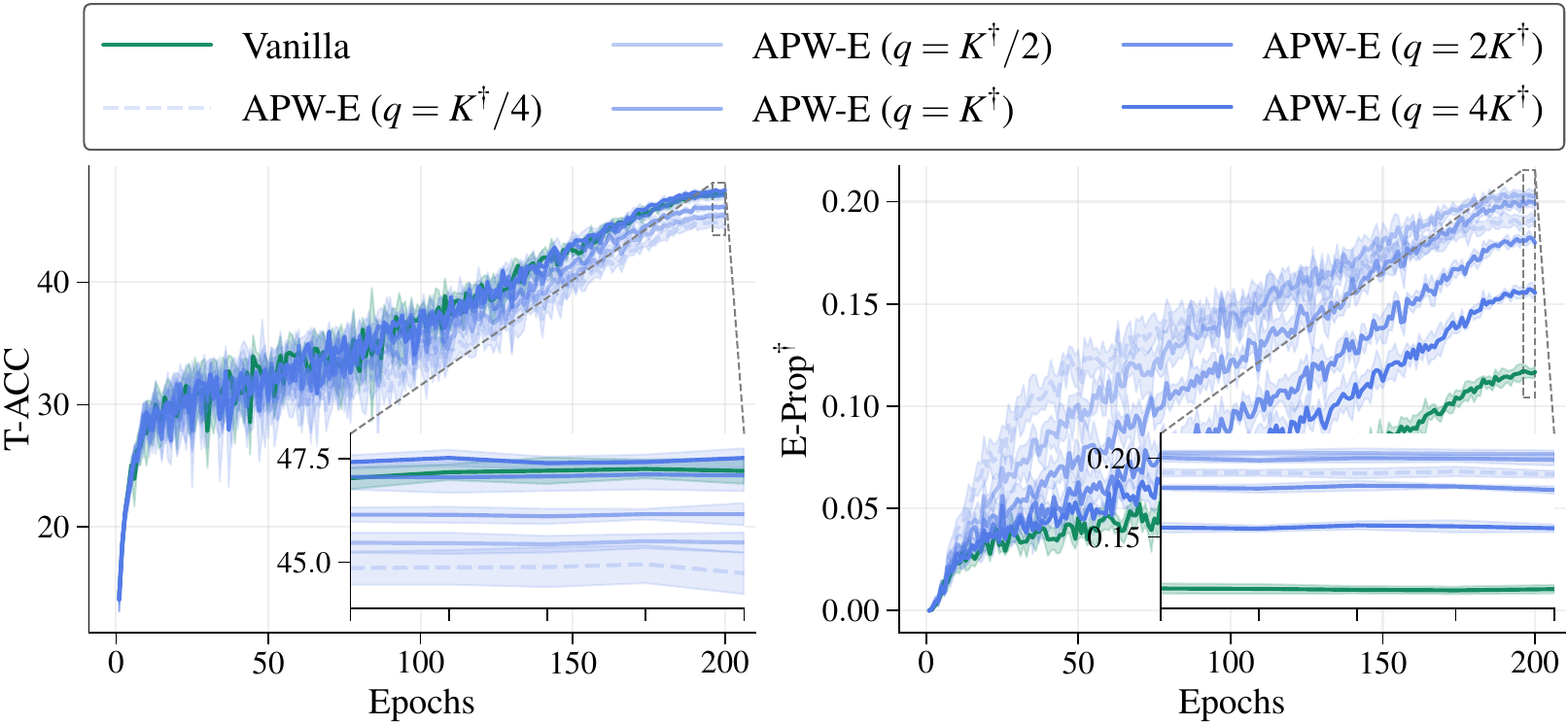}
\caption{Sensitivity of APW to $q$. Experimental results on CIFAR-100 (N-40) trained using Simple-CNN. We report T-ACC and E-Prop$^\dagger$ of APW-E under different $q \in \{K^\dagger/4, K^\dagger/2, K^\dagger, 2K^\dagger, 4K^\dagger\}$. Each curve shows the mean over 5 runs, with shaded regions indicating $\pm$ one standard deviation.}\label{fig:different_q}
\end{figure}

To further support this interpretation, we conduct a sensitivity analysis of the choice of $q$ in APW. As shown in Figure~\ref{fig:different_q}, increasing $q$ gradually weakens the denoising effect (E-Prop$^\dagger$) while strengthening the guidance effect (T-ACC). When $q=4K^\dagger$, APW achieves a better balance by maintaining the denoising while restoring the guidance. However, simply decreasing $q$ does not monotonically enhance the denoising, as E-Prop$^\dagger$ at $q=K^\dagger/4$ is lower than that at $q=K^\dagger/2$. These results suggest that tuning $q$ can improve the performance of APW, especially when the model struggles to fit the dataset; in such cases, a larger $q$ is a more robust choice.

\begin{table*}[htbp]
\centering
\caption{Experimental results on WebVision trained using Inception-ResNet-v2. We evaluate the proposed APW, S-APW, and M-APW using the three weighting approaches (E-weighting, I-weighting, and EI-weighting) and report the Top-1-based metrics T-ACC and E-Prop$^\dagger$. In each row, the entries marked with {\setlength{\fboxsep}{2pt}\colorbox{tablemidblue}{\textbf{blue}}} represent the maximum value, and the entries marked with {\setlength{\fboxsep}{2pt}\colorbox{tablelightblue}{lightblue}} represent the second-highest value. }\label{tab:details_webvision}
\begin{threeparttable}
\resizebox{\linewidth}{!}{
\begin{tabular}{l|c|ccc|cccc|ccc}
\hline
\multicolumn{1}{l|}{}                                   & \multicolumn{1}{c|}{}                            & \multicolumn{3}{c|}{APWs}                                                                                   & \multicolumn{4}{c|}{S-APWs}                                                                                                                        & \multicolumn{3}{c}{M-APWs} \\ \cline{3-12}
\multicolumn{1}{l|}{\multirow{-2}{*}{Metrics}}          & \multicolumn{1}{c|}{\multirow{-2}{*}{Vanilla}}   & \multicolumn{1}{c}{APW-E}              & \multicolumn{1}{c}{APW-I}  & \multicolumn{1}{c|}{APW-EI}           & \multicolumn{1}{c}{S-APW-A}       & \multicolumn{1}{c}{S-APW-E}  & \multicolumn{1}{c}{S-APW-I}        & \multicolumn{1}{c|}{S-APW-EI}              & \multicolumn{1}{c}{M-APW-E}                    & \multicolumn{1}{c}{M-APW-I}               & \multicolumn{1}{c}{M-APW-EI}   \\ \hline
% =======================================================
% ===============     WebVision (e-v1)    ===============
% =======================================================
T-ACC                                                   & 66.13                                            & \cellcolor{tablelightblue}{66.23}      & 66.15                      & 66.14                                 & 66.02                             & 66.20                        & 66.14                              & \cellcolor{tablemidblue}\textbf{66.25}     & 66.13                                          & 66.20                                     & 66.16 \\
E-Prop$^\dagger$                                        & 51.80                                            & 53.41                                  & 52.54                      & 52.95                                 & 51.76                             & 53.47                        & 53.24                              & 53.20                                      & 62.86                                          & \cellcolor{tablemidblue}\textbf{62.94}    & \cellcolor{tablelightblue}{62.89} \\ \hline
\end{tabular}}
\end{threeparttable}
\end{table*}

CurBench~\cite{zhou2024curbench} includes the representative CL methods from CurML~\cite{zhou2022curml} and extends the benchmark settings beyond image classification to NLP and graph machine learning. Following the showcased benchmark examples in CurBench, we fine-tune BERT~\cite{devlin2019bert} on RTE and train GCN~\cite{kipf2016semi} on NCI1 for a fair comparison. In addition, we evaluate APW with two choices of $e$, namely, the default setting and an automatically selected value. Specifically, we determine $e$ using the PD method proposed in~\cite{zhou2025new}, which automatically detects the transition point between two phases of the training process. As shown in Table~\ref{tab:nlp_and_graph}, the results obtained with the default $e$ and the automatically selected $e$ are comparable. The default $e$ usually performs better, likely because it leverages the estimated label-noise rate.

\begin{table}[htbp]
\centering
\caption{Sensitivity of APW to $e$. Experimental results on RTE trained using BERT and NCI1 trained using GCN. We report T-ACC of APW-E when using the default $e$ and the PD-estimated $e$ during training. Since both datasets are binary classification tasks, T-ACC equals E-Prop$^\dagger$ under our evaluation definition. Bold entries indicate the best result in each row.}\label{tab:nlp_and_graph}
\resizebox{0.7\linewidth}{!}{
\begin{tabular}{l|cccc}
\hline
\multicolumn{1}{l|}{\multirow{2}{*}{Method}} & \multicolumn{4}{c}{Dataset} \\ \cline{2-5}
\multicolumn{1}{l|}{}                        & RTE            & RTE (N-40)     & NCI1           & NCI1 (N-40) \\ \hline
Vanilla                                      & 64.40          & 54.42          & 69.99          & 55.67 \\
APW-E (Default $e$)                          & 65.20          & \textbf{57.13} & 71.70          & \textbf{58.03} \\
APW-E (PD-estimated $e$)                     & \textbf{65.49} & 56.32          & \textbf{71.94} & 57.24 \\ \hline
\end{tabular}    
}
\end{table}

\subsection{Comparison with Loss-Reweighting Methods}\label{section:experiments:comparison}

Following the categorization in CurBench~\cite{zhou2024curbench}, the APW method falls into the ``via Loss Reweighting'' category of CL methods, {\it i.e.,} it is a loss-reweighting CL method. Therefore, we use the representative loss-reweighting methods from CurBench, including ScreenerNet~\cite{kim2018screenernet}, MW-Net~\cite{shu2019meta}, and SuperLoss~\cite{castells2020superloss}, as comparison methods.

To assess the effectiveness of APW under real-world label noise, we conduct comparative experiments on CIFAR-10N, CIFAR-100N, and Mini-WebVision. CIFAR-10N and CIFAR-100N are noisy variants of CIFAR-10 and CIFAR-100, respectively, and their labels are obtained via human re-annotation. Thus, these two datasets contain label noise reflecting real-world annotation errors. Mini-WebVision is a subset of WebVision, and it is commonly used as a substitute for the full WebVision dataset for rapid comparison of CL methods. We consider four network architectures: ResNet-18~\cite{he2016deep}, EfficientNet~\cite{tan2019efficientnet}, ConvNeXt~\cite{liu2022convnet}, and MaxViT~\cite{tu2022maxvit}. These networks span residual CNNs (ResNet), modern CNN designs (EfficientNet and ConvNeXt), and hybrid transformer models (MaxViT) to enable a more comprehensive evaluation.

\begin{table*}[htbp]
\centering
\caption{Comparison of T-ACC for different CL methods on three datasets (CIFAR-10N, CIFAR-100N, and Mini-WebVision) trained using four backbone models (ResNet-18, EfficientNet, ConvNeXt, and MaxViT). Each table cell shows a bar chart where the x-axis denotes the CL methods and the y-axis denotes T-ACC. For APWs and M-APWs, bars indicate the mean performance over the three weighting approaches (E-weighting, I-weighting, and EI-weighting), and error bars span the minimum and maximum across these strategies. The bottom row lists the CL method corresponding to each bar color.}\label{tab:comparison_CL}
\small 
\begin{threeparttable} 
\setlength{\tabcolsep}{2pt}
\begin{tabular}{ c| >{\centering\arraybackslash}m{0.22\linewidth} 
                | >{\centering\arraybackslash}m{0.22\linewidth} 
                | >{\centering\arraybackslash}m{0.22\linewidth} 
                | >{\centering\arraybackslash}m{0.22\linewidth}}
\toprule[0.7pt]
\multicolumn{1}{c}{} 
& ResNet-18
& EfficientNet
& ConvNeXt 
& MaxViT \\
\midrule[0.5pt]
% =======================================================
% ==================     CIFAR-10-N    ==================
% =======================================================
\makecell[c]{\rotatebox{90}{{CIFAR-10N}}}
% ResNet-18
&\begin{minipage}[b]{0.22\textwidth}
\centering
\includegraphics[width=\linewidth]{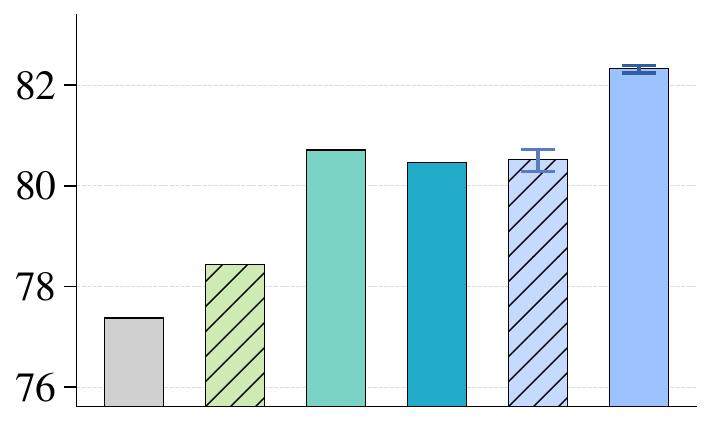}
\end{minipage}
% EfficientNet
&\begin{minipage}[b]{0.22\textwidth}
\centering
\includegraphics[width=\linewidth]{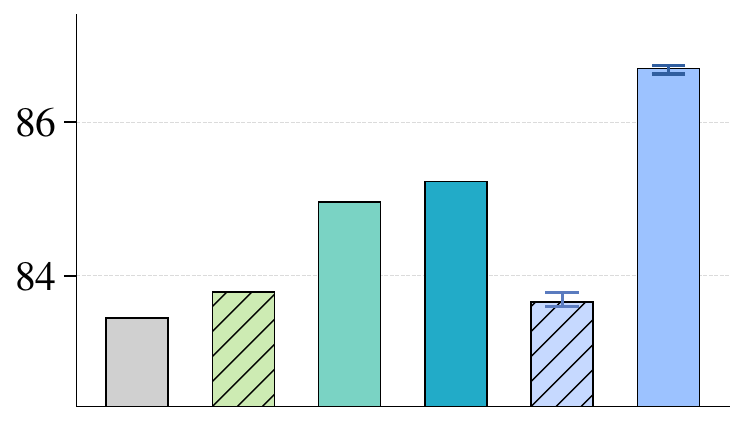}
\end{minipage}
% ConvNeXt
&\begin{minipage}[b]{0.22\textwidth}
\centering
\includegraphics[width=\linewidth]{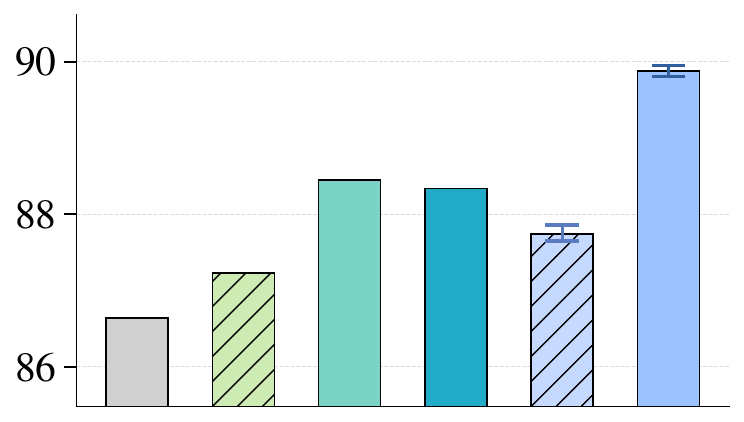}
\end{minipage}
% MaxViT
&\begin{minipage}[b]{0.22\textwidth}
\centering
\includegraphics[width=\linewidth]{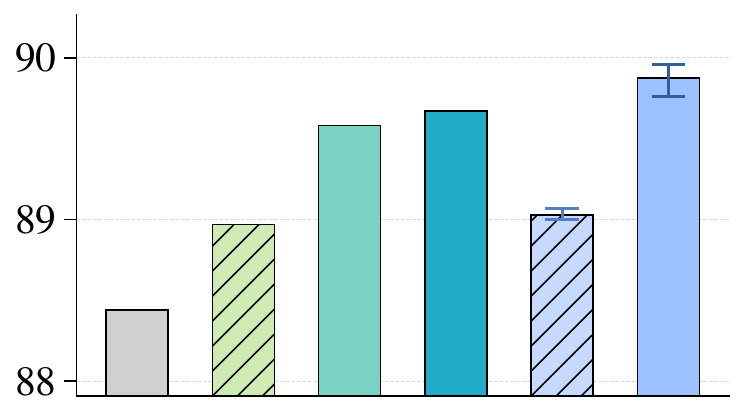}
\end{minipage}
\\
\midrule[0.5pt]
% ======================================================
% =================     CIFAR-100-N    =================
% ======================================================
\makecell[c]{\rotatebox{90}{{CIFAR-100N}}}
% ResNet-18
&\begin{minipage}[b]{0.22\textwidth}
\centering
\includegraphics[width=\linewidth]{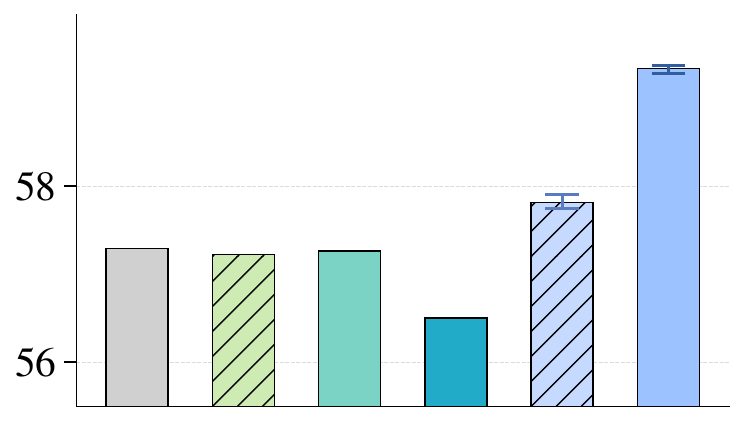}
\end{minipage}
% EfficientNet
&\begin{minipage}[b]{0.22\textwidth}
\centering
\includegraphics[width=\linewidth]{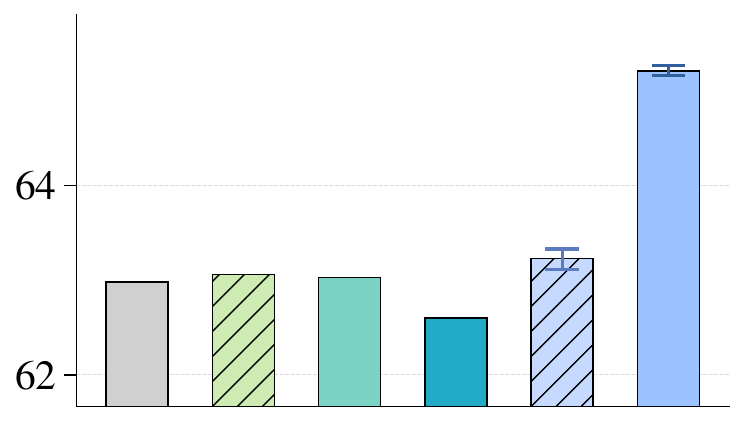}
\end{minipage}
% ConvNeXt
&\begin{minipage}[b]{0.22\textwidth}
\centering
\includegraphics[width=\linewidth]{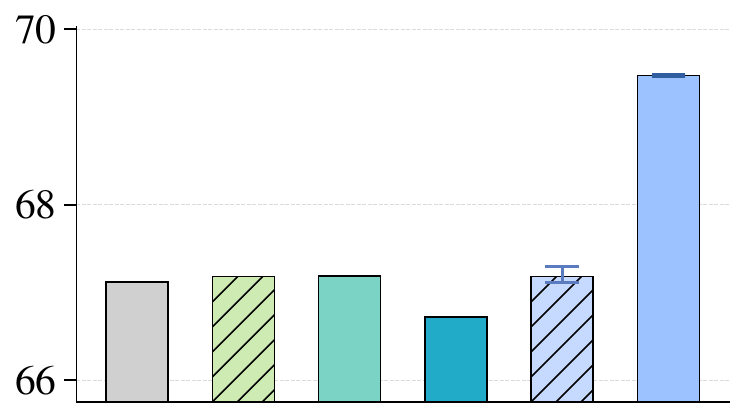}
\end{minipage}
% MaxViT
&\begin{minipage}[b]{0.22\textwidth}
\centering
\includegraphics[width=\linewidth]{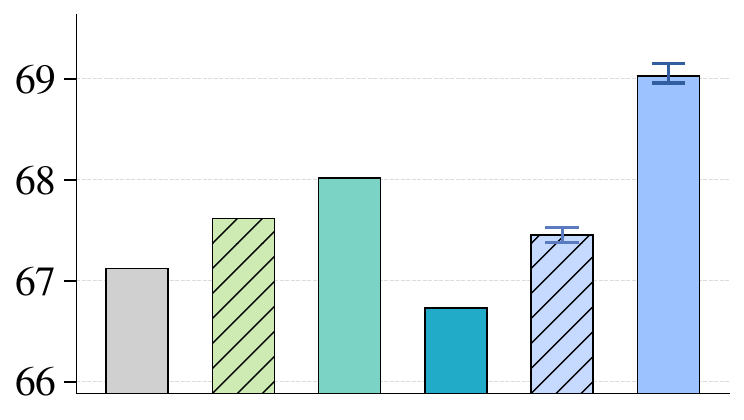}
\end{minipage}
\\
\midrule[0.5pt]
% =======================================================
% ==================     Mini-WebVision    ==================
% =======================================================
\makecell[c]{\rotatebox{90}{{Mini-WebVision}}}
% ResNet-18
&\begin{minipage}[b]{0.22\textwidth}
\centering
\includegraphics[width=\linewidth]{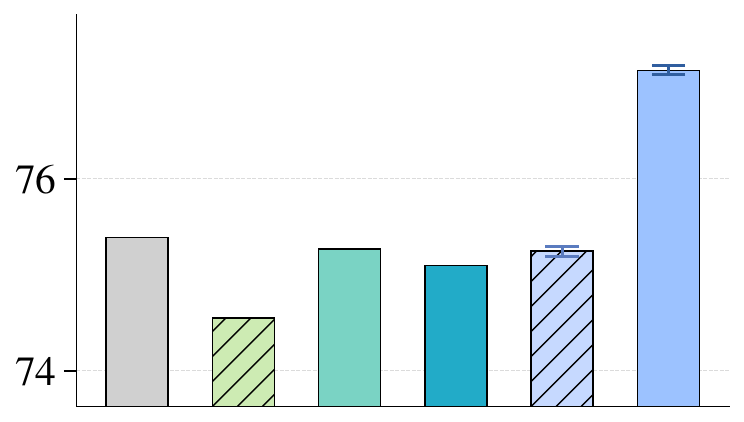}
\end{minipage}
% EfficientNet
&\begin{minipage}[b]{0.22\textwidth}
\centering
\includegraphics[width=\linewidth]{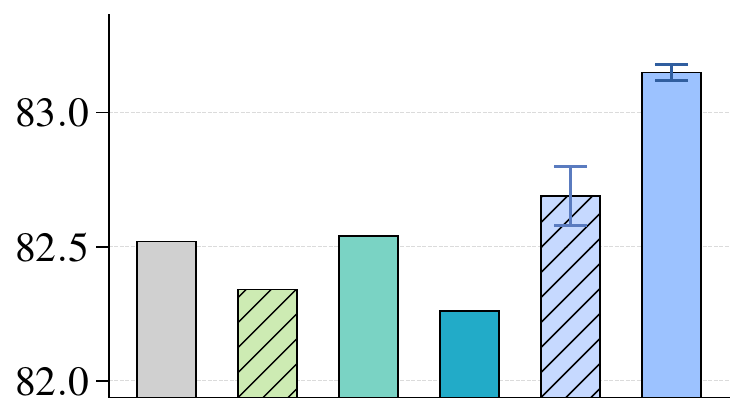}
\end{minipage}
% ConvNeXt
&\begin{minipage}[b]{0.22\textwidth}
\centering
\includegraphics[width=\linewidth]{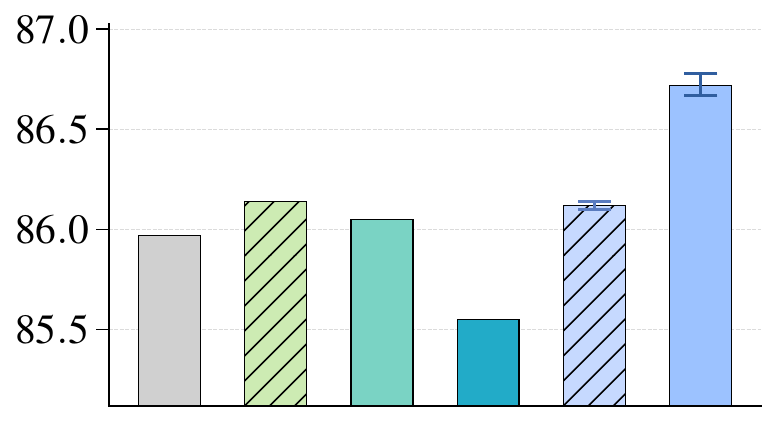}
\end{minipage}
% MaxViT
&\begin{minipage}[b]{0.22\textwidth}
\centering
\includegraphics[width=\linewidth]{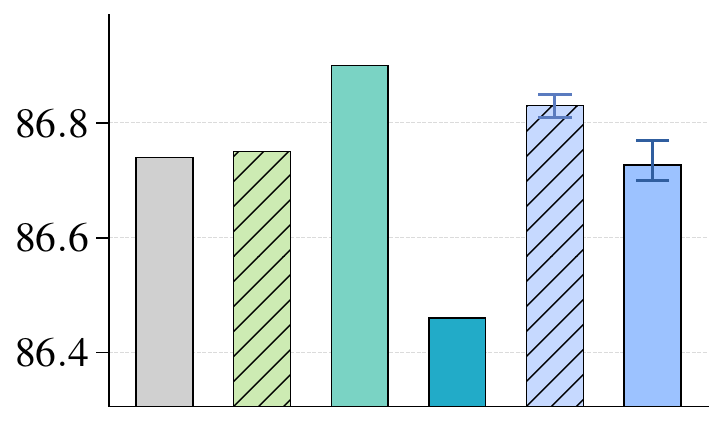}
\end{minipage}
\\
\midrule[0.5pt]
% ======================= Legend =======================
\multicolumn{5}{c}{%
\includegraphics[width=0.88\linewidth]{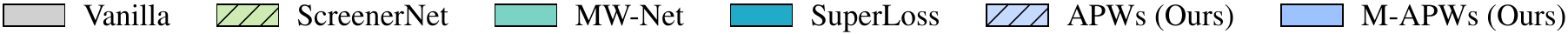}
} 
\\[-3pt]
\bottomrule[0.7pt]
\end{tabular}
\end{threeparttable}
\end{table*}

Table~\ref{tab:comparison_CL} reports the comparative results. It can be observed that, across most experimental settings, APW consistently improves over the vanilla baseline, and M-APW further outperforms both APW and other loss-reweighting methods. These results demonstrate the effectiveness of the proposed methods on the datasets with real-world label noise, and show that they can be applied to diverse network architectures. Additionally, other loss-reweighting methods often exhibit inconsistent performance within the same dataset across different networks, and in some settings they even fail to surpass Vanilla. This observation is consistent with the findings of CurBench~\cite{zhou2024curbench}, which show that the existing CL methods typically require careful hyperparameter tuning for specific problems and that no CL method is universally optimal across all scenarios.

Using the same experimental settings as in Table~\ref{tab:comparison_CL}, we conduct an ablation study on M-APW. Specifically, in Appendix~\ref{appendix:implementation:Ablation-M-APW}, we compare standard mixup and M-APW using the metrics T-ACC and E-Prop$^\dagger$. The experimental results show that the standard mixup improves T-ACC over the vanilla baseline on real-world datasets, but it often reduces E-Prop$^\dagger$. In contrast, M-APW consistently improves both T-ACC and E-Prop$^\dagger$. These findings indicate that M-APW combines APW's ability to increase E-Prop with mixup's robustness to label noise.

\subsection{Comprehensive Evaluation of APW and Its Variants on WebVision}\label{section:experiments:WebVision}

We consider WebVision~\cite{li2017WebVision}, which is a large-scale real-world dataset containing over 2.4 million images across 1,000 categories. The images were crawled from publicly available web sources. The label noise arises from the web scraping process, and includes 1) in-domain noise, where image labels may be misinterpreted or incorrect, and 2) out-of-domain noise, where images from unrelated classes are mistakenly assigned to the target categories. For the network architecture, we adopt the commonly used baseline model for WebVision, the Inception-ResNet-v2 backbone~\cite{szegedy2017inception}.

We comprehensively evaluate the proposed APW, S-APW, and M-APW using the three weighting approaches on the WebVision dataset. As shown in Table~\ref{tab:details_webvision}, compared with the vanilla baseline, APW and S-APW achieve similar performance, slightly improving T-ACC and modestly increasing E-Prop$^\dagger$. Taking advantage of the noise robustness of mixup, M-APW achieves a much larger improvement in E-Prop$^\dagger$, by around 10 percentage points. These results support the theoretical findings, showing that the APW method substantially improves prediction confidence, even when the network is trained on real-world datasets.

\section{Conclusion}\label{section:conclusion}

In this paper, we propose the adaptively point-weighting curriculum learning method, APW, which is to our knowledge the first loss-reweighting method that adaptively enforces the easy-to-hard training paradigm in CL. APW is lightweight and broadly applicable. It involves only two interpretable hyperparameters, the fewest among the existing automatic CL methods, and its weighting strategy accounts for the entire training set and the evolving network state. To demonstrate the applicability, we propose two practical variants, S-APW and M-APW. Our theoretical analysis shows that APW can increase E-Prop on the training set, which can enhance E-Prop on the test set.

While we focus on classification tasks, in future work we will extend APW by integrating it with more sample-level techniques to tackle a broader range of tasks. Moreover, the difficulty measurer of APW considers only a binary relationship between the error threshold and the per-sample training loss, which makes it difficult to determine whether hard samples arise from label noise or from genuinely complex data patterns. This limitation could potentially be mitigated by combining APW with meta-learning.

\bibliography{APW}
\bibliographystyle{plainnat}

%%%%%%%%%%%%%%%%%%%%%%%%%%%%%%%%%%%%%%%%%%%%%%%%%%%%%%%%%%%%%%%%%%%%%%%%%%%%%%%
%%%%%%%%%%%%%%%%%%%%%%%%%%%%%%%%%%%%%%%%%%%%%%%%%%%%%%%%%%%%%%%%%%%%%%%%%%%%%%%
% APPENDIX
%%%%%%%%%%%%%%%%%%%%%%%%%%%%%%%%%%%%%%%%%%%%%%%%%%%%%%%%%%%%%%%%%%%%%%%%%%%%%%%
%%%%%%%%%%%%%%%%%%%%%%%%%%%%%%%%%%%%%%%%%%%%%%%%%%%%%%%%%%%%%%%%%%%%%%%%%%%%%%%
\newpage
\appendix
\onecolumn

\clearpage

\section{Implementation Details}\label{appendix:implementation}

In this appendix, we present the detailed algorithms for the three weighting approaches (E-weighting, I-weighting, and EI-weighting) and analyze the differences between these approaches. We then provide implementation details for the PD method proposed in~\cite{zhou2025new} to automatically determine the value of $e$. Finally, we provide a detailed description of the two variants of APW, namely, S-APW and M-APW.

\subsection{Three Weighting Approaches: E-weighting, I-weighting, and EI-weighting}\label{appendix:implementation:approaches}

Let $\mathcal{Z} = \{\mathbf{z}^{(n)}\}_{n=1}^{N}$ denote the entire training set, and let $\mathbf{w}_{k} = \{w_{k}^{(n)}\}_{n=1}^{N}$ denote the resulting sample-weight vector assigned by APW at the $k$-th epoch. Let $\widetilde{\mathcal{Z}}_{k,s} \subset \mathcal{Z}$ denote the mini-batch at the $s$-th iteration of the $k$-th epoch, with size $N_{k,s} = |\widetilde{\mathcal{Z}}_{k,s}|$, where $s \in \mathbb{N}^+$. The reweighted mini-batch loss at this iteration is defined as
\begin{equation}\label{eq:mini_batch_APW_loss}
\mathrm{L}_{k,s}^\mathrm{APW} = \sum_{\mathbf{z}^{(i)} \in \widetilde{\mathcal{Z}}_{k,s}} w_{k,s}^{(i)} \,\mathrm{L}^{(i)}_{k,s},
\end{equation}
where $\mathrm{L}^{(i)}_{k,s}$ is the per-sample loss of $\mathbf{z}^{(i)} \in \widetilde{\mathcal{Z}}_{k,s}$ evaluated at the current network parameters $\Theta_{k,s}$, and $w_{k,s}^{(i)}$ is the corresponding sample weight which is obtained from $\mathbf{w}_k$ in E-weighting and computed on the fly in I-weighting/EI-weighting. Next, we introduce the detailed algorithms for the three weighting approaches, using the $k$-th epoch to illustrate the procedure.
\begin{enumerate}
\item \textbf{E-weighting:}
\begin{itemize}
\item At the beginning of epoch $k$, compute the loss vector $\{\mathrm{L}_k^{(n)}\}_{n=1}^{N}$ with the network parameters $\Theta_{k,0}$ ({\it i.e.,} $\Theta_{k}$)
\item Calculate the value of $\alpha_k$ ({\it cf.} Eq.~\eqref{eq:alpha_k}) using the loss vector $\{\mathrm{L}_k^{(n)}\}_{n=1}^{N}$
\item {\it Calculate the sample-weight vector $\mathbf{w}_{k} = \{w_{k}^{(n)}\}_{n=1}^{N}$ ({\it cf.} Eq.~\eqref{eq:weight_update}) using $\alpha_k$ and $\{w_{k-1}^{(n)}\}_{n=1}^{N}$}
\item For each mini-batch $\widetilde{\mathcal{Z}}_{k,s} \subset \mathcal{Z}$:
\begin{itemize}
\item compute the per-sample losses $\mathrm{L}^{(i)}_{k,s}$ using the network parameters $\Theta_{k,s}$
\item {\it for each $\mathbf{z}^{(i)} \in \widetilde{\mathcal{Z}}_{k,s}$, set $w_{k,s}^{(i)} \leftarrow w_k^{(i)}$; then, normalize the sample weights in the mini-batch so that they sum to one}
\item compute the mini-batch loss $\mathrm{L}_{k,s}^\mathrm{APW}$ by Eq.~\eqref{eq:mini_batch_APW_loss} with the sample weights $w_{k,s}^{(i)}$
\item update $\Theta_{k,s}$ with the gradient $\nabla_{\Theta_{k,s}} \mathrm{L}_{k,s}^\mathrm{APW}$
\end{itemize}
\end{itemize}
\item \textbf{I-weighting:}
\begin{itemize}
\item At the beginning of epoch $k$, compute the loss vector $\{\mathrm{L}_k^{(n)}\}_{n=1}^{N}$ with the network parameters $\Theta_{k,0}$ ({\it i.e.,} $\Theta_{k}$)
\item Calculate the value of $\alpha_k$ ({\it cf.} Eq.~\eqref{eq:alpha_k}) using the loss vector $\{\mathrm{L}_k^{(n)}\}_{n=1}^{N}$
\item For each mini-batch $\widetilde{\mathcal{Z}}_{k,s} \subset \mathcal{Z}$:
\begin{itemize}
\item compute the per-sample losses $\mathrm{L}^{(i)}_{k,s}$ using the network parameters $\Theta_{k,s}$
\item {\it for each $\mathbf{z}^{(i)} \in \widetilde{\mathcal{Z}}_{k,s}$, calculate the mini-batch weight $w_{k,s}^{(i)}$ by adapting the global update manner in Eq.~\eqref{eq:weight_update} to the current mini-batch, using $\alpha_k$ and $w_{k-1}^{(i)}$, and then update $w_{k}^{(i)} \leftarrow \frac{N_{k,s}}{N} w_{k,s}^{(i)}$}
\item compute the mini-batch loss $\mathrm{L}_{k,s}^\mathrm{APW}$ by Eq.~\eqref{eq:mini_batch_APW_loss} with the sample weights $w_{k,s}^{(i)}$
\item update $\Theta_{k,s}$ with the gradient $\nabla_{\Theta_{k,s}} \mathrm{L}_{k,s}^\mathrm{APW}$
\end{itemize}
\item {\it Normalize the sample weights $w_{k}^{(i)}$ so that they sum to one, and then obtain the resulting sample-weight vector $\mathbf{w}_k = \{w_k^{(n)}\}_{n=1}^{N}$.}
\end{itemize}
\item \textbf{EI-weighting:}
\begin{itemize}
\item At the beginning of epoch $k$, compute the loss vector $\{\mathrm{L}_k^{(n)}\}_{n=1}^{N}$ with the network parameters $\Theta_{k,0}$ ({\it i.e.,} $\Theta_{k}$)
\item Calculate the value of $\alpha_k$ ({\it cf.} Eq.~\eqref{eq:alpha_k}) using the loss vector $\{\mathrm{L}_k^{(n)}\}_{n=1}^{N}$
\item {\it Calculate the sample-weight vector $\widetilde{\mathbf{w}}_{k} = \{\widetilde{w}_{k}^{(n)}\}_{n=1}^{N}$ ({\it cf.} Eq.~\eqref{eq:weight_update}) using $\alpha_k$ and $\{w_{k-1}^{(n)}\}_{n=1}^{N}$}
\item For each mini-batch $\widetilde{\mathcal{Z}}_{k,s} \subset \mathcal{Z}$:
\begin{itemize}
\item compute the per-sample losses $\mathrm{L}^{(i)}_{k,s}$ using the network parameters $\Theta_{k,s}$
\item {\it for each $\mathbf{z}^{(i)} \in \widetilde{\mathcal{Z}}_{k,s}$, calculate the mini-batch weight $w_{k,s}^{(i)}$ by adapting the global update manner in Eq.~\eqref{eq:weight_update} to the current mini-batch, using $\alpha_k$ and $\widetilde{w}_{k}^{(i)}$, and then update $w_{k}^{(i)} \leftarrow \frac{N_{k,s}}{N} w_{k,s}^{(i)}$}
\item compute the mini-batch loss $\mathrm{L}_{k,s}^\mathrm{APW}$ by Eq.~\eqref{eq:mini_batch_APW_loss} with the sample weights $w_{k,s}^{(i)}$
\item update $\Theta_{k,s}$ with the gradient $\nabla_{\Theta_{k,s}} \mathrm{L}_{k,s}^\mathrm{APW}$
\end{itemize}
\item {\it Normalize the sample weights $w_{k}^{(i)}$ so that they sum to one, and then obtain the resulting sample-weight vector $\mathbf{w}_k = \{w_k^{(n)}\}_{n=1}^{N}$.}
\end{itemize}
\end{enumerate}
In the above description, the italicized steps indicate where the three weighting approaches differ. E-weighting is the default weighting approach that updates the weights of all training samples once at the beginning of each epoch. In contrast, I-weighting updates sample weights by taking into account the changes in the training state within an epoch. EI-weighting combines E-weighting and I-weighting. Specifically, it performs a global update of all sample weights once per epoch and further adjusts them according to the intra-epoch training state. When updating the global sample-weight vector from a mini-batch, we scale the updated weights by a factor of $N_{k,s}/N$, where $N_{k,s}$ is the mini-batch size and $N$ is the entire training set size, to keep the updated weights on the same scale as the global sample-weight vector.

\subsection{Parameter Dissimilarity (PD) Method}\label{appendix:implementation:PD}

The latest empirical evidence on network learning dynamics shows a two-phase training behavior~\cite{zhou2025new}, and the PD method was proposed to detect the phase transition point. Inspired by that work, APW introduces an error threshold $e$ and distinguishes between an early and a later training phase by comparing the training-difficulty indicator $\rho_k$ with $\frac{1}{2}$ (see Table~\ref{tab:relation}). When $\rho_k$ decreases to $\frac{1}{2}$, this indicates that the total weight on hard samples drops to one half ({\it cf.} Eq.~\eqref{eq:rho_k}). Therefore, to align APW's phase split with the PD-estimated transition, we set the error threshold $e$ in APW to the average training loss at the phase transition point estimated by the PD method.

The PD method uses the cosine dissimilarity between two flattened parameter vectors as a quantitative measure for analyzing the directionality of training dynamics. Specifically, the network parameters are recorded over training in chronological order, forming a sequence of checkpoints $\{ \Theta^{(t)} \}_{t=1}^{T}$. Each checkpoint $\Theta^{(t)}$ is then flattened into a parameter vector $\operatorname{vec}(\Theta^{(t)})$. The cosine dissimilarity between two checkpoints is defined as
\begin{equation*}
d(t_0,t_1) = 1 - \cos\big\langle \operatorname{vec}(\Theta^{(t_0)}), \operatorname{vec}(\Theta^{(t_1)}) \big\rangle,
\end{equation*}
where
\begin{equation*}
\cos \big\langle \operatorname{vec}(\Theta^{(t_0)}), \operatorname{vec}(\Theta^{(t_1)}) \big\rangle =  \frac{\big\langle \operatorname{vec}(\Theta^{(t_0)}), \operatorname{vec}(\Theta^{(t_1)}) \big\rangle}{\|\operatorname{vec}(\Theta^{(t_0)})\|_2\,\|\operatorname{vec}(\Theta^{(t_1)})\|_2}.
\end{equation*}
Empirical evidence in~\cite{zhou2025new} suggests that a large $d(t_0,t_1)$ typically indicates a pronounced change in the optimization direction between checkpoints $t_0$ and $t_1$ (with $1 \le t_0 < t_1 \leq T$). Thus, the phase transition point can be detected from the two-dimensional grid of $d(t_0,t_1)$ values. Specifically, there often exists a time index $t_1$ such that, regardless of the choice of $t_0 < t_1$, $d(t_0,t_1)$ attains its peak at nearly the same $t_1$. Motivated by this phenomenon, we define the following measure
\begin{equation*}
d^{\dagger}(t_1) = \frac{1}{t_1-1}\sum_{t_0=1}^{t_1-1} d(t_0,t_1).
\end{equation*}
In practice, we record a sequence of checkpoints during the training process of Vanilla and then approximate the phase transition point as
\begin{equation*}
t^{*}=\arg\max_{t_1} d^{\dagger}(t_1).
\end{equation*}
Subsequently, we compute the average per-sample loss at this point and use it as the error threshold $e$ in APW.

\subsection{Sampling-APW (S-APW)}~\label{appendix:implementation:S-APW}

The CL methods that assign every training sample a binary weight typically start from an easy subset of the training data and gradually introduce harder samples as training progresses~\cite{jiang2015self,zhou2018minimax,matiisen2019teacher,kong2021adaptive}. The training scheduler in such CL methods increases the difficulty (or complexity) of the current training set according to a curriculum schedule until all training samples are included. This data selection manner assigns a binary weight to every training sample, whereas the APW method assigns a soft weight to every training sample. To illustrate the application of APW, we devise the S-APW method by taking the APW-assigned weight $w_{k}^{(n)}$ for sample $\mathbf{z}^{(n)}$ as its sampling probability. Let $\mathcal{Z}$ denote the entire training set, and let $\mathcal{Z}_k$ denote the training subset used at the $k$-th epoch. The workflow of S-APW is summarized in Algorithm~\ref{alg:S-APW}.

\begin{algorithm}[htbp]
\caption{The workflow of the S-APW method}\label{alg:S-APW}
\begin{algorithmic}[1]
\STATE \textbf{Initialization:} set $k \leftarrow 0$, initial training subset $\mathcal{Z}_{0} \leftarrow \emptyset$, sampling fraction $r_s=0.05$.
\REPEAT
\STATE $k \leftarrow k + 1$. 
\STATE Calculate the sample-weight vector $\mathbf{w}_k$ on the entire training set $\mathcal{Z}$ using the APW method.
\STATE Sample a subset $\mathcal{M}_k \subseteq \mathcal{Z}$ of size $|\mathcal{M}_k|=\left\lfloor r_s |\mathcal{Z}| \right\rfloor$ with sampling probabilities given by $\mathbf{w}_k$.
\STATE $\mathcal{Z}_k \leftarrow \mathcal{Z}_{k-1} \cup \mathcal{M}_k$
\STATE Train the network on $\mathcal{Z}_k$ using the average training loss.
\UNTIL{$|\mathcal{Z} \setminus \mathcal{Z}_k| < \left\lfloor r_s |\mathcal{Z}| \right\rfloor$} 
\STATE Train the network on the entire training set $\mathcal{Z}$ in all subsequent training epochs.
\end{algorithmic}
\end{algorithm}

In Algorithm~\ref{alg:S-APW}, when $\mathcal{Z}_k$ is still a proper subset of $\mathcal{Z}$, APW only adopts the E-weighting approach because I-weighting and EI-weighting are not applied to the sampled subset, and the network parameters are updated by minimizing the average training loss rather than the reweighted loss ({\it cf.} Eq.~\eqref{eq:APW_loss}).

In the early training phase, the easy samples have large weights. Therefore, these samples are more likely to be sampled from the entire training set $\mathcal{Z}$ and added to the current training subset $\mathcal{Z}_k$. In the later training phase, the weights of hard samples are increased, and these samples are gradually incorporated into $\mathcal{Z}_k$ until the number of remaining samples $|\mathcal{Z} \setminus \mathcal{Z}_k|$ becomes smaller than $\lfloor r_s |\mathcal{Z}| \rfloor$. The subsequent network training is then performed on the entire dataset $\mathcal{Z}$. We can apply one of the three weighting approaches (E-weighting, I-weighting, and EI-weighting) and update the network parameters by minimizing the reweighted loss ({\it cf.} Appendix~\ref{appendix:implementation:approaches}).

\subsection{Mixup-APW (M-APW)}\label{appendix:implementation:M-APW}

Let $C$ denote the number of classes, and let $\{(\mathbf{x}^{(n)}, \mathbf{y}^{(n)})\}_{n=1}^{N}$ be the training set, where each label $\mathbf{y}^{(n)}$ is one-hot encoded, {\it i.e.,} $\mathbf{y}^{(n)} \in \{0,1\}^C$ with $\mathbf{y}^{(n)} = [y^{(n)}_{1},\ldots,y^{(n)}_{C}]^\top$ and $\sum_{c=1}^C y^{(n)}_{c} = 1$. Given two samples $(\mathbf{x}^{(i)},\mathbf{y}^{(i)})$ and $(\mathbf{x}^{(j)},\mathbf{y}^{(j)})$, the standard mixup method forms a mixed example from the two samples by sampling $\lambda \sim \mathrm{Beta}(\alpha_{\text{mix}},\alpha_{\text{mix}})$ and setting
\begin{equation*}
\widetilde{\mathbf{x}} = \lambda \mathbf{x}^{(i)} + (1-\lambda) \mathbf{x}^{(j)}, \qquad
\widetilde{\mathbf{y}} = \lambda \mathbf{y}^{(i)} + (1-\lambda) \mathbf{y}^{(j)}.
\end{equation*}
Let $f(\mathbf{x})$ denote the model prediction for an input $\mathbf{x}$, and the corresponding cross-entropy loss is
\begin{equation*}
\mathrm{L} = \mathrm{CE}(\mathbf{y},f(\mathbf{x})).
\end{equation*}
For the mixed example $(\widetilde{\mathbf{x}},\widetilde{\mathbf{y}})$, the mixup loss can be written as
\begin{equation}\label{eq:mixup_CE}
\widetilde{\mathrm{L}} = \mathrm{CE}(\widetilde{\mathbf{y}},f(\widetilde{\mathbf{x}})) = \lambda\,\mathrm{CE}(\mathbf{y}^{(i)},f(\widetilde{\mathbf{x}})) + (1-\lambda)\,\mathrm{CE}(\mathbf{y}^{(j)},f(\widetilde{\mathbf{x}})).
\end{equation}
The mixing coefficient $\lambda$ is randomly sampled from the distribution $\mathrm{Beta}(\alpha_{\text{mix}},\alpha_{\text{mix}})$, which induces a convex interpolation of both inputs and labels and can improve robustness to label noise~\cite{zhang2017mixup,li2020dividemix,peng2020suppressing,cordeiro2021propmix}. We propose M-APW by replacing the mixup interpolation coefficient with the normalized APW-assigned weights of the two samples, and accordingly Eq.~\eqref{eq:mixup_CE} is replaced with the M-APW mixup loss
\begin{equation}\label{eq:mixup_weighted_CE}
\widetilde{\mathrm{L}}^{\mathrm{APW}}  = \frac{w^{(i)}}{w^{(i)}+w^{(j)}} \,\mathrm{CE}(\mathbf{y}^{(i)},f(\widetilde{\mathbf{x}})) + \frac{w^{(j)}}{w^{(i)}+w^{(j)}} \,\mathrm{CE}(\mathbf{y}^{(j)},f(\widetilde{\mathbf{x}})),
\end{equation}
where $w^{(i)}$ and $w^{(j)}$ are the APW-assigned weights of samples $(\mathbf{x}^{(i)},\mathbf{y}^{(i)})$ and $(\mathbf{x}^{(j)},\mathbf{y}^{(j)})$, respectively. Next, we replace the reweighted mini-batch loss ({\it cf.} Eq.~\eqref{eq:mini_batch_APW_loss}) with the mini-batch average of the M-APW mixup losses, where each mixup loss is computed using Eq.~\eqref{eq:mixup_weighted_CE}.

M-APW combines the idea of CL with mixup augmentation and thus the mixed examples follow the easy-to-hard training paradigm. In the early training phase, the mixed examples are biased toward easy samples, whereas hard samples receive less emphasis. In the later training phase, M-APW gradually increases the emphasis on hard samples, leading to a more balanced emphasis across all training samples. 

\subsection{Ablation Study on M-APW}\label{appendix:implementation:Ablation-M-APW}

Next, we conduct an ablation study on M-APW. As a supplement to Section~\ref{section:experiments:comparison}, we use the same experimental settings as in Table~\ref{tab:comparison_CL} and compare standard mixup and M-APW using the metrics T-ACC and E-Prop$^\dagger$. Table~\ref{tab:mixup_details} reports the results. It can be observed that the standard mixup consistently improves T-ACC over the vanilla baseline, but it reduces E-Prop$^\dagger$ on ConvNeXt and MaxViT. This degradation may stem from label noise, which can lead to noisy mixed targets and thus hurt the model's prediction confidence. In contrast, M-APWs achieve the best T-ACC and E-Prop$^\dagger$ in almost all settings, except with Mini-WebVision using MaxViT. These results suggest that M-APW combines mixup's robustness to label noise, improving predictive accuracy, with APW's ability to enhance prediction confidence.

\begin{table*}[htbp]
\centering
\caption{Experimental results on three datasets (CIFAR-10N, CIFAR-100N, and Mini-WebVision) trained using four backbone models (ResNet-18, EfficientNet, ConvNeXt, and MaxViT). We compare standard mixup and M-APW using the metrics T-ACC and E-Prop$^\dagger$. In each row, the entries marked with {\setlength{\fboxsep}{3pt}\colorbox{tablemidblue}{blue}} represent the maximum value, and the entries marked with {\setlength{\fboxsep}{3pt}\colorbox{tablemidgreen}{green}} represent the minimum value.}\label{tab:mixup_details}
\begin{threeparttable}
\resizebox{0.93\linewidth}{!}{
\begin{tabular}{l|c|c|c|c|ccc}
\hline
\multicolumn{1}{l|}{}                                  & \multicolumn{1}{l|}{}                                & \multicolumn{1}{c|}{}                          & \multicolumn{1}{c|}{}                          & \multicolumn{1}{c|}{}                          & \multicolumn{3}{c}{M-APWs} \\
\cline{6-8}
\multicolumn{1}{l|}{\multirow{-2}{*}{Dataset}}         & \multicolumn{1}{l|}{\multirow{-2}{*}{Models}}        & \multicolumn{1}{c|}{\multirow{-2}{*}{Metrics}} & \multicolumn{1}{c|}{\multirow{-2}{*}{Vanilla}} & \multicolumn{1}{c|}{\multirow{-2}{*}{\makecell[c]{Standard \\ mixup}}}   & \multicolumn{1}{c}{M-APW-E} & \multicolumn{1}{c}{M-APW-I} & \multicolumn{1}{c}{M-APW-EI} \\
\hline

% =========================
% ======== CIFAR-10N ======
% =========================
\multirow{8}{*}{\makecell[l]{CIFAR-10N}}
& \multicolumn{1}{l|}{}                                & T-ACC            & \cellcolor{tablemidgreen}{77.37} & 80.49 & 82.36 & 82.24 & \cellcolor{tablemidblue}{82.39} \\
& \multicolumn{1}{l|}{\multirow{-2}{*}{ResNet-18}}     & E-Prop$^\dagger$ & \cellcolor{tablemidgreen}{51.17} & 52.83 & 60.73 & 60.56 & \cellcolor{tablemidblue}{60.99} \\ \cline{2-8}

& \multicolumn{1}{l|}{}                                & T-ACC            & \cellcolor{tablemidgreen}{83.45} & 85.23 & \cellcolor{tablemidblue}{86.74} & 86.63 & \cellcolor{tablemidblue}{86.74} \\
& \multicolumn{1}{l|}{\multirow{-2}{*}{EfficientNet}}  & E-Prop$^\dagger$ & \cellcolor{tablemidgreen}{61.42} & 64.70 & 73.18 & 72.92 & \cellcolor{tablemidblue}{73.28} \\ \cline{2-8}

& \multicolumn{1}{l|}{}                                & T-ACC            & \cellcolor{tablemidgreen}{86.64} & 88.35 & 89.87 & 89.81 & \cellcolor{tablemidblue}{89.95} \\
& \multicolumn{1}{l|}{\multirow{-2}{*}{ConvNeXt}}      & E-Prop$^\dagger$ & 64.98 & \cellcolor{tablemidgreen}{64.14} & 71.83 & 71.56 & \cellcolor{tablemidblue}{71.98} \\ \cline{2-8}

& \multicolumn{1}{l|}{}                                & T-ACC            & \cellcolor{tablemidgreen}{88.44} & 89.39 & 89.90 & 89.76 & \cellcolor{tablemidblue}{89.96} \\
& \multicolumn{1}{l|}{\multirow{-2}{*}{MaxViT}}        & E-Prop$^\dagger$ & 69.06 & \cellcolor{tablemidgreen}{65.26} & 70.57 & 70.26 & \cellcolor{tablemidblue}{70.95} \\ \hline

% ==========================
% ======== CIFAR-100N ======
% ==========================
\multirow{8}{*}{\makecell[l]{CIFAR-100N}}
& \multicolumn{1}{l|}{}                                & T-ACC            & \cellcolor{tablemidgreen}{57.29} & 59.10 & \cellcolor{tablemidblue}{59.37} & 59.35 & 59.28 \\
& \multicolumn{1}{l|}{\multirow{-2}{*}{ResNet-18}}     & E-Prop$^\dagger$ & \cellcolor{tablemidgreen}{34.24} & 36.04 & \cellcolor{tablemidblue}{39.16} & 39.09 & 38.21 \\ \cline{2-8}

& \multicolumn{1}{l|}{}                                & T-ACC            & \cellcolor{tablemidgreen}{62.98} & 64.57 & \cellcolor{tablemidblue}{65.27} & 65.21 & 65.16 \\
& \multicolumn{1}{l|}{\multirow{-2}{*}{EfficientNet}}  & E-Prop$^\dagger$ & \cellcolor{tablemidgreen}{40.77} & 45.66 & 51.31 & \cellcolor{tablemidblue}{51.42} & 50.55 \\ \cline{2-8}

& \multicolumn{1}{l|}{}                                & T-ACC            & \cellcolor{tablemidgreen}{67.12} & 68.53 & \cellcolor{tablemidblue}{69.48} & 69.47 & 69.46 \\
& \multicolumn{1}{l|}{\multirow{-2}{*}{ConvNeXt}}      & E-Prop$^\dagger$ & \cellcolor{tablemidgreen}{45.02} & 46.08 & \cellcolor{tablemidblue}{51.74} & 51.72 & 50.32 \\ \cline{2-8}

& \multicolumn{1}{l|}{}                                & T-ACC            & \cellcolor{tablemidgreen}{67.12} & 68.63 & \cellcolor{tablemidblue}{69.15} & 68.96 & 68.98 \\
& \multicolumn{1}{l|}{\multirow{-2}{*}{MaxViT}}        & E-Prop$^\dagger$ & 46.27 & \cellcolor{tablemidgreen}{44.28} & 47.11 & \cellcolor{tablemidblue}{47.13} & 46.27 \\ \hline

% ============================
% ======== mini-WebVision =====
% ============================
\multirow{8}{*}{\makecell[l]{Mini-WebVision}}
& \multicolumn{1}{l|}{}                                & T-ACC            & \cellcolor{tablemidgreen}{75.39} & 76.58 & 77.12 & 77.09 & \cellcolor{tablemidblue}{77.18} \\
& \multicolumn{1}{l|}{\multirow{-2}{*}{ResNet-18}}     & E-Prop$^\dagger$ & 67.29 & \cellcolor{tablemidgreen}{65.37} & \cellcolor{tablemidblue}{68.90} & 68.86 & \cellcolor{tablemidblue}{68.90} \\ \cline{2-8}

& \multicolumn{1}{l|}{}                                & T-ACC            & \cellcolor{tablemidgreen}{82.52} & 83.01 & 83.12 & 83.15 & \cellcolor{tablemidblue}{83.18} \\
& \multicolumn{1}{l|}{\multirow{-2}{*}{EfficientNet}}  & E-Prop$^\dagger$ & \cellcolor{tablemidgreen}{76.43} & 76.92 & \cellcolor{tablemidblue}{79.78} & 79.76 & 79.61 \\ \cline{2-8}

& \multicolumn{1}{l|}{}                                & T-ACC            & \cellcolor{tablemidgreen}{85.97} & 86.43 & 86.67 & 86.71 & \cellcolor{tablemidblue}{86.78} \\
& \multicolumn{1}{l|}{\multirow{-2}{*}{ConvNeXt}}      & E-Prop$^\dagger$ & 81.83 & \cellcolor{tablemidgreen}{81.10} & \cellcolor{tablemidblue}{83.69} & 83.65 & 83.57 \\ \cline{2-8}

& \multicolumn{1}{l|}{}                                & T-ACC            & 86.74 & \cellcolor{tablemidblue}{86.90} & 86.71 & 86.77 & \cellcolor{tablemidgreen}{86.70} \\
& \multicolumn{1}{l|}{\multirow{-2}{*}{MaxViT}}        & E-Prop$^\dagger$ & \cellcolor{tablemidblue}{84.01} & 82.15 & 82.24 & 82.29 & \cellcolor{tablemidgreen}{82.07} \\ \hline

\end{tabular}%
}
\end{threeparttable}
\end{table*}

\clearpage

\section{Detailed Descriptions of Datasets and Experimental Settings}\label{appendix:descriptions}

In this appendix, we provide detailed descriptions of the datasets and experimental settings used in Section~\ref{section:experiments}, including the network architectures and optimizer settings.

\subsection{Datasets}\label{appendix:descriptions:datasets}

Table~\ref{tab:datasets} lists the datasets and indicates whether label noise is inherent or artificially added. CIFAR-10/100, CIFAR-10N/100N, Mini-WebVision, and WebVision are image classification datasets. For CIFAR-10/100 and CIFAR-10N/100N, we split the original training set into training and validation sets with a 9:1 ratio. For WebVision and its subset Mini-WebVision, since test labels are not publicly available, we report results on their validation sets. For the NLP dataset RTE, since test labels are not publicly available, we report results on the validation set. For the graph dataset NCI1, since there is no standard train/test split, we randomly divide the dataset into training, validation, and test sets with an 8:1:1 ratio. We describe the datasets in detail below.

\begin{table*}[htbp]
\centering
\renewcommand{\arraystretch}{1.1}
\caption{The characteristics of the classification datasets used in our experiments. In the ``Contain label noise'' column, a ``\cmark'' indicates that the dataset includes inherent label noise. The percentage in parentheses is the reported noise rate, and an asterisk (``*'') denotes an unknown proportion. In the ``Add synthetic label noise'' column, a ``\cmark'' indicates that we report results on both the clean version and the noisy version with 40\% synthetic label noise. In the ``Num. of test'' column, ``-'' indicates that the test labels are not publicly available.}\label{tab:datasets}
%\vspace{0.1in}
\resizebox{\linewidth}{!}{
\begin{tabular}{l|lllllll}
\toprule[1.2pt]
Dataset          & Contain label noise & Add synthetic label noise    & Num. of training         & Num. of validation     & Num. of test            & Num. of classes  \\ \midrule
CIFAR-10         & \xmark              & \cmark                       & $4.5\times10^4$          & $5\times10^3$          & $1\times10^4$           & 10 \\
CIFAR-100        & \xmark              & \cmark                       & $4.5\times10^4$          & $5\times10^3$          & $1\times10^4$           & 100 \\ \midrule
RTE              & \xmark              & \cmark                       & $\approx2.5\times10^3$   & $\approx2.8\times10^2$ & -                       & 2 \\ 
NCI1             & \xmark              & \cmark                       & $\approx3.3\times10^3$   & $\approx4.1\times10^2$ & $\approx4.1\times10^2$  & 2 \\ \midrule
CIFAR-10N-worse  & \cmark ($40.21 \%$) & \xmark                       & $4.5\times10^4$          & $5\times10^3$          & $1\times10^4$           & 10 \\
CIFAR-100N-fine  & \cmark ($40.20 \%$) & \xmark                       & $4.5\times10^4$          & $5\times10^3$          & $1\times10^4$           & 100 \\
Mini-WebVision   & \cmark ($*$)        & \xmark                       & $\approx6.59\times10^4$  & $2.5\times10^3$        & -                       & 50 \\
WebVision        & \cmark ($*$)        & \xmark                       & $\approx2.44\times10^6$  & $5\times10^4$          & -                       & 1000 \\
\bottomrule
\end{tabular}}
\end{table*}

{\bf CIFAR-10 and CIFAR-100}~{\normalfont\cite{krizhevsky2009learning}.}
CurML provides a fair benchmark across different CL methods~\cite{zhou2022curml} and uses CIFAR-10 with and without 40\% synthetic label noise. Specifically, the synthetic label noise is applied to the training set by independently changing the label of every training sample, with probability 40\%, to a different class label sampled uniformly from the remaining classes. Following CurML, we use the same network architecture and optimizer settings, and further consider both CIFAR-10 and CIFAR-100, with and without 40\% synthetic label noise.

{\bf RTE}~{\normalfont\cite{wang2018glue}} {\bf and NCI1}~{\normalfont\cite{morris2020tudataset}.}
CurBench~\cite{zhou2024curbench} summarizes the representative CL methods from CurML and extends the benchmark used in CurML from image classification datasets to NLP and graph datasets. To further align our experimental setup with CurBench, we follow its showcased examples by fine-tuning BERT on RTE and training a GCN on NCI1. The datasets RTE and NCI1 originally have clean labels. Following the implementation in CurBench, we conduct experiments with and without 40\% synthetic label noise.

{\bf CIFAR-10N and CIFAR-100N}~{\normalfont\cite{wei2021learning}.} 
CIFAR-10N and CIFAR-100N are noisy variants of CIFAR-10 and CIFAR-100 obtained via human re-annotation on Amazon Mechanical Turk. We use two versions from~\cite{wei2021learning}: 1) CIFAR-10N-worse, where each image has three human annotations, and if any annotation is incorrect, the noisy label is sampled uniformly from the incorrect annotations (otherwise the clean label is kept); and 2) CIFAR-100N-fine, where each image is annotated by a single worker who first selects one of 20 super-classes and then chooses a fine class within that super-class. We set $p_{\text{noise}}$ to the reported noise rate for each dataset (see Table~\ref{tab:datasets}). For brevity, we refer to CIFAR-10N-worse and CIFAR-100N-fine as CIFAR-10N and CIFAR-100N, respectively, throughout the main text.

{\bf WebVision and Mini-WebVision}~{\normalfont\cite{li2017WebVision}.}
WebVision is a large-scale webly supervised dataset collected from Flickr and Google Images using category-level search queries derived from the 1,000 classes of ILSVRC 2012. Its training labels come directly from the queries rather than image-level human annotations, which introduces substantial label noise ({\it e.g.,} off-topic retrievals, ambiguous images, or cases where the queried concept is absent or not salient). Prior work estimates the noise rate of WebVision to be around 20\%~\cite{lu2022noise}. Mini-WebVision is a smaller subset that includes the Google Images partition for the first 50 classes of ILSVRC 2012, and is commonly used as a lightweight benchmark for rapidly assessing the performance of CL algorithms under label noise. Recent analysis reports that Mini-WebVision contains 70.30\% clean labels, 5.32\% in-distribution noise, and 24.38\% out-of-distribution noise~\cite{albert2022addressing}. Since both datasets contain at least around 20\% noisy labels, we conservatively set the assumed noise rate to $p_{\text{noise}} = 0.2$ for both WebVision and Mini-WebVision.

\subsection{Experimental Settings}\label{appendix:descriptions:settings}

We describe the network architectures and optimizer settings used for the datasets listed in Table~\ref{tab:datasets}.

For the experiments on CIFAR-10/100 with and without 40\% synthetic label noise, we follow the experimental settings of the fair benchmark in CurML~\cite{zhou2022curml}. Specifically, the network consists of two convolutional blocks followed by a fully connected classifier. Each block contains two $3\times3$ convolutional layers with 32 channels, each followed by batch normalization and ReLU activation, and a $2\times2$ max-pooling layer for spatial downsampling. The network is trained for 200 epochs using stochastic gradient descent (SGD) with an initial learning rate of 0.1, and a cosine learning rate schedule is used to decay the learning rate, with a minimum value of $1\times10^{-6}$.

For the experiments on RTE and NCI1, with and without 40\% synthetic label noise, we follow the experimental settings of the fair comparison in CurBench~\cite{zhou2024curbench}. For the experiments on RTE using BERT, we use the bert-base-uncased checkpoint~\cite{devlin2019bert}, and fine-tune the model for 3 epochs using AdamW with a constant learning rate of $2\times10^{-5}$. For the experiments on NCI1 using GCN~\cite{kipf2016semi}, the network consists of three graph convolutional layers followed by a graph-level linear classifier. Each layer uses 64 hidden channels, with ReLU applied after the first two layers. Node representations are aggregated via global mean pooling, after which dropout is applied before the classifier. The network is trained for 200 epochs using Adam with a constant learning rate of 0.01.

For the experiments on CIFAR-10N/100N and Mini-WebVision, we use four network architectures: ResNet-18~\cite{he2016deep}, EfficientNet-B0~\cite{tan2019efficientnet}, ConvNeXt-Tiny~\cite{liu2022convnet}, and MaxViT-T~\cite{tu2022maxvit}. These networks span residual CNNs (ResNet), modern CNN designs (EfficientNet and ConvNeXt), and hybrid transformer models (MaxViT). Each network is trained for 50 epochs using SGD with an initial learning rate of 0.01, momentum of 0.9, and a cosine learning rate schedule is used to decay the learning rate, with a minimum value of $1\times10^{-6}$. 

Following~\cite{chen2019understanding}, we adopt the baseline network architecture and optimizer settings for WebVision. The network is the Inception-ResNet-v2 backbone~\cite{szegedy2017inception}. Specifically, the network is trained for 100 epochs using SGD with an initial learning rate of 0.02, momentum of 0.9, and a step learning rate schedule is used to decay the learning rate by a factor of 0.1 at the $50$-th epoch.

The combination of selecting the best checkpoint based on the validation set and applying learning-rate decay helps stabilize training and mitigate the negative effects of APW assigning excessively large weights to persistent hard samples.

\clearpage

\section{Interpreting the Weighting Behavior of APW}\label{appendix:APW}

In this appendix, we use the LR model on a linearly separable dataset as an illustrative example to provide intuition for the weighting behavior of APW. We first visualize the weighting dynamics and then explain how APW assigns and updates sample weights during training. Our analysis shows that the error threshold $e$ in APW induces a confidence band around the separating hyperplane. This perspective further suggests that, by explicitly promoting a higher E-Prop rather than solely minimizing empirical risk, APW can lead to improved generalization. In addition, we investigate how different phase transitions (controlled by the phase threshold) affect the performance of APW.

\subsection{Illustrative Example Setup}\label{appendix:APW:example}

We generate 600 samples with 2 features and 2 classes, where each class is drawn from a Gaussian cluster with standard deviation 1.5 and the class centers are sampled uniformly from $[-10,10]^2$. We repeat the sampling until the resulting dataset is linearly separable. The dataset is then split into 70\% for training and 30\% for testing. We use the standard unregularized LR model with the L-BFGS optimizer, setting the step size to 0.01, the gradient tolerance to $10^{-5}$, and the maximum number of iterations to 150. For the APW method (using the E-weighting approach), we set $q$ to one-tenth of the training set size and choose $e = 0.3$. We report E-Prop on both the training and test sets using this error threshold.

\begin{figure}[htbp]
\centering
\begin{subfigure}[b]{0.27\textwidth}
\centering
\includegraphics[width=\linewidth]{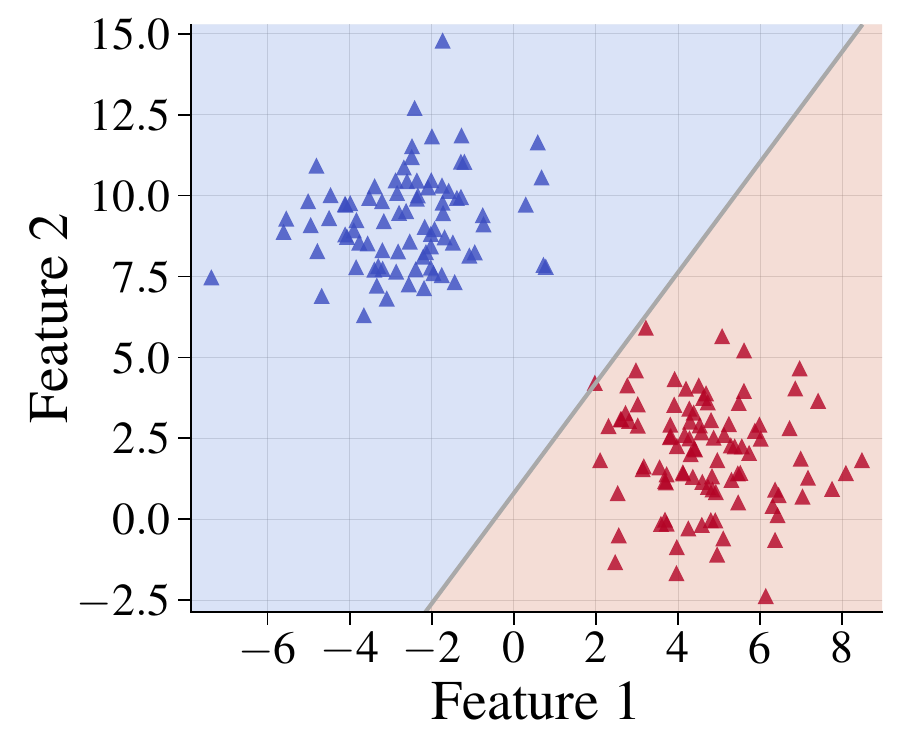}
\caption{Standard LR}\label{fig:LR_results:standard}
\end{subfigure}
\begin{subfigure}[b]{0.27\textwidth}
\centering
\includegraphics[width=\linewidth]{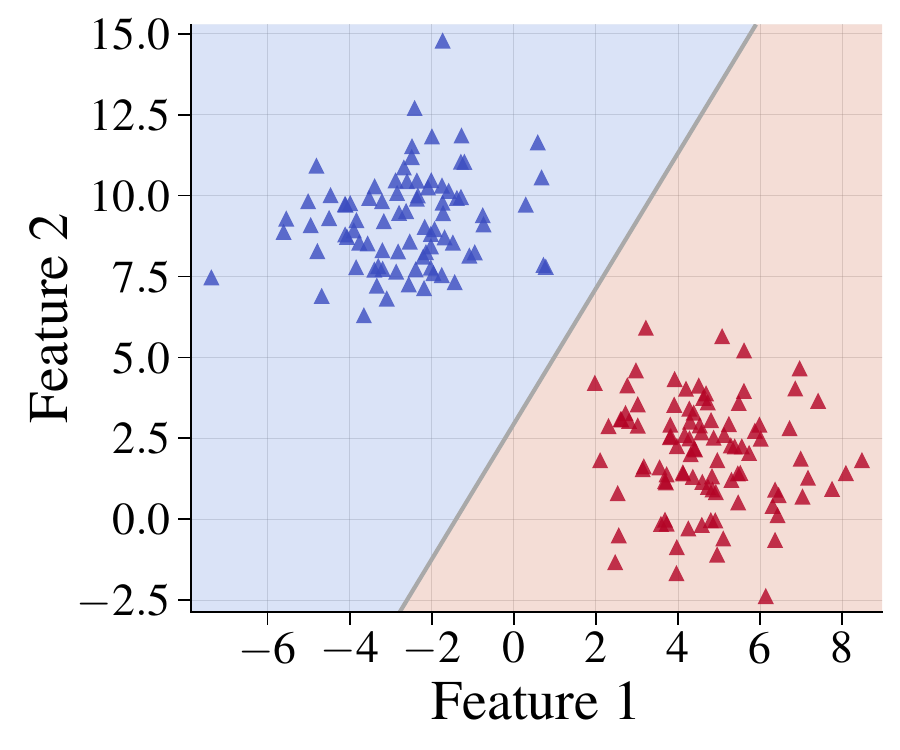}
\caption{APW + LR}\label{fig:LR_results:APW}
\end{subfigure}
\begin{subfigure}[b]{0.27\textwidth}
\centering
\includegraphics[width=\linewidth]{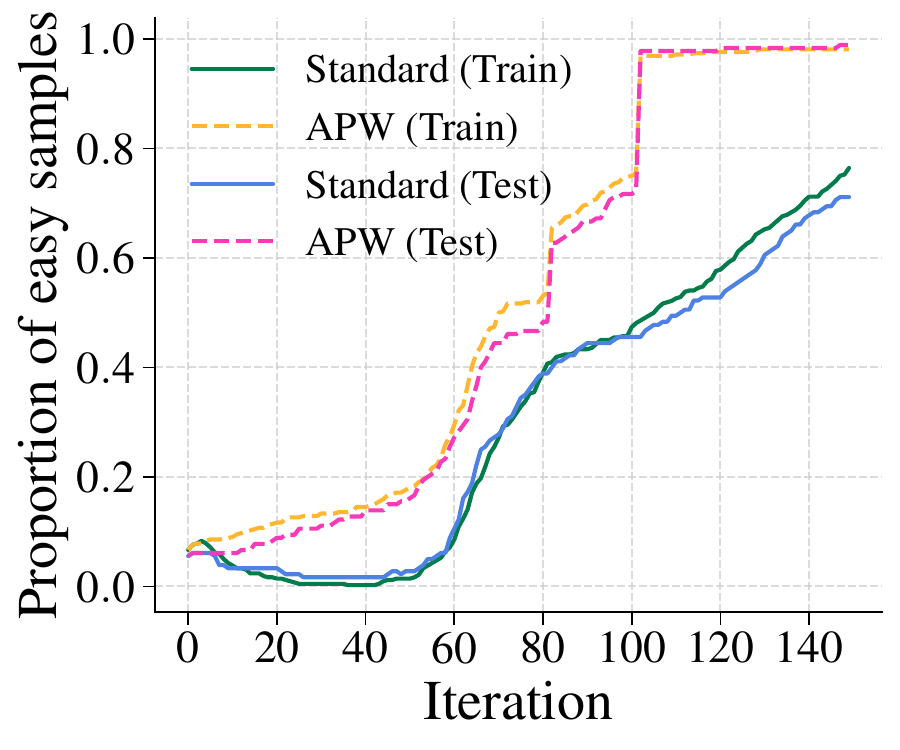}
\caption{E-Prop}\label{fig:LR_results:E_Prop}
\end{subfigure}
\caption{A diagram illustrating the generalization performance ({\it cf.} Theorem~\ref{theorem:generalization_error}) of applying the APW method. (a) An unregularized logistic regression (LR) model struggles when trained on a linearly separable dataset because the maximum-likelihood solution is attained only in the limit of an unbounded parameter norm; as a result, the hyperplane reached in finite iterations empirically generalizes poorly on the test set. For a detailed theoretical analysis, see Appendix~\ref{appendix:APW:band}. (b)-(c) Theorem~\ref{theorem:generalization_error} indicates that APW can enhance the E-Prop on the test set by improving that on the training set. In this manner, APW enables the LR model to obtain a better hyperplane that improves the generalization performance on the test set.}
\label{fig:LR_results}
\end{figure}

We provide an intuitive example to illustrate the core idea of APW. As shown in Figure~\ref{fig:LR_results}, with the error threshold on the sample loss of the LR model, APW quickly converts more training samples into easy samples during training and maintains this status well. This characteristic of APW helps the LR model learn a separating hyperplane with strong generalization performance. In the same spirit, we consider image classification. We introduce an error threshold on the cross-entropy loss of the training samples and show that increasing E-Prop tends to improve prediction accuracy and enhance prediction confidence ({\it cf.} Theorem~\ref{theorem:generalization_error}). The latter is crucial when training networks on real-world datasets.

\subsection{Error Threshold Induces a Confidence Band}\label{appendix:APW:band}

Consider a linearly separable training set $\{(\mathbf{x}_\ddagger^{(n)}, y_\ddagger^{(n)})\}_{n=1}^{N}$. To simplify the theoretical analysis, we append an additional dimension with value $1$ to each feature vector. There exists a vector $\omega_\star$ such that $y_\ddagger^{(n)} \omega_\star^\top \mathbf{x}_\ddagger^{(n)} > 0$ for all $n$, meaning that every training sample is correctly classified. For any scalar $A > 0$, the scaled vector $A\omega_\star$ also correctly classifies all training samples, and the logistic loss function is given by
\begin{equation*}
\mathrm{L}_\ddagger(A\omega_\star) = \sum_{n=1}^{N} \log\left(1 + \exp\left(-A \left(y_\ddagger^{(n)} \omega_\star^\top \mathbf{x}_\ddagger^{(n)}\right)\right)\right) 
\to 0 \quad \text{as } A \to \infty.
\end{equation*}
Thus, the infimum of the logistic loss is $0$, but it is not attained at any finite $\omega$, and the maximum-likelihood solution is achieved only when $\|\omega\| \to \infty$~\cite{albert1984existence,soudry2018implicit}. 

To address this challenge, one can impose explicit regularization on $\omega$ to obtain a finite solution~\cite{cessie1992ridge,heinze2002solution}, or instead rely on implicit regularization from optimization dynamics~\cite{hu2022early} or use a variable step size schedule~\cite{axiotis2023gradient}. APW provides an alternative way to constrain the choice of the separating hyperplane by introducing an error threshold $e < \log(2)$ and encouraging more training samples to be marked as easy, which can increase E-Prop on the test set. If a sample $(\mathbf{x}_\ddagger, y_\ddagger)$ can be marked as an easy sample, its logistic loss is below the error threshold $e$:
\begin{equation*}
\log\left(1+\exp\left(-y_\ddagger\,\omega^\top \mathbf{x}_\ddagger\right)\right) \leq e \quad \iff \quad y_\ddagger\,\omega^\top \mathbf{x}_\ddagger \geq -\log\left(\exp({e})-1\right)>0.
\end{equation*}
Geometrically, starting from the current hyperplane $\omega^\top \mathbf{x}_\ddagger = 0$, we can draw two parallel hyperplanes
\begin{equation*}
\omega^\top \mathbf{x}_\ddagger = \pm M_e \;\; \text{with } M_e = -\log(\exp(e)-1)>0,
\end{equation*}
which are symmetric with respect to the hyperplane $\omega^\top \mathbf{x}_\ddagger = 0$. The region between these two hyperplanes can be interpreted as a confidence band. The samples that are correctly classified ({\it i.e.,} $y_\ddagger\,\omega^\top \mathbf{x}_\ddagger > 0$) and lie outside this confidence band are exactly those whose logistic loss does not exceed the error threshold $e$, namely, the easy samples.

The APW method introduces an error threshold $e$ and makes the LR model focus on easy samples in the early training phase to quickly identify an approximate separating direction, and subsequently shifts focus to turning hard samples into easy ones in the later training phase. In this manner, APW induces a confidence band for the LR model and effectively improves the generalization performance of the learned hyperplane. This confidence band interpretation is consistent with our theoretical findings in Section~\ref{section:theoretical_analysis} that APW rapidly increases E-Prop on the training set and then improves E-Prop on the test set.

\subsection{Weighting Dynamics Visualization}

In the main text, Figure~\ref{fig:LR_results} compares the learned hyperplanes and E-Prop when training the LR model with and without APW on the linearly separable dataset. Here, we further analyze the weighting behavior of APW. 

As shown in Figure~\ref{fig:LR_results_detail}, each row shows three snapshots at the same training iterations to illustrate the evolution of the hyperplane. Comparing Figure~\ref{fig:LR_results_detail:standard} and Figure~\ref{fig:LR_results_detail:APW}, APW assigns larger weights to the easy samples (the left column of Figure~\ref{fig:LR_results_detail:APW}), while in the later phase, it gradually shifts the LR model's attention to the hard samples (the middle column of Figure~\ref{fig:LR_results_detail:APW}). As a result, the final hyperplane provides a better partition of the training data (the right column of Figure~\ref{fig:LR_results_detail:APW}). 

Figure~\ref{fig:LR_results_detail:APW_margin} illustrates how APW induces a confidence band for the LR model and encourages the training process to convert more hard samples into easy samples. As training progresses, the confidence band associated with the final hyperplane contains only a few hard samples, so that the resulting hyperplane better exploits the information in the training data and exhibits improved generalization compared with the standard LR model.

\begin{figure}[htbp]
\centering
% ======================= Standard LR =======================
\begin{subfigure}{\linewidth}
\centering
\includegraphics[width=0.27\linewidth]{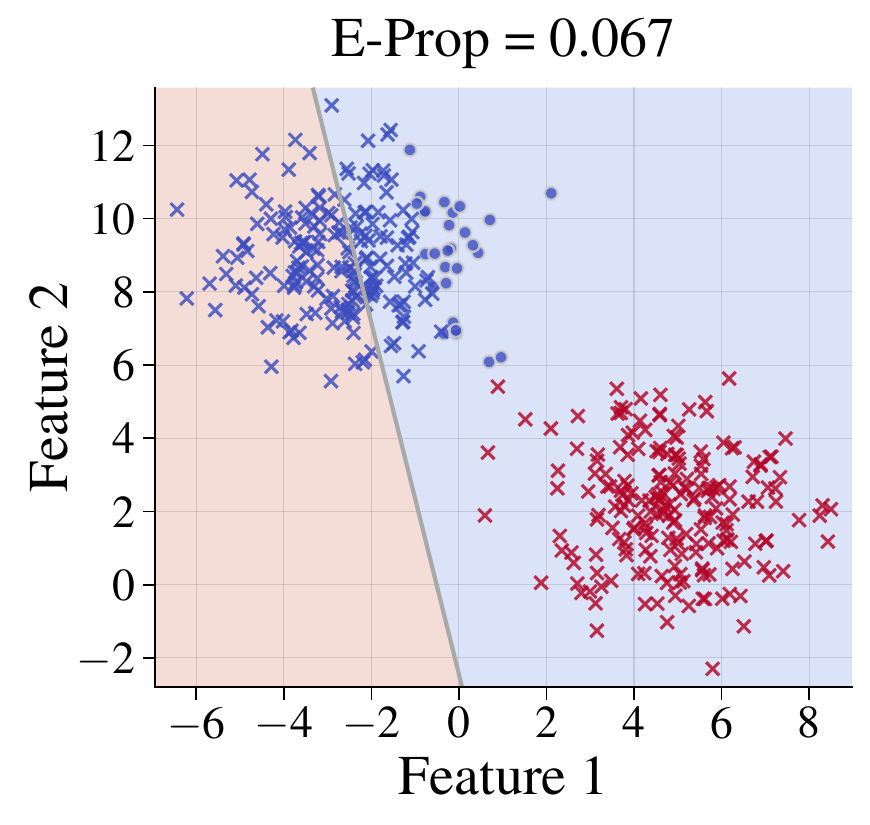}
\includegraphics[width=0.27\linewidth]{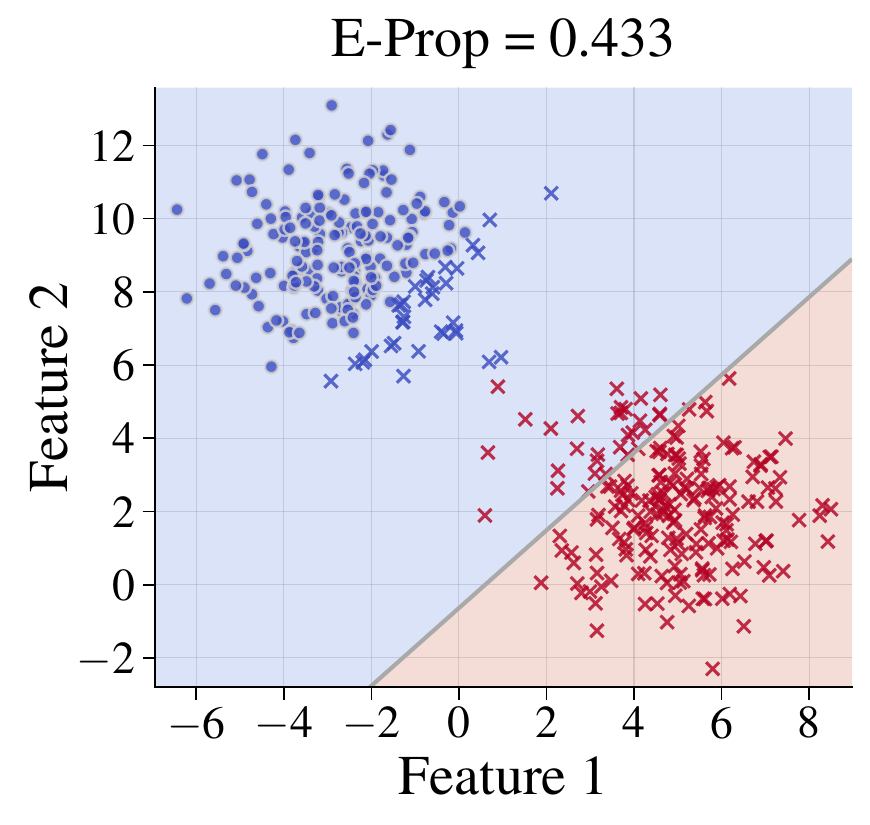}
\includegraphics[width=0.27\linewidth]{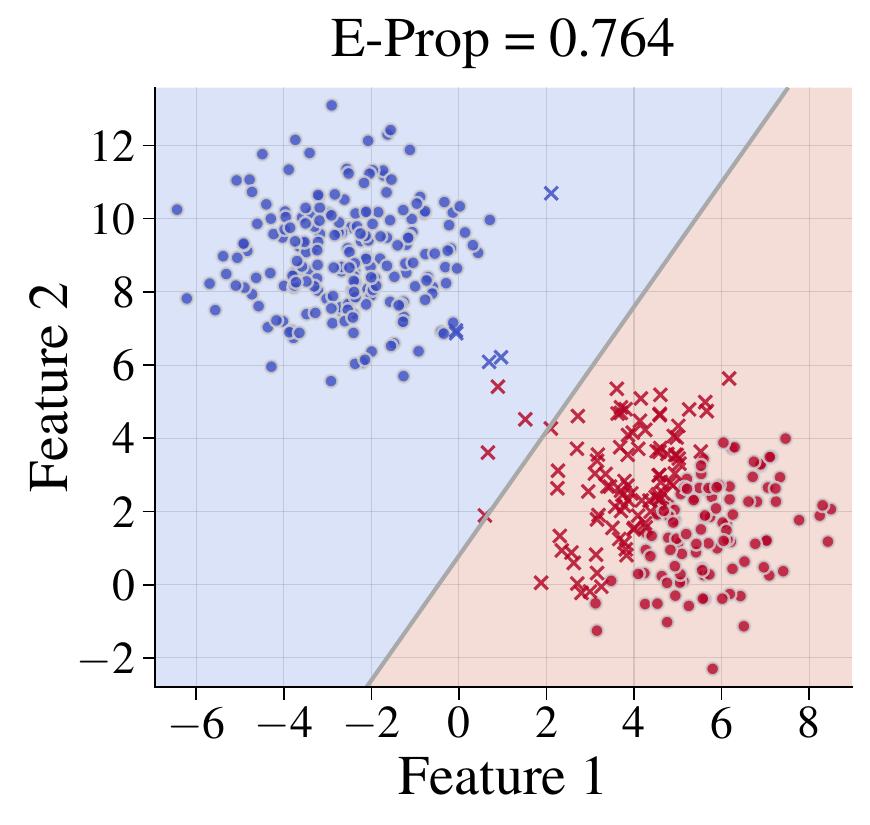}
\caption{Evolution of the hyperplane of the standard LR model.}
\label{fig:LR_results_detail:standard}
\end{subfigure}
%\vspace{0.1cm}
% ======================= APW LR =======================
\begin{subfigure}{\linewidth}
\centering
\includegraphics[width=0.27\linewidth]{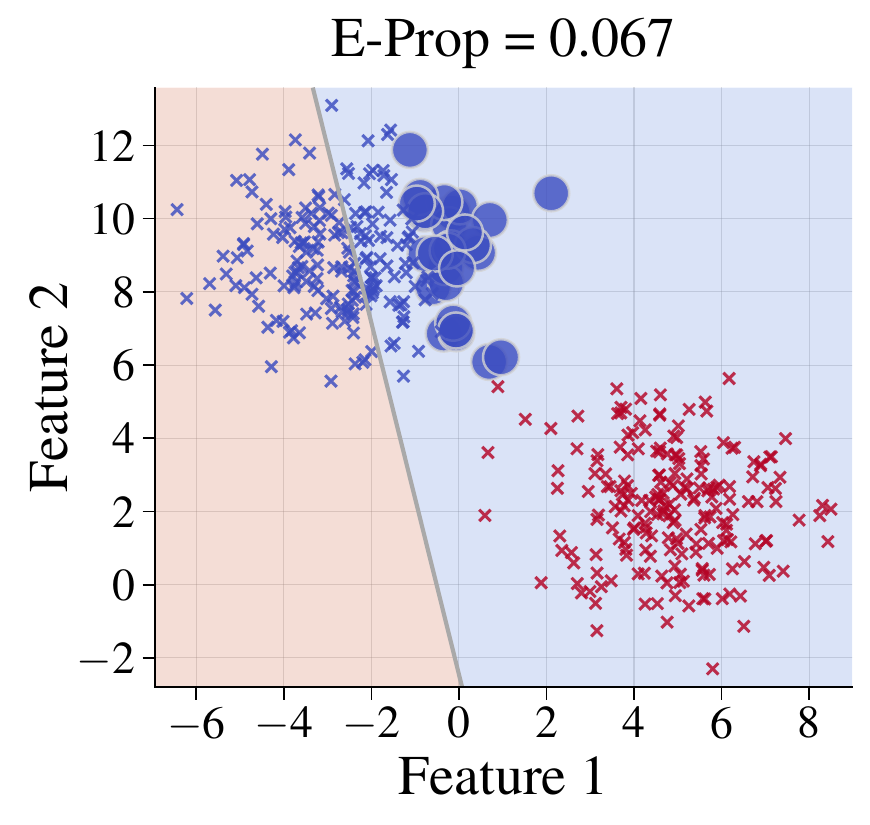}
\includegraphics[width=0.27\linewidth]{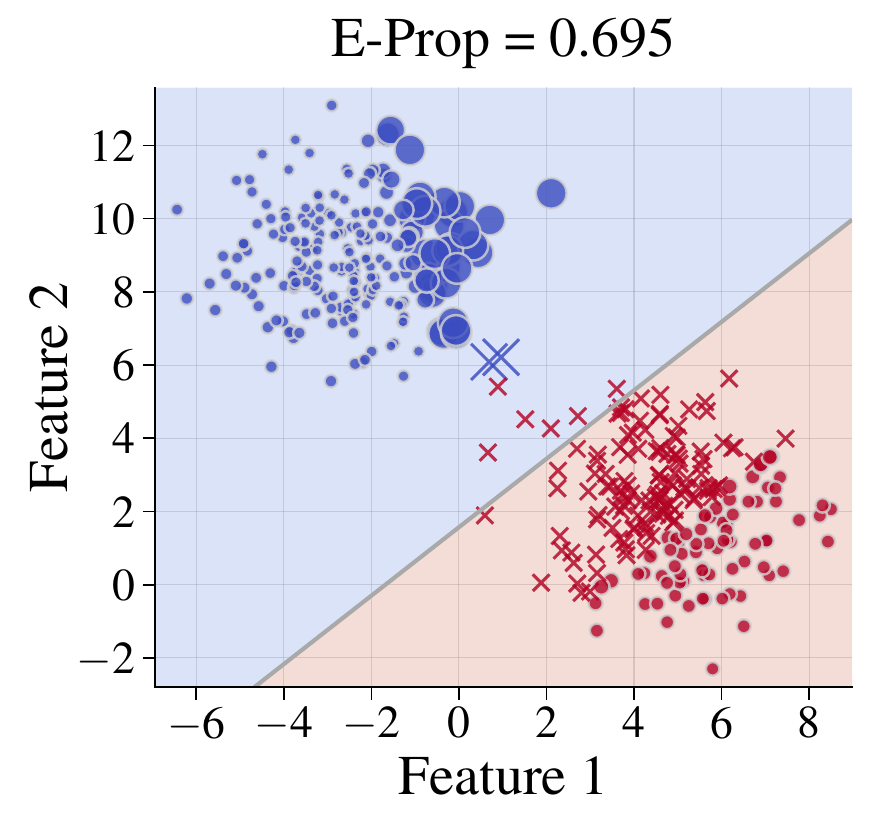}
\includegraphics[width=0.27\linewidth]{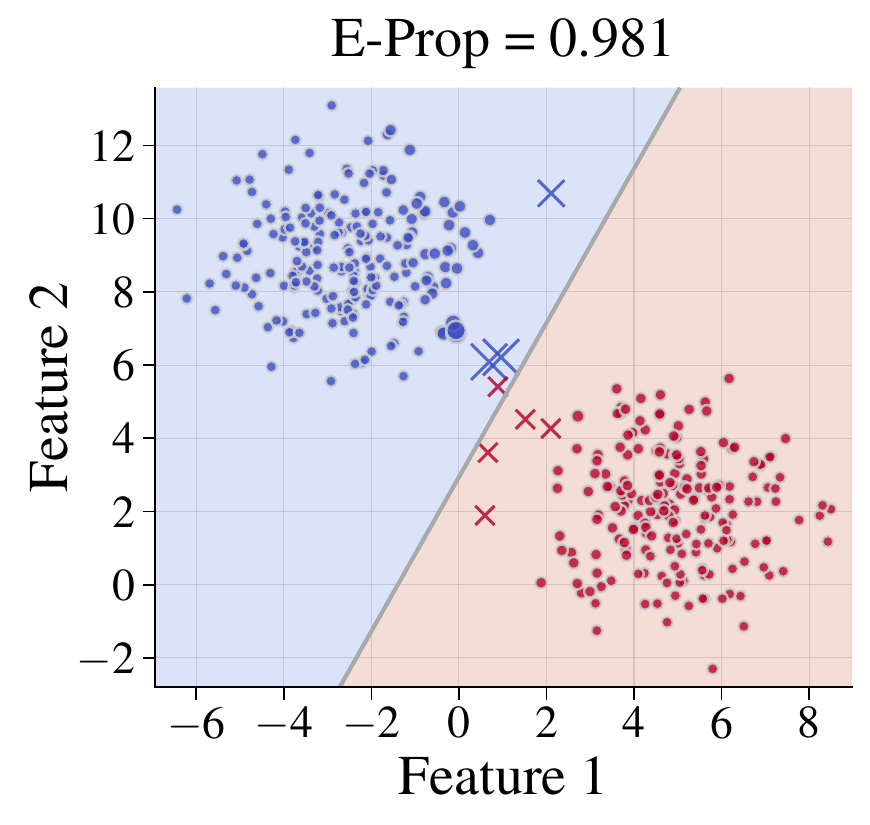}
\caption{Evolution of the hyperplane of the LR model trained with APW.}
\label{fig:LR_results_detail:APW}
\end{subfigure}
%\vspace{0.1cm}
% ======================= APW LR margin =======================
\begin{subfigure}{\linewidth}
\centering
\includegraphics[width=0.27\linewidth]{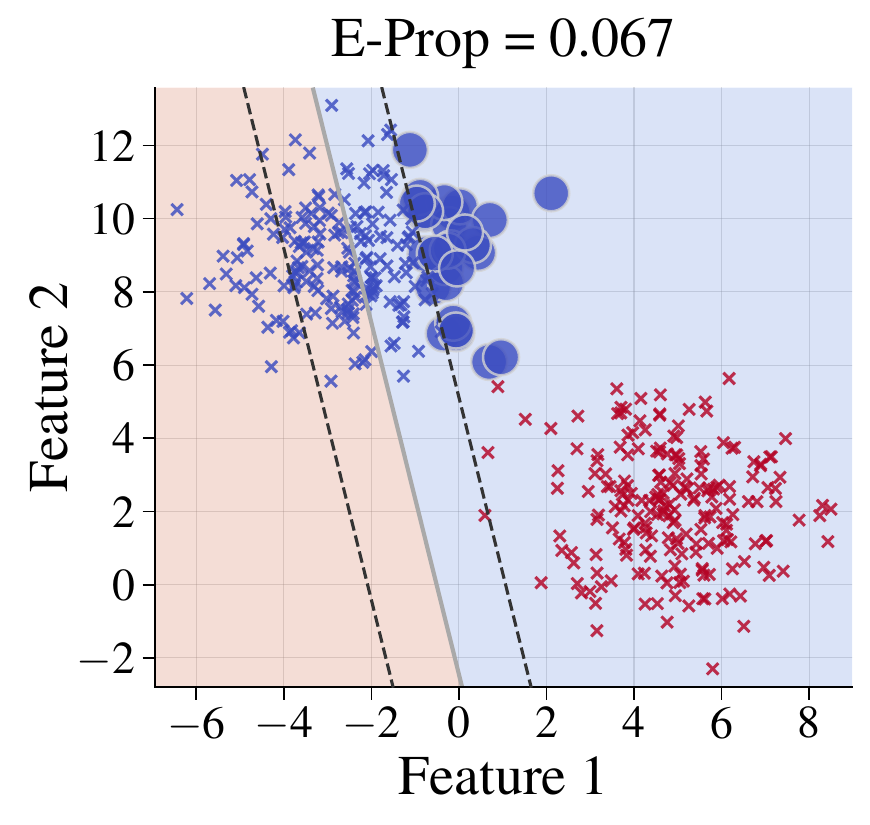}
\includegraphics[width=0.27\linewidth]{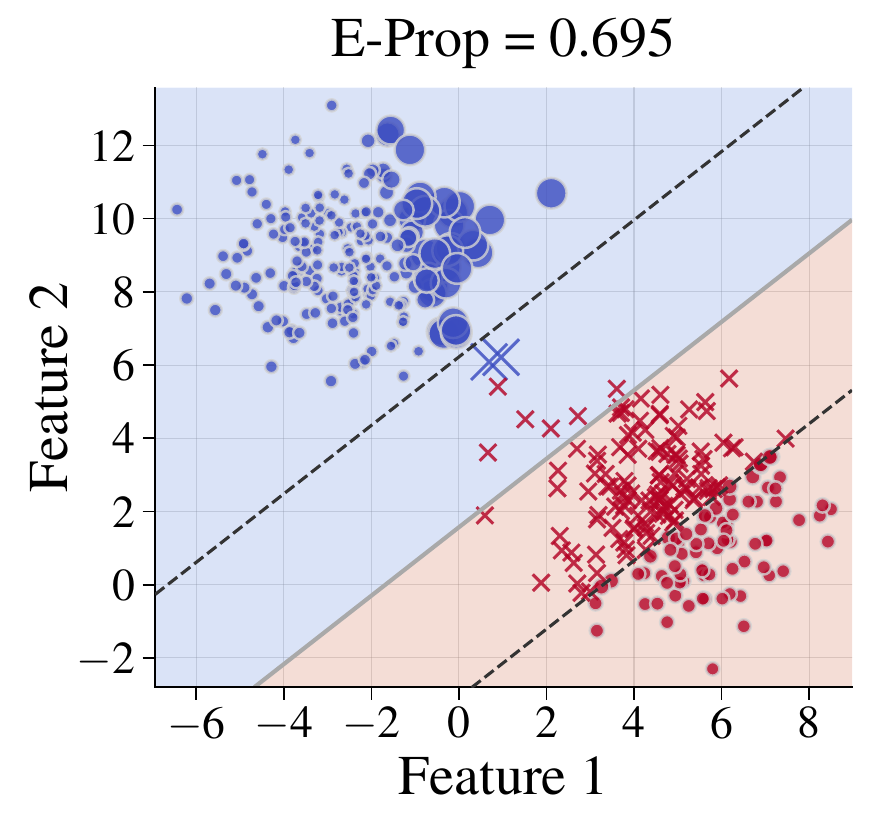}
\includegraphics[width=0.27\linewidth]{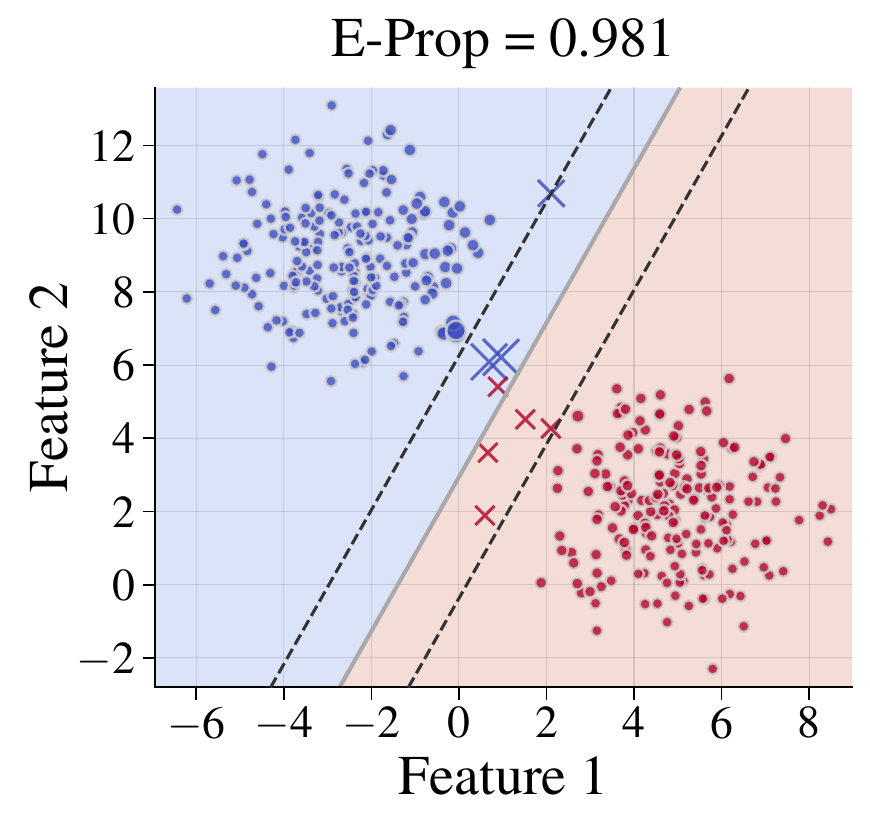}
\caption{The error threshold $e$ in APW induces a confidence band for the hyperplane.}
\label{fig:LR_results_detail:APW_margin}
\end{subfigure}
\caption{Comparison of the hyperplanes of the standard LR model and the LR model trained with APW. The plots in each column are taken at the same training iteration for both methods. In each plot, blue and red markers denote training samples from the two classes; circles indicate easy samples and crosses indicate hard samples. In the second row, the marker size is proportional to the APW-assigned weight of each sample. In the third row, the region between the black dashed lines represents the confidence band, and correctly classified samples that lie outside this confidence band are marked as easy samples.}
\label{fig:LR_results_detail}
\end{figure}

\subsection{Phase Threshold $\tau$ and Phase Transition}

By introducing a phase threshold $\tau$, the relationship between $\rho_k$ and $\frac{1}{2}$ in Eq.~\eqref{eq:alpha_k} can be generalized to the relationship between $\rho_k$ and any $\tau$ given by
\begin{equation}
\alpha_k^\ddagger
= \frac{1}{q} \log\!\left( \frac{(1-\rho_k)\,\tau}{\rho_k(1-\tau)} \right)
= \frac{1}{q} \left[ \log \frac{1-\rho_k}{\rho_k} + \log \frac{\tau}{1-\tau} \right].
\end{equation}
We denote APW with the phase threshold $\tau$ as APW$^\ddagger$. Correspondingly, the weight update scheme in Table~\ref{tab:relation} is generalized to that in Table~\ref{tab:relation_tau}. As training progresses, once $\rho_k$ decreases below $\tau$, the training state transitions from the early phase to the later phase. However, a small $\tau$ leads to a late transition, and APW$^\ddagger$ may cause the model to overemphasize easy samples, which impedes learning from more informative examples. In contrast, a large $\tau$ leads to an early transition, and APW$^\ddagger$ may cause the model to overemphasize hard samples, which makes optimization more difficult.

\begin{table}[htbp]
\centering   
\caption{Weight update in the early and later training phases for APW$^\ddagger$.}\label{tab:relation_tau}
\resizebox{0.8\linewidth}{!}{
\begin{tabular}{c|c|c}
\toprule[1.2pt]
Training phase & Early training phase ($\rho_k > \tau$) & Later training phase ($\rho_k < \tau$) \\
\midrule
Sign of $\alpha^\ddagger_k$ & $\alpha^\ddagger_k < 0$ & $\alpha^\ddagger_k > 0$ \\
\midrule
\makecell{Weight update \\ (unnormalized)} 
& \makecell{$\beta_{k}(\mathbf{z}^{(n)};e) = +1$: $w_k^{(n)} > w_{k-1}^{(n)}$ \\ $\beta_{k}(\mathbf{z}^{(n)};e) = -1$: $w_k^{(n)} < w_{k-1}^{(n)}$} 
& \makecell{ $\beta_{k}(\mathbf{z}^{(n)};e) = +1$: $w_k^{(n)} < w_{k-1}^{(n)}$ \\  $\beta_{k}(\mathbf{z}^{(n)};e) = -1$: $w_k^{(n)} > w_{k-1}^{(n)}$} \\
\bottomrule[1.2pt]
\end{tabular}
}
\end{table}

As shown in Table~\ref{tab:LR_results_tau}, the experimental results are obtained under different phase thresholds $\tau$, and an overly large or overly small $\tau$ can cause the LR model to learn a poor hyperplane, which is consistent with the discussion above. Therefore, we recommend $\tau = \frac{1}{2}$, under which APW$^\ddagger$ reduces to APW and exhibits more stable performance. Moreover, our theoretical analysis is established under $\tau = \frac{1}{2}$, further supporting the validity of this choice.

\begin{table}[htbp]
\centering
\caption{Effect of the phase threshold $\tau$ in APW$^\ddagger$ on LR. For the first column (training set), the marker definitions follow those in Fig.~\ref{fig:LR_results_detail}; for the second column (test set), the marker definitions follow those in Fig.~\ref{fig:LR_results}. For each plot, the region between the black dashed lines represents the confidence band, and correctly classified samples that lie outside this confidence band are marked as easy samples.}\label{tab:LR_results_tau}
\resizebox{0.57\linewidth}{!}{%
\begin{threeparttable} 
\begin{tabular}{
c|
>{\centering\arraybackslash}m{0.32\linewidth}|
>{\centering\arraybackslash}m{0.32\linewidth}
}
\toprule[1pt]
\multicolumn{1}{c}{} & Training set & Test set \\
\midrule[1pt]
\makecell[c]{\rotatebox{90}{{$\tau=0.1$}}}
& \begin{minipage}[b]{\linewidth}
\centering
\includegraphics[width=\linewidth]{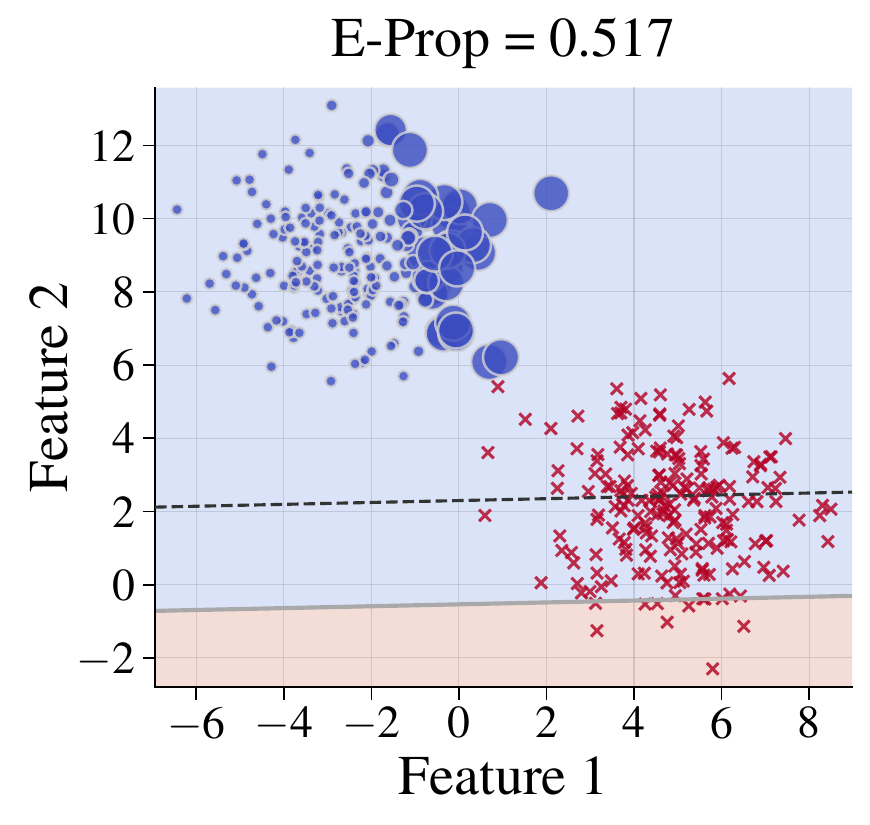}
\end{minipage}
& \begin{minipage}[b]{\linewidth}
\centering
\includegraphics[width=\linewidth]{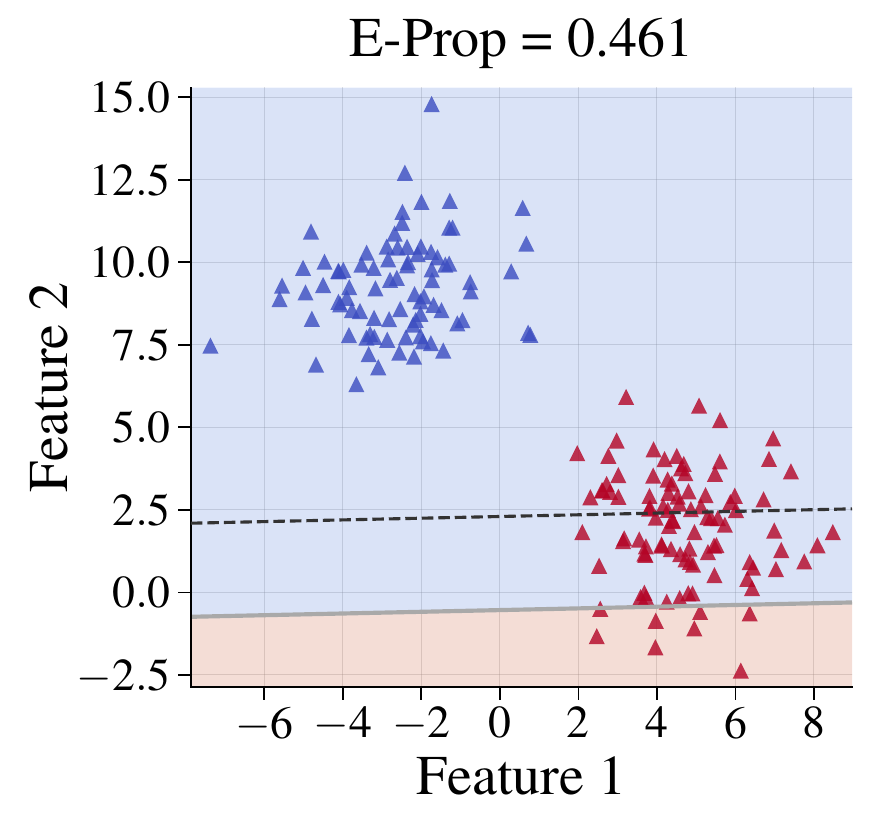}
\end{minipage}
\\
\midrule[0.7pt]
\makecell[c]{\rotatebox{90}{{$\tau=0.3$}}}
& \begin{minipage}[b]{\linewidth}
\centering
\includegraphics[width=\linewidth]{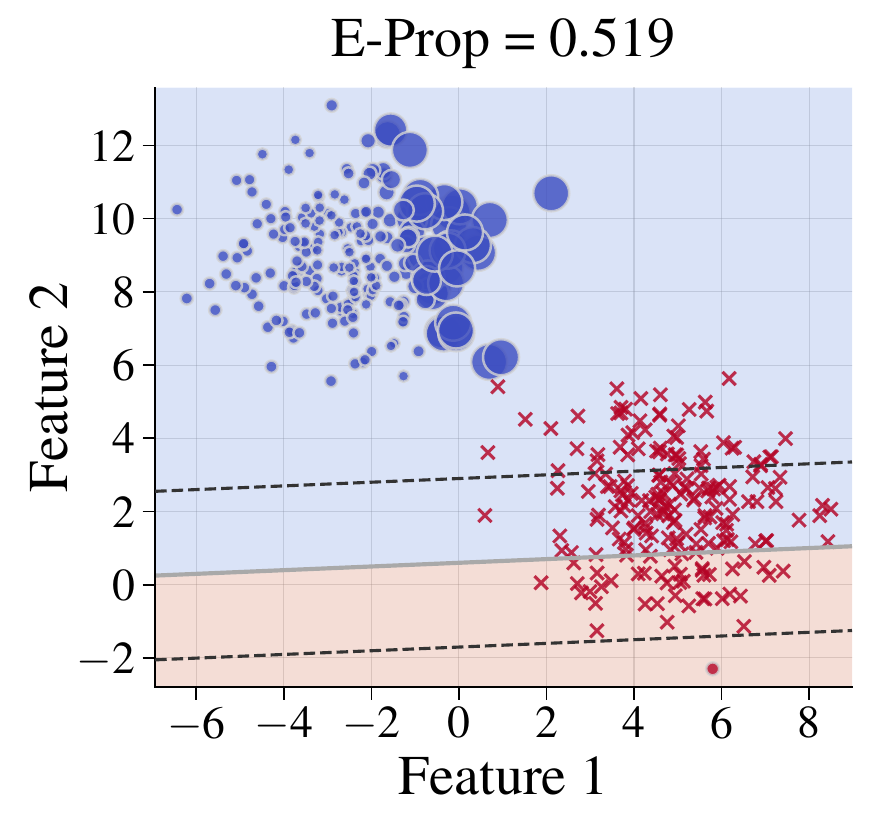}
\end{minipage}
& \begin{minipage}[b]{\linewidth}
\centering
\includegraphics[width=\linewidth]{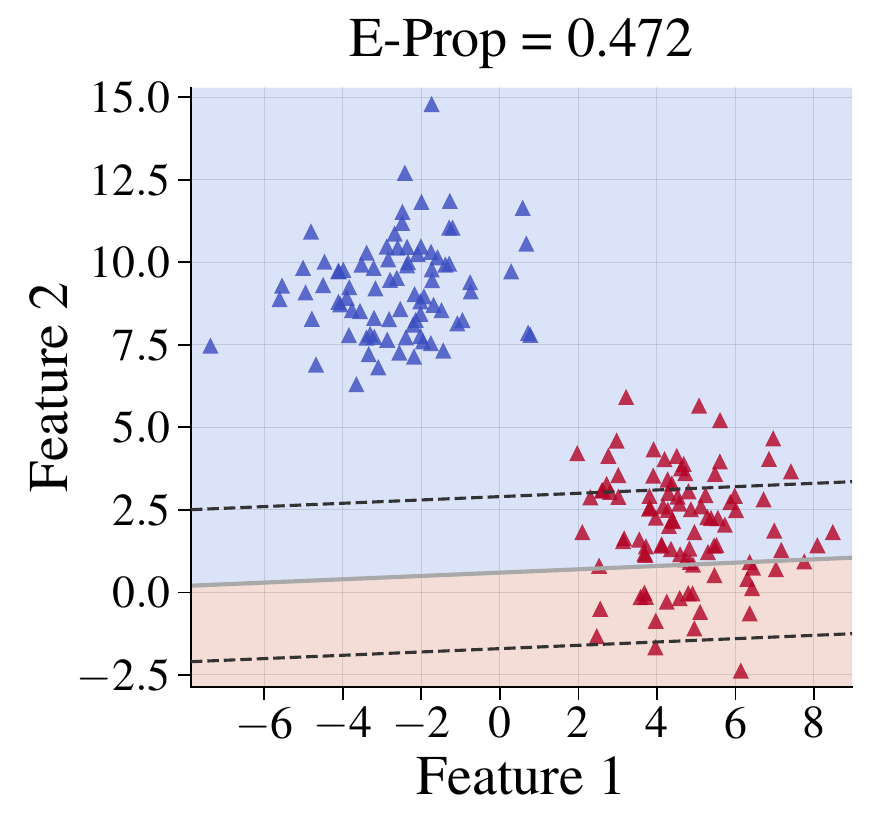}
\end{minipage}
\\
\midrule[0.7pt]
\makecell[c]{\rotatebox{90}{{$\tau=0.5$ (default)}}}
& \begin{minipage}[b]{\linewidth}
\centering
\includegraphics[width=\linewidth]{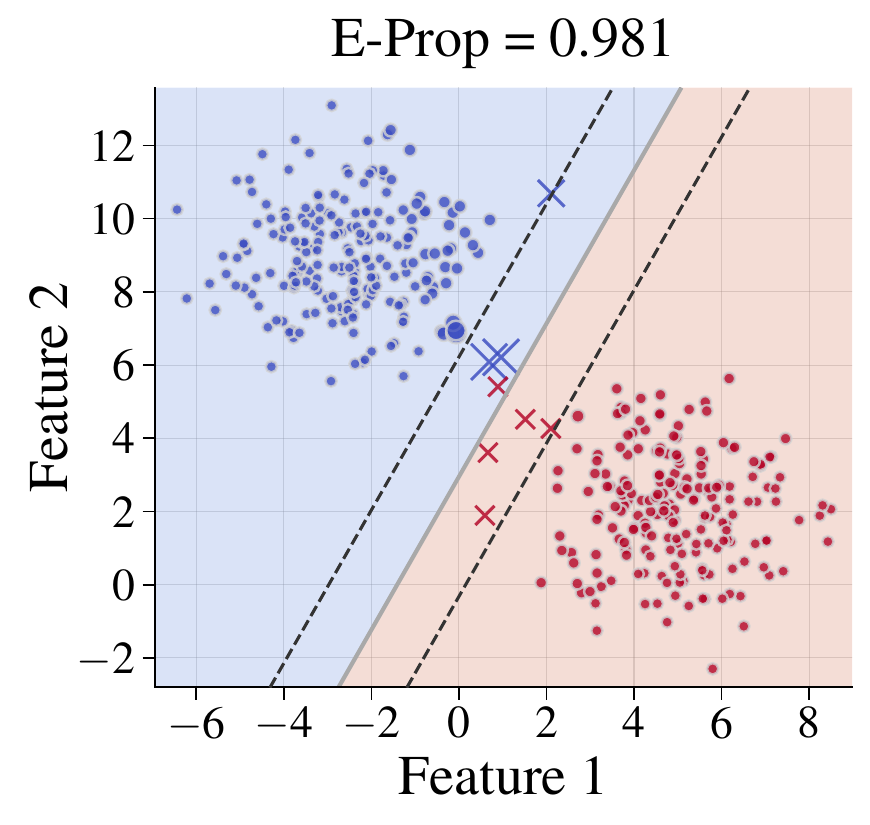}
\end{minipage}
& \begin{minipage}[b]{\linewidth}
\centering
\includegraphics[width=\linewidth]{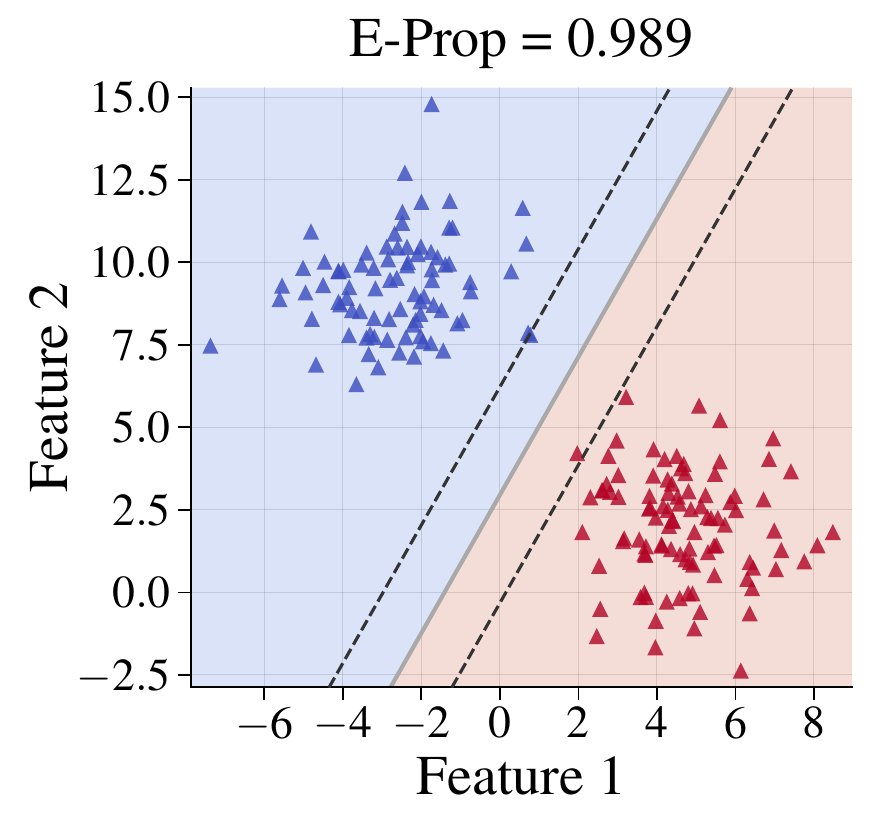}
\end{minipage}
\\
\midrule[0.7pt]
\makecell[c]{\rotatebox{90}{{$\tau=0.7$}}}
& \begin{minipage}[b]{\linewidth}
\centering
\includegraphics[width=\linewidth]{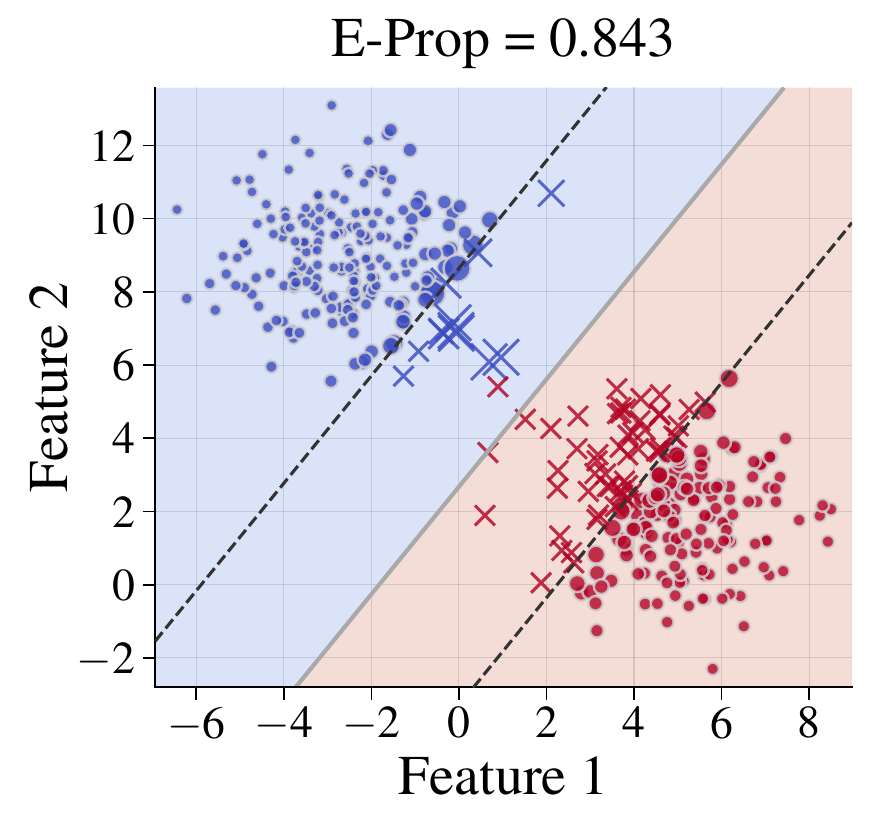}
\end{minipage}
& \begin{minipage}[b]{\linewidth}
\centering
\includegraphics[width=\linewidth]{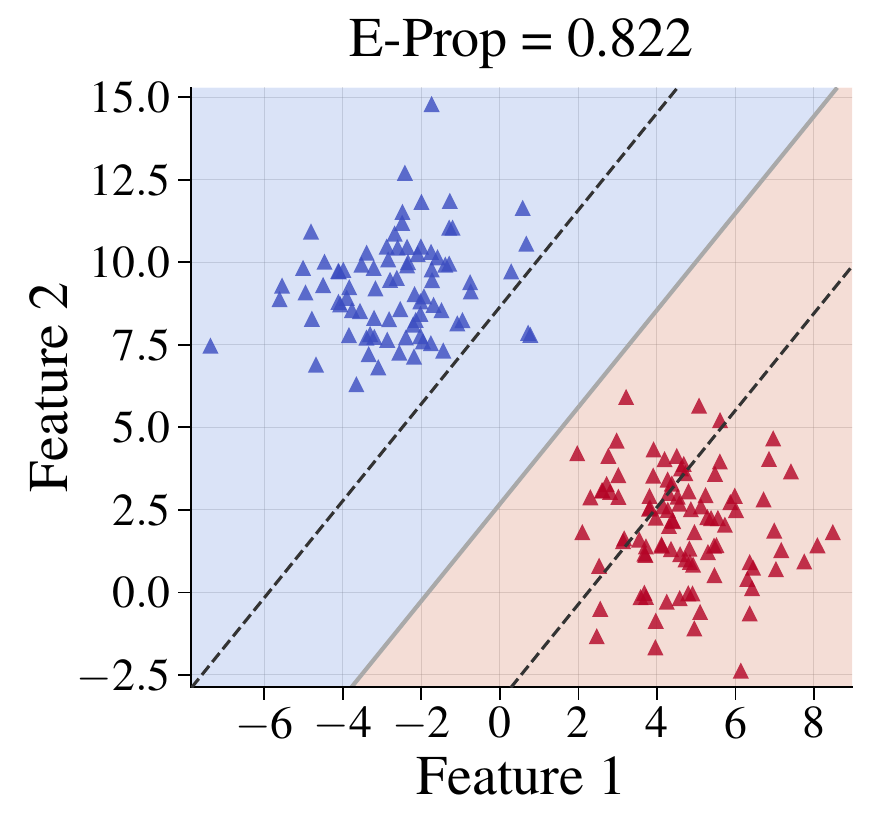}
\end{minipage}
\\
\midrule[0.7pt]
\makecell[c]{\rotatebox{90}{{$\tau=0.9$}}}
& \begin{minipage}[b]{\linewidth}
\centering
\includegraphics[width=\linewidth]{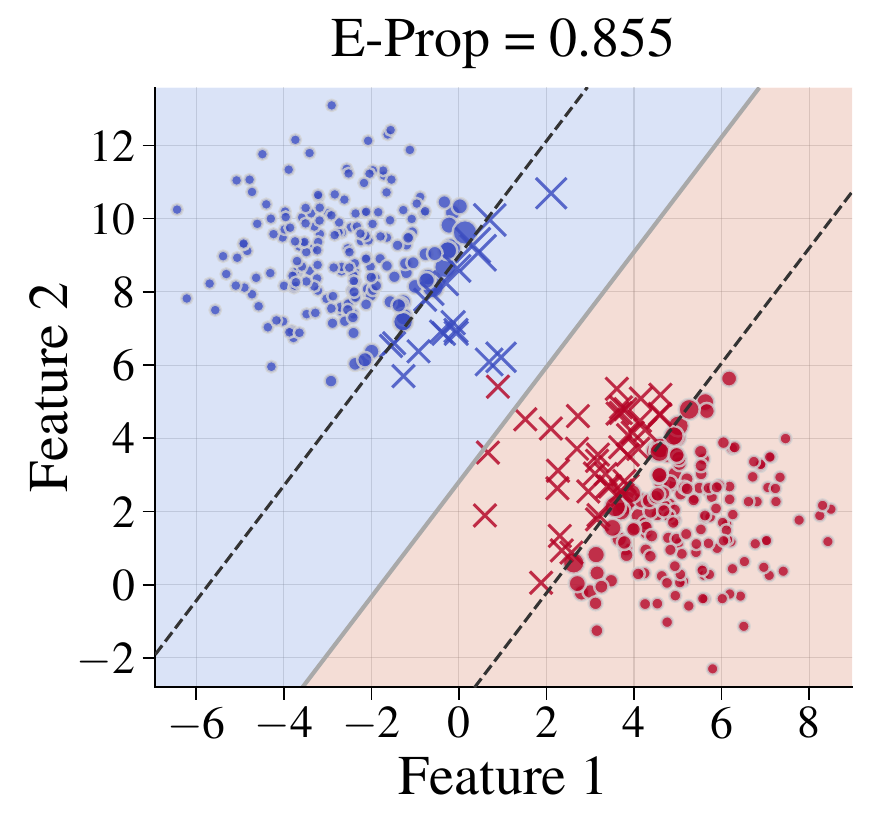}
\end{minipage}
& \begin{minipage}[b]{\linewidth}
\centering
\includegraphics[width=\linewidth]{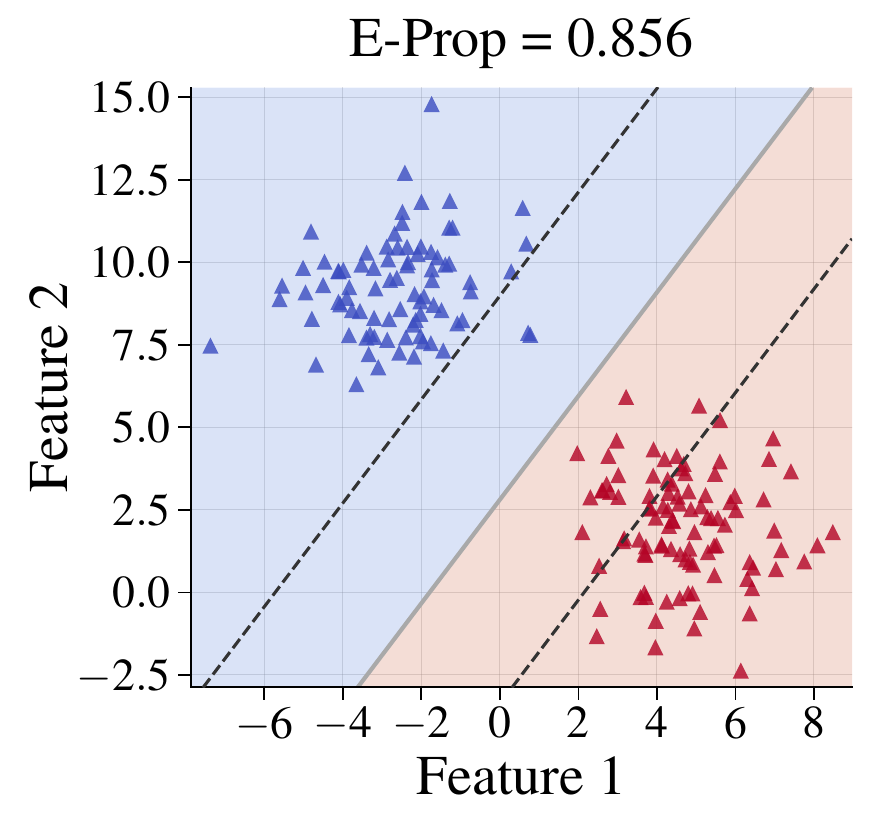}
\end{minipage}
\\
\bottomrule[1pt]
\end{tabular}
\end{threeparttable}
}
\end{table}

\clearpage

\section{Proofs of Main Results}\label{appendix:proof}

In this appendix, we prove the theoretical results of this paper, including Lemma~\ref{lemma:ineq:m_K}, Theorem~\ref{theorem:rho_k}, Theorem~\ref{theorem:training_effectiveness}, Theorem \ref{theorem:training_stability}, Lemma~\ref{lemma:theta} and Theorem~\ref{theorem:generalization_error}. 

\subsection{Proof of Lemma~\ref{lemma:ineq:m_K}}\label{proof:inequation}	
\begin{proof}
For any $a,b,c,d>0$, if $b-d>0$, it holds that
\begin{equation*}
\frac{a}{b}>\frac{c}{d}>0 \implies \frac{a-c}{b-d}>\frac{c}{d}>0.
\end{equation*}
Therefore, by setting $a := g_{K+\Delta_K}(\mathbf{z}^{(n)})$, $b := A_{K+\Delta_K}$, $c := g_K(\mathbf{z}^{(n)})$ and $d := A_K$, since $b-d=\sum_{k=K+1}^{K+\Delta_K}\alpha_k>0$ of the $\Delta_K$-convergence phase, we have
\begin{equation*}
\frac{g_{K+\Delta_K}(\mathbf{z}^{(n)})}{A_{K+\Delta_K}}>\frac{g_K(\mathbf{z}^{(n)})}{A_K}>0 \implies \frac{g_{K+\Delta_K}(\mathbf{z}^{(n)})-g_{K}(\mathbf{z}^{(n)})}
{A_{K+\Delta_K}-A_{K}}=\frac{\sum_{k=K+1}^{K+\Delta_K}\alpha_k\,\beta_k(\mathbf{z}^{(n)};e)}
{\sum_{k=K+1}^{K+\Delta_K}\alpha_k}>\frac{g_K(\mathbf{z}^{(n)})}{A_K}>0.
\end{equation*}
This completes the proof.
\end{proof}

\subsection{Proof of Theorem~\ref{theorem:rho_k}}\label{proof:theorem_rho_k}	

\begin{proof}
According to Eq.~\eqref{eq:APW_loss}, we have
\begin{equation*}
\mathrm{L}_{k}^{\mathrm{APW}} \geq \sum_{\{n: \beta_{k}(\mathbf{z}^{(n)}; e) = -1\}} w_{k}^{(n)} \mathrm{L}_{k}^{(n)} 
> e \sum_{\{n: \beta_{k}(\mathbf{z}^{(n)}; e) = -1\}} w_{k}^{(n)}.
\end{equation*}
Thus,
\begin{equation}\label{ineq:w_k}
\sum_{\{n: \beta_k(\mathbf{z}^{(n)}; e) = -1\}} w_k^{(n)} < \frac{1}{e} \mathrm{L}_k^{\mathrm{APW}}.
\end{equation}
If the hard samples in the $k$-th epoch are still the hard ones in the $(k+1)$-th epoch, it holds that
\begin{equation*}
\beta_k(\mathbf{z}^{(n)}; e) = -1 \iff \beta_{k+1}(\mathbf{z}^{(n)}; e) = -1, 
\end{equation*}
and then 
\begin{equation*}
\sum_{\{n: \beta_k(\mathbf{z}^{(n)}; e) = -1\}} w_k^{(n)} 
= \sum_{\{n: \beta_{k+1}(\mathbf{z}^{(n)}; e) = -1\}} w_k^{(n)} = \rho_{k+1}.
\end{equation*}
It follows from Eq.~\eqref{ineq:w_k} that
\begin{equation*}
\rho_{k+1} < \frac{1}{e} \mathrm{L}_k^{\mathrm{APW}}.
\end{equation*}
This completes the proof.
\end{proof}

\subsection{Proof of Theorem~\ref{theorem:training_effectiveness}}\label{proof:theorem_training_effectiveness}	

\begin{proof}
The proof is divided into two stages:

(1) In the first stage, we would like to prove
\begin{equation}\label{ineq:I}
\frac{1}{N}\sum_{n=1}^N \mathbf{I}\left(g_K(\mathbf{z}^{(n)}) \leq 0\right) \leq \frac{1}{N}\sum_{n=1}^N \exp\left(-g_K(\mathbf{z}^{(n)})\right) = \prod_{k=1}^K Z_k.
\end{equation}
In fact,  when $g_K(\mathbf{z}^{(n)}) > 0$,
\begin{equation*}
\mathbf{I}\left(g_K(\mathbf{z}^{(n)}) \leq 0\right) = 0 \;\text{and}\; \exp\left(-g_K(\mathbf{z}^{(n)})\right) > 0 \implies \mathbf{I}\left(g_K(\mathbf{z}^{(n)}) \leq 0\right) < \exp\left(-g_K(\mathbf{z}^{(n)})\right).
\end{equation*}
When $g_K(\mathbf{z}^{(n)}) \leq 0$,
\begin{equation*}
\mathbf{I}\left(g_K(\mathbf{z}^{(n)}) \leq 0\right) = 1 \;\text{and}\; \exp\left(-g_K(\mathbf{z}^{(n)})\right) \geq 1\implies  \mathbf{I}\left(g_K(\mathbf{z}^{(n)}) \leq 0\right) \leq \exp\left(-g_K(\mathbf{z}^{(n)})\right).
\end{equation*}
To sum up, we have obtained the first inequality of Eq.~\eqref{ineq:I}.
Next, according to Eq.~\eqref{eq:g_K}, we have
\begin{equation*}
\begin{aligned}
\frac{1}{N}\sum_{n=1}^N \exp\left(-g_K(\mathbf{z}^{(n)})\right) 
&=\frac{1}{N}\sum_{n=1}^N \exp\left(-\sum_{k=1}^K \alpha_k \beta_k\left(\mathbf{z}^{(n)}; e\right)\right)\\
&=\sum_{n=1}^N \frac{1}{N} \prod_{k=1}^K \exp\left(-\alpha_k \beta_k\left(\mathbf{z}^{(n)}; e\right)\right)\\
&= \sum_{n=1}^N w_{0}^{(n)} \prod_{k=1}^K \exp\left(-\alpha_k \beta_k\left(\mathbf{z}^{(n)}; e\right)\right) \quad (\text{since}\; w_{0}^{(n)} = \frac{1}{N})\\
&= \sum_{n=1}^N w_{1}^{(n)} Z_1 \prod_{k=2}^K \exp\left(-\alpha_k \beta_k\left(\mathbf{z}^{(n)}; e\right)\right)  \quad (\text{according to Eq.~\eqref{eq:weight_update}})\\
&= Z_1 Z_2 \cdots Z_{K-1} \sum_{n=1}^N w_{K-1}^{(n)} \exp\left(-\alpha_K \beta_K\left(\mathbf{z}^{(n)}; e\right)\right) \\
&= \prod_{k=1}^K Z_k \left(\sum_{n=1}^N w_{K}^{(n)}\right)\\
&= \prod_{k=1}^K Z_k. \quad (\text{since}\; \sum_{n=1}^N w_{k}^{(n)} = 1,\; \forall k\in \mathbb{N})
\end{aligned}
\end{equation*}
Until now, the equality of Eq.~\eqref{ineq:I} is achieved. It completes the first stage of the proof.

(2) In the second stage, we set $q \geq 2$ and define $\gamma_k = \frac{1}{2} - \rho_k$. We aim to prove
\begin{equation}\label{eq:pai_z_k_a}
\prod_{k=1}^K Z_k \leq \exp\left(-\frac{4}{q} \sum_{k=1}^K \gamma_k^2\right).
\end{equation}
At first, we rewrite $Z_k$ as:
\begin{align}\label{eq:Z_k}
Z_k &= \sum_{n=1}^{N} w_{k-1}^{(n)} \exp\big(-\alpha_{k} \beta_k(\mathbf{z}^{(n)};e)\big) \nonumber\\
&= \sum_{\{n: \beta_k(\mathbf{z}^{(n)}; e)=1\}} w_{k-1}^{(n)} \exp\big(-\alpha_{k} \beta_k(\mathbf{z}^{(n)};e)\big) + \sum_{\{n: \beta_k(\mathbf{z}^{(n)}; e)=-1\}} w_{k-1}^{(n)} \exp\big(-\alpha_{k} \beta_k(\mathbf{z}^{(n)};e)\big)  \nonumber\\
&= (1-\rho_{k}) \exp(-\alpha_{k}) + \rho_{k}\exp(\alpha_{k}). \quad \text{(since $\rho_{k}  = \sum_{\{n: \beta_k(\mathbf{z}^{(n)}; e)=-1\}}w_{k-1}^{(n)}$ and $\sum_{n=1}^Nw_{k-1}^{(n)}=1$)} 
\end{align}
Subsequently, we consider two cases that are dependent on the value of $q$.
\begin{itemize}
\item {\bf Case 1 ($q = 2$):} Setting $\alpha_k = \frac{1}{2} \log\left(\frac{1 - \rho_k}{\rho_k}\right)$, according to Eq.~\eqref{eq:Z_k}, we have
\begin{equation*}
Z_k = 2\sqrt{\rho_k (1 - \rho_k)} = 2\sqrt{\frac{1}{4} - \gamma_k^2} = \sqrt{1 - 4\gamma_k^2}. \quad (\text{since} \; \gamma_k = \frac{1}{2} - \rho_k)
\end{equation*}
If $0 < x < 1$, it is true that $\ln(1 - x) \leq -x$, and then
\begin{equation*}
\exp\left(\frac{1}{2} \ln(1 - x)\right) \leq \exp\left(-\frac{x}{2}\right)\implies\sqrt{1 - x} \leq \exp\left(-\frac{x}{2}\right).
\end{equation*}
Since $\rho_k\in(0,1)$ and $\gamma_k = \frac{1}{2} - \rho_k$, we have $4\gamma_k^2 \in(0,1)$. Letting $x = 4\gamma_k^2$, it follows from the fact $\sqrt{1 - x} \leq \exp\left(-\frac{x}{2}\right)$ that
\begin{equation}\label{ineq:gamma_k}
\sqrt{1 - 4\gamma_k^2} \leq \exp(-2\gamma_k^2). 
\end{equation}
Therefore, we have 
\begin{equation*}
\prod_{k=1}^K Z_k = \prod_{k=1}^K \sqrt{1 - 4\gamma_k^2} \leq \prod_{k=1}^K \exp(-2\gamma_k^2) = \exp\left(-2 \sum_{k=1}^K \gamma_k^2\right).
\end{equation*}
It achieves the result Eq.~\eqref{eq:pai_z_k_a} in the case of $q = 2$.
\item {\bf Case 2 ($q > 2$):} When $q > 2$, we set $\alpha_k = \frac{1}{q} \log\left(\frac{1 - \rho_k}{\rho_k}\right)$. According to Eq.~\eqref{eq:Z_k}, we have
\begin{equation*}
\begin{aligned}
Z_k &= (1 - \rho_k) \exp(-\alpha_k) + \rho_k \exp(\alpha_k) \\
&= (1 - \rho_k)^{1 - \frac{1}{q}} \rho_k^{\frac{1}{q}} + \rho_k^{1 - \frac{1}{q}} (1 - \rho_k)^{\frac{1}{q}} \\
&= (1 - \rho_k)^{\frac{1}{q}} \rho_k^{\frac{1}{q}} \left[(1 - \rho_k)^{1 - \frac{2}{q}} + \rho_k^{1 - \frac{2}{q}}\right].
\end{aligned}
\end{equation*}
Define a function
\begin{equation*}
f(x) := (1 - x)^{\frac{1}{q}} x^{\frac{1}{q}} \left[(1 - x)^{1 - \frac{2}{q}} + x^{1 - \frac{2}{q}}\right],
\end{equation*}
such that
\begin{equation*}
Z_k = f(\rho_k).
\end{equation*}
Taking the natural logarithm on $f(x)$, we define
\begin{align*}
&\phi(x) := \ln(f(x)) = \frac{1}{q} \ln((1 - x)x) + \ln\left((1 - x)^{1 - \frac{2}{q}} + x^{1 - \frac{2}{q}}\right),\\
&\phi_1(x) := \frac{1}{q} \ln((1 - x)x),\\
&\phi_2(x) := \ln\left((1 - x)^{1 - \frac{2}{q}} + x^{1 - \frac{2}{q}}\right).
\end{align*}
Thus, we have $f(x) = \exp(\phi(x)) = \exp(\phi_1(x)+\phi_2(x))$ and then, according to $Z_k = f(\rho_k)$,
\begin{equation*}
Z_k = \exp(\phi(\rho_k)) = \exp(\phi_1(\rho_k)) \cdot \exp(\phi_2(\rho_k))  \implies \prod_{k=1}^K Z_k = \prod_{k=1}^K \exp(\phi_1(\rho_k)) \cdot \prod_{k=1}^K \exp(\phi_2(\rho_k)).
\end{equation*}
By combining Eq.~\eqref{ineq:gamma_k} and $\gamma_k = \frac{1}{2}-\rho_k$, we arrive at
\begin{equation}\label{ineq:exp_a}
\begin{aligned}
\prod_{k=1}^K \exp(\phi_1(\rho_k)) &= \prod_{k=1}^K \sqrt[q]{(1 - \rho_k)\rho_k} \\
&= \prod_{k=1}^K \sqrt[q]{\frac{1}{4}(1 - 4\gamma_k^2)}\\
&\leq \prod_{k=1}^K \sqrt[q]{\frac{1}{4} \exp(-4\gamma_k^2)}\\
&= \sqrt[q]{\frac{1}{4^K} \exp\left(-4 \sum_{k=1}^K \gamma_k^2\right)} \\
&= 4^{-\frac{K}{q}} \exp\left(-\frac{4}{q} \sum_{k=1}^K \gamma_k^2\right).
\end{aligned}
\end{equation}
Moreover, consider composite function $\phi_2(x)=\ln\left((1-x)^{1-\frac2q}+x^{1-\frac2q}\right)$ with $x \in (0,1)$. Since, $(1-x)^{1-\frac2q}+x^{1-\frac2q}$ is a concave function ($x \in (0,1)$ and $q>2$) and $\ln(x)$ is increasing and concave ($x \in (0,\infty)$), we have $\phi_2(x)$ is a concave function with $x \in (0,1)$. Additionally, $\phi_2(x)$ is symmetric about $x = \frac{1}{2}$.
Therefore, we have
\begin{equation*}
\phi_2(x) \leq \phi_2\left(\frac{1}{2}\right) = \frac{2 \ln(2)}{q},
\end{equation*}
which implies that
\begin{equation}\label{ineq:exp_b}
\prod_{k=1}^K \exp(\phi_2(\rho_k)) \leq \exp\left(\frac{2K \ln(2)}{q}\right) = 4^{\frac{K}{q}}.
\end{equation}
The combination of Eq.~\eqref{ineq:exp_a} and Eq.~\eqref{ineq:exp_b} leads to
\begin{equation*}
\begin{aligned}
\prod_{k=1}^K Z_k &= \prod_{k=1}^K \exp(\phi(\rho_k)) \\
&= \prod_{k=1}^K \exp(\phi_1(\rho_k)) \cdot \prod_{k=1}^K \exp(\phi_2(\rho_k)) \\
&\leq 4^{-\frac{K}{q}} \exp\left(-\frac{4}{q} \sum_{k=1}^K \gamma_k^2\right) \cdot 4^{\frac{K}{q}} \\
&= \exp\left(-\frac{4}{q} \sum_{k=1}^K \gamma_k^2\right).
\end{aligned}
\end{equation*}
It achieves the result \eqref{eq:pai_z_k_a} in the case of $q > 2$.
\end{itemize}
The combination of the results in the two cases leads to
\begin{equation*}
\prod_{k=1}^K Z_k \leq \exp\left(-\frac{4}{q} \sum_{k=1}^K \gamma_k^2\right), \quad \forall \; q \geq 2.
\end{equation*}
This completes the proof.
\end{proof}

\subsection{Proof of Theorem~\ref{theorem:training_stability}}\label{proof:theorem_training_stability}	

\begin{proof}
The proof is divided into two stages:

\textbf{(1)} Let $ q \geq 2 $ and assume $ A_K > 0 $. In the first stage, we aim to prove:
\begin{equation}\label{ineq:P_z_desire}
\mathbb{P}_{\mathbf{z} \sim \mathcal{Z}}\left[ m_K(\mathbf{z}) \leq \theta \right] \leq 4^{\frac{K}{q}} \prod_{k=1}^K \rho_k^{\frac{1 - \theta}{q}} (1 - \rho_k)^{\frac{1 + \theta}{q}}.
\end{equation}
Given
\begin{equation*}
m_K(\mathbf{z}^{(n)}) = \frac{g_K(\mathbf{z}^{(n)})}{A_K} = \frac{\sum_{k=1}^K \alpha_k \beta_k(\mathbf{z}^{(n)}; e)}{\sum_{k=1}^K \alpha_k},
\end{equation*}
under the conditions $ m(\mathbf{z}^{(n)}) \leq \theta $ and $A_K =  \sum_{k=1}^K \alpha_k > 0 $, we have:
\begin{equation*}
\sum_{k=1}^K \alpha_k \beta_k(\mathbf{z}^{(n)}; e) \leq \theta \sum_{k=1}^K \alpha_k.
\end{equation*}
Therefore,
\begin{equation*}
\exp\left( -\sum_{k=1}^K \alpha_k \beta_k(\mathbf{z}^{(n)}; e) + \theta \sum_{k=1}^K \alpha_k \right) \geq 1.
\end{equation*}
This implies
\begin{equation}\label{ineq:probability_exp}
\mathbf{I}\{ m_K(\mathbf{z}^{(n)}) \leq \theta \} \leq \exp\left( -\sum_{k=1}^K \alpha_k \beta_k(\mathbf{z}^{(n)}; e) + \theta \sum_{k=1}^K \alpha_k \right),
\end{equation}
where $\mathbf{I}(\cdot)$ is the indicator function. We have
\begin{equation*}
\begin{aligned}
\mathbb{E}_{\mathbf{z} \sim \mathcal{Z}} \left[ \mathbf{I}\{ m_K(\mathbf{z}) \leq \theta \} \right] &= 1 \cdot \mathbb{P}_{\mathbf{z} \sim \mathcal{Z}}\left[ m_K(\mathbf{z}) \leq \theta \right] + 0 \cdot \mathbb{P}_{\mathbf{z} \sim \mathcal{Z}}\left[ m_K(\mathbf{z}) > \theta \right] \\
&= \mathbb{P}_{\mathbf{z} \sim \mathcal{Z}}\left[ m_K(\mathbf{z}) \leq \theta \right],
\end{aligned}
\end{equation*}
and then
\begin{equation*}
\begin{aligned}
\mathbb{P}_{\mathbf{z} \sim \mathcal{Z}}\left[ m_K(\mathbf{z}) \leq \theta \right] &= \mathbb{E}_{\mathbf{z} \sim \mathcal{Z}} \left[ \mathbf{I}\{ m_K(\mathbf{z}) \leq \theta \} \right] \\
&\leq \mathbb{E}_{\mathbf{z} \sim \mathcal{Z}} \left[ \exp\left( -\sum_{k=1}^K \alpha_k \beta_k(\mathbf{z}; e) + \theta \sum_{k=1}^K \alpha_k \right) \right].  \quad (\text{according to} \; \text{Eq.~\eqref{ineq:probability_exp}}) 
\end{aligned}
\end{equation*}
Considering that the $ N $ samples are independently and uniformly sampled from $ \mathcal{Z} $, we have
\begin{equation*}
\mathbb{E}_{\mathbf{z} \sim \mathcal{Z}}[\cdot] = \frac{1}{N} \sum_{n=1}^N [\cdot],
\end{equation*}
and then
\begin{equation*}
\begin{aligned}
\mathbb{P}_{\mathbf{z} \sim \mathcal{Z}}\left[ m_K(\mathbf{z}) \leq \theta \right] &\leq \mathbb{E}_{\mathbf{z} \sim \mathcal{Z}} \left[ \exp\left( -\sum_{k=1}^K \alpha_k \beta_k(\mathbf{z}; e) + \theta \sum_{k=1}^K \alpha_k \right) \right]\\
&= \frac{1}{N} \sum_{n=1}^N \exp\left( -\sum_{k=1}^K \alpha_k \beta_k(\mathbf{z}^{(n)}; e) + \theta \sum_{k=1}^K \alpha_k \right) \\
&= \frac{1}{N} \sum_{n=1}^N \left[ \exp\left( -\sum_{k=1}^K \alpha_k \beta_k(\mathbf{z}^{(n)}; e) \right) \cdot \exp\left( \theta \sum_{k=1}^K \alpha_k \right) \right] \\
&= \exp\left( \theta \sum_{k=1}^K \alpha_k \right) \cdot \left( \frac{1}{N} \sum_{n=1}^N \exp\left( -\sum_{k=1}^K \alpha_k \beta_k(\mathbf{z}^{(n)}; e) \right) \right),
\end{aligned}
\end{equation*}
where
\begin{equation*}
\frac{1}{N} \sum_{n=1}^N \exp\left( -\sum_{k=1}^K \alpha_k \beta_k(\mathbf{z}^{(n)}; e) \right) = \frac{1}{N} \sum_{n=1}^N \exp(-g_K(\mathbf{z}^{(n)})) = \prod_{k=1}^K Z_k, \quad (\text{according to} \; \text{Eq.~\eqref{ineq:I}})
\end{equation*}
therefore
\begin{equation}\label{ineq:P_z}
\mathbb{P}_{\mathbf{z} \sim \mathcal{Z}}\left[ m_K(\mathbf{z}) \leq \theta \right] \leq \exp\left( \theta \sum_{k=1}^K \alpha_k \right) \cdot \prod_{k=1}^K Z_k = \left(\prod_{k=1}^K \exp(\theta \alpha_k) \right) \prod_{k=1}^K Z_k = \prod_{k=1}^K \exp(\theta \alpha_k) Z_k.
\end{equation}
Next, we examine $\exp(\theta \alpha_k) Z_k$ in two cases that are dependent on the value of $q$:
\begin{itemize}
\item \textbf{Case 1 ($q = 2$):} according to Eq.~\eqref{eq:Z_k}, we have 
\begin{align*}
&\alpha_k = \frac{1}{2} \log\left( \frac{1 - \rho_k}{\rho_k} \right),\\
&Z_k = (1-\rho_{k}) \exp(-\alpha_{k}) + \rho_{k}\exp(\alpha_{k}) = 2 \sqrt{ \rho_k (1 - \rho_k) }.
\end{align*}
Therefore, 
\begin{equation*}
\begin{aligned}
\exp(\theta \alpha_k) \cdot Z_k &= \left( \frac{1 - \rho_k}{\rho_k} \right)^{\frac{\theta}{2}} \cdot 2 \sqrt{ \rho_k (1 - \rho_k) } \\
&= 2 \rho_k^{\frac{1}{2}} (1 - \rho_k)^{\frac{1}{2}} \left( \frac{1 - \rho_k}{\rho_k} \right)^{\frac{\theta}{2}} \\
&= 2 \rho_k^{\frac{1 - \theta}{2}} (1 - \rho_k)^{\frac{1 + \theta}{2}}.
\end{aligned}
\end{equation*}
Following this equation, we have
\begin{equation*}
\prod_{k=1}^K \exp(\theta \alpha_k) Z_k = \prod_{k=1}^K \left[ 2 \rho_k^{\frac{1 - \theta}{2}} (1 - \rho_k)^{\frac{1 + \theta}{2}} \right] = 2^K \prod_{k=1}^K \rho_k^{\frac{1 - \theta}{2}} (1 - \rho_k)^{\frac{1 + \theta}{2}},
\end{equation*}
and according to Eq.~\eqref{ineq:P_z}, we have
\begin{align*}
\mathbb{P}_{\mathbf{z} \sim \mathcal{Z}}\left[ m_K(\mathbf{z}) \leq \theta \right] &\leq \prod_{k=1}^K \exp(\theta \alpha_k) Z_k\\
&= 2^K \prod_{k=1}^K \rho_k^{\frac{1 - \theta}{2}} (1 - \rho_k)^{\frac{1 + \theta}{2}}\\
&= 2^K \prod_{k=1}^K \sqrt{ \rho_k^{1 - \theta} (1 - \rho_k)^{1 + \theta} }.
\end{align*}
This achieves the result Eq.~\eqref{ineq:P_z_desire} in the case of $q = 2$.

\item \textbf{Case 2 ($q > 2$):} according to Eq.~\eqref{eq:Z_k}, we have 
\begin{align*}
&\alpha_k = \frac{1}{q} \log\left( \frac{1 - \rho_k}{\rho_k} \right),\\
&Z_k = (1 - \rho_k)^{\frac{1}{q}} \rho_k^{\frac{1}{q}} \left[ (1 - \rho_k)^{1 - \frac{2}{q}} + \rho_k^{1 - \frac{2}{q}} \right].
\end{align*}
Similar to the proof of Theorem \ref{theorem:training_effectiveness}, we define a function
\begin{equation*}
f(x) := (1 - x)^{\frac{1}{q}} x^{\frac{1}{q}} \left[(1 - x)^{1 - \frac{2}{q}} + x^{1 - \frac{2}{q}}\right],
\end{equation*}
such that
\begin{equation*}
Z_k = f(\rho_k).
\end{equation*}
Taking the natural logarithm on $f(x)$, we define
\begin{align*}
&\phi(x) := \ln(f(x)) = \frac{1}{q} \ln((1 - x)x) + \ln\left((1 - x)^{1 - \frac{2}{q}} + x^{1 - \frac{2}{q}}\right),\\
&\phi_1(x) := \frac{1}{q} \ln((1 - x)x),\\
&\phi_2(x) := \ln\left((1 - x)^{1 - \frac{2}{q}} + x^{1 - \frac{2}{q}}\right).
\end{align*}
Thus, according to $f(x) = \exp(\phi(x)) = \exp(\phi_1(x)+\phi_2(x))$ and $Z_k = f(\rho_k)$, we have
\begin{equation*}
Z_k = \exp(\phi(\rho_k)) = \exp(\phi_1(\rho_k)) \cdot \exp(\phi_2(\rho_k)) \implies \prod_{k=1}^K Z_k = \prod_{k=1}^K \exp(\phi_1(\rho_k)) \cdot \prod_{k=1}^K \exp(\phi_2(\rho_k)).
\end{equation*}
Following the proof of Theorem \ref{theorem:training_effectiveness}, we have 
\begin{align*}
&\prod_{k=1}^K \exp(\phi_1(\rho_k)) = \prod_{k=1}^K \sqrt[q]{ (1 - \rho_k) \rho_k },\\
&\prod_{k=1}^K \exp(\phi_2(\rho_k)) \leq 4^{\frac{K}{q}}.
\end{align*}
Therefore,
\begin{equation*}
\prod_{k=1}^K Z_k \leq 4^{\frac{K}{q}} \prod_{k=1}^K \sqrt[q]{ (1 - \rho_k) \rho_k }.
\end{equation*}
According to $\alpha_k = \frac{1}{q} \log\left( \frac{1 - \rho_k}{\rho_k} \right)$, it holds that 
\begin{equation*}
\begin{aligned}
\prod_{k=1}^K \exp(\theta \alpha_k) Z_k &= \prod_{k=1}^K \exp(\theta \alpha_k) \cdot \prod_{k=1}^K Z_k\\
&= \prod_{k=1}^K \left[ \left( \frac{1 - \rho_k}{\rho_k} \right)^{\frac{\theta}{q}} \right] \cdot \prod_{k=1}^K Z_k\\
&\leq 4^{\frac{K}{q}} \prod_{k=1}^K \left[ \left( \frac{1 - \rho_k}{\rho_k} \right)^{\frac{\theta}{q}} \cdot \sqrt[q]{ (1 - \rho_k) \rho_k } \right] \\
&= 4^{\frac{K}{q}} \prod_{k=1}^K \left[ (1 - \rho_k)^{\frac{1 + \theta}{q}} \rho_k^{\frac{1 - \theta}{q}} \right].
\end{aligned}
\end{equation*}
Thus, according to Eq.~\eqref{ineq:P_z}, we have
\begin{equation*}
\mathbb{P}_{\mathbf{z} \sim \mathcal{Z}}\left[ m_K(\mathbf{z}) \leq \theta \right] \leq 4^{\frac{K}{q}} \prod_{k=1}^K \rho_k^{\frac{1 - \theta}{q}} (1 - \rho_k)^{\frac{1 + \theta}{q}}.
\end{equation*}
This achieves the result Eq.~\eqref{ineq:P_z_desire} in the case of $q > 2$. 
\end{itemize}
Combining the results of the two cases, we have proven that
\begin{equation*}
\mathbb{P}_{\mathbf{z} \sim \mathcal{Z}} \left[ m_K(\mathbf{z}) \leq \theta \right] \leq 4^{\frac{K}{q}} \prod_{k=1}^K \rho_k^{\frac{1-\theta}{q}} (1 - \rho_k)^{\frac{1+\theta}{q}} , \quad \forall \; q \geq 2.
\end{equation*}
Until now, we have completed the proof of the first stage.

\textbf{(2)} In the second stage, we aim to prove:
\begin{equation}\label{eq:theta}
4^{\frac{K}{q}} \prod_{k=1}^K \rho_k^{\frac{1-\theta}{q}} (1 - \rho_k)^{\frac{1+\theta}{q}} = \prod_{k=1}^K \delta(\theta,\gamma_k),
\end{equation}
where $\delta(\theta,\gamma_k) =  (1 - 2\gamma_k)^{\frac{1 - \theta}{q}} (1 + 2\gamma_k)^{\frac{1 + \theta}{q}} $. We also consider two cases: $q = 2$ and $q > 2$.

\begin{itemize}
\item \textbf{Case 1 ($q = 2$):} It holds that
\begin{align*}
2^K \prod_{k=1}^K \rho_k^{\frac{1-\theta}{2}} (1 - \rho_k)^{\frac{1+\theta}{2}} &= 2^K \prod_{k=1}^K \sqrt{ \rho_k^{1-\theta} (1 - \rho_k)^{1+\theta} }\\
&= \prod_{k=1}^K \sqrt{ 2^{1-\theta} 2^{1+\theta} \left( \frac{1}{2} - \gamma_k \right)^{1-\theta} \left( \frac{1}{2} + \gamma_k \right)^{1+\theta} } \quad (\text{according to} \; \gamma_k = \frac{1}{2} - \rho_k)\\
&= \prod_{k=1}^K \sqrt{ (1 - 2\gamma_k)^{1-\theta} (1 + 2\gamma_k)^{1+\theta} } \\
&= \prod_{k=1}^K \left( (1 - 2\gamma_k)^{\frac{1-\theta}{2}} (1 + 2\gamma_k)^{\frac{1+\theta}{2}} \right)\\
&= \prod_{k=1}^K \delta(\theta,\gamma_k). \quad (\text{according to} \; q=2\;\text{and}\;\delta(\theta,\gamma_k) =  (1 - 2\gamma_k)^{\frac{1 - \theta}{q}} (1 + 2\gamma_k)^{\frac{1 + \theta}{q}} )
\end{align*}
This achieves the result Eq.~\eqref{eq:theta} in the case of $q =2$.
\item \textbf{Case 2 ($q > 2$):} We have
\begin{align*}
4^{\frac{K}{q}} \prod_{k=1}^K \rho_k^{\frac{1-\theta}{q}} (1 - \rho_k)^{\frac{1+\theta}{q}} &= \prod_{k=1}^K \left( 4^{\frac{1}{q}} \left( \frac{1}{2} - \gamma_k \right)^{\frac{1-\theta}{q}} \left( \frac{1}{2} + \gamma_k \right)^{\frac{1+\theta}{q}} \right)  \quad (\text{according to} \; \gamma_k = \frac{1}{2} - \rho_k)\\
&= \prod_{k=1}^K \left( 2^{\frac{1-\theta}{q}} 2^{\frac{1+\theta}{q}} \left( \frac{1}{2} - \gamma_k \right)^{\frac{1-\theta}{q}} \left( \frac{1}{2} + \gamma_k \right)^{\frac{1+\theta}{q}} \right) \\
&= \prod_{k=1}^K \left( (1 - 2\gamma_k)^{\frac{1-\theta}{q}} (1 + 2\gamma_k)^{\frac{1+\theta}{q}} \right)\\
&= \prod_{k=1}^K \delta(\theta,\gamma_k). \quad (\text{according to} \; q>2\;\text{and}\;\delta(\theta,\gamma_k) =  (1 - 2\gamma_k)^{\frac{1 - \theta}{q}} (1 + 2\gamma_k)^{\frac{1 + \theta}{q}} )
\end{align*}
This achieves the result Eq.~\eqref{eq:theta} in the case of $q > 2$.
\end{itemize}
By combining the results of the two cases, we arrive at
\begin{equation*}
4^{\frac{K}{q}} \prod_{k=1}^K \rho_k^{\frac{1-\theta}{q}} (1 - \rho_k)^{\frac{1+\theta}{q}} = \prod_{k=1}^K \delta(\theta,\gamma_k) , \quad \forall \; q \geq 2.
\end{equation*}
Finally, by combining Eqs.~\eqref{ineq:P_z_desire} and \eqref{eq:theta}, we have
\begin{equation*}
\mathbb{P}_{\mathbf{z} \sim \mathcal{Z}} \left[ m_K(\mathbf{z}) \leq \theta \right] \leq \prod_{k=1}^K \delta(\theta,\gamma_k) , \quad \forall \; q \geq 2.
\end{equation*}
Equivalently,
\begin{equation*}
\mathbb{P}_{\mathbf{z} \sim \mathcal{Z}} \left[ m_K(\mathbf{z}) > \theta \right] \geq 1- \prod_{k=1}^K \delta(\theta,\gamma_k) , \quad \forall \; q \geq 2.
\end{equation*}
This completes the proof.

\end{proof}

\subsection{Proof of Lemma~\ref{lemma:theta}}\label{proof:remark_theta}	
\begin{proof}
Consider the function
\begin{equation*}
\delta(\theta, \gamma_k) = \left(1 - 2\gamma_k\right)^{\frac{1 - \theta}{q}} \left(1 + 2\gamma_k\right)^{\frac{1 + \theta}{q}}.
\end{equation*}
Taking the natural logarithm on both sides, since $0<\theta \leq \gamma_k<\frac{1}{2}$, we obtain
\begin{equation*}
\ln \delta(\theta, \gamma_k) = \frac{1 - \theta}{q} \ln(1 - 2\gamma_k) + \frac{1 + \theta}{q} \ln(1 + 2\gamma_k).
\end{equation*}
Define
\begin{equation*}
\phi(\theta, \gamma_k) := q \ln \delta(\theta, \gamma_k) = (1 - \theta) \ln(1 - 2\gamma_k) + (1 + \theta) \ln(1 + 2\gamma_k).
\end{equation*}
Since $ q \geq 2 $, our objective is to prove the inequality
\begin{equation*}
\delta(\theta, \gamma_k) \leq \delta(\gamma_k, \gamma_k) \leq \delta(\gamma_{\min},\gamma_{\min}) < 1.
\end{equation*}
Next, consider the derivative of $ \phi(\theta, \gamma_k) $ with respect to $ \theta $,
\begin{equation*}
\frac{\partial \phi(\theta, \gamma_k)}{\partial \theta} = -\ln(1 - 2\gamma_k) + \ln(1 + 2\gamma_k) = \ln\left(\frac{1 + 2\gamma_k}{1 - 2\gamma_k}\right) > 0.
\end{equation*}
Therefore, $ \phi(\theta, \gamma_k) $ is monotonically increasing with respect to $ \theta $ and then we have
\begin{equation*}
\phi(\theta, \gamma_k) \leq \phi(\gamma_k, \gamma_k).
\end{equation*}
Furthermore, define the function
\begin{equation*}
\psi(x) = (1 - x) \ln(1 - 2x) + (1 + x) \ln(1 + 2x).
\end{equation*}
It is straightforward to verify that $ \psi(x) $ is a symmetric concave function within the interval $ \left(-\frac{1}{2}, \frac{1}{2}\right)$ and its axis of symmetry is $x=0$. Specifically, for $ x \in \left(0, \frac{1}{2}\right) $, $ \psi(x) $ is monotonically decreasing and satisfies $ \psi(x) < 0 $. Therefore, we have
\begin{equation*}
\phi(\gamma_k, \gamma_k) = \psi(\gamma_k) < 0.
\end{equation*}
Since $ \gamma_{\min} $ is the minimum of all $ \gamma_k $ in the finite number of epochs, we have
\begin{equation*}
\phi(\gamma_k, \gamma_k) = \psi(\gamma_k) \leq \phi(\gamma_{\min}, \gamma_{\min}) = \psi(\gamma_{\min}) < 0.
\end{equation*}
%
%This is because $ \psi(x) $ is monotonically decreasing with $ x \in \left(0, \frac{1}{2}\right) $.

By combining these results, for any $0<\theta \leq \gamma_k<\frac{1}{2}$, we obtain
\begin{equation*}
\phi(\theta, \gamma_k) \leq \phi(\gamma_k, \gamma_k) \leq \phi(\gamma_{\min}, \gamma_{\min}) < 0 \implies q \ln \delta(\theta, \gamma_k) < 0 \implies \delta(\theta, \gamma_k) < 1.
\end{equation*}
It follows from $\exp(x)$ is monotonically increasing w.r.t. $x$ and $ \delta(\theta, \gamma_k) = \exp(\frac{1}{q}\phi(\theta, \gamma_k))$ that 
\begin{equation*}
\delta(\theta, \gamma_k) \leq \delta(\gamma_k, \gamma_k) \leq \delta(\gamma_{\min}, \gamma_{\min}) < 1.
\end{equation*}
This completes the proof.
\end{proof}

\subsection{Proof of Theorem~\ref{theorem:generalization_error}}\label{proof:generalization_error}	

First, let us introduce some useful concepts from \cite{bartlett1998boosting}. Given a class $\mathcal{F}$ of real-valued functions and a training set $\mathcal{S}'$ of size $m$, a function class $\widehat{\mathcal{F}}$ is said to be an $\epsilon$-\textit{sloppy} $\theta$-\textit{cover} of $\mathcal{F}$ w.r.t. $\mathcal{S}'$ if for all $f \in \mathcal{F}$, there exists $\widehat{f} \in \widehat{\mathcal{F}}$ with 
\begin{equation*}
\mathbb{P}_{x \sim \mathcal{S}'}[|\widehat{f}(x) - f(x)| > \theta] < \epsilon,
\end{equation*}
for the given $\theta,\epsilon >0$. Let $\mathcal{N}(\mathcal{F}, \theta, \epsilon, m)$ denote the maximum, over all training sets $\mathcal{S}$ of size $m$, of the size of the smallest $\epsilon$-sloppy $\theta$-cover of $\mathcal{F}$ w.r.t. $\mathcal{S}'$. Here $\mathcal{S}'$ is the projection of $\mathcal{S}$ onto $X$. Under these notations, the following result holds:
\begin{theorem}[Theorem 4 of \cite{bartlett1998boosting}]\label{theorem:bartlett}
Let $\mathcal{F}$ be a class of real-valued functions defined on the instance space $X$. Let $\mathcal{D}$ be a distribution over $X \times \{-1, 1\}$, and let $\mathcal{S}$ be a sample of $m$ examples chosen independently at random according to $\mathcal{D}$. Let $\epsilon > 0$ and let $\theta > 0$. Then the probability over the random choice of the training set $\mathcal{S}$ that there exists any function $f \in \mathcal{F}$ for which
\begin{equation*}
\mathbb{P}_{\mathcal{D}}[yf(x) \leq 0] > \mathbb{P}_{\mathcal{S}}[yf(x) \leq \theta] + \epsilon
\end{equation*}
is at most
\begin{equation*}
2 \mathcal{N}(\mathcal{F}, \theta / 2, \epsilon / 8, 2m) \exp\left(-\epsilon^2 m / 32\right).
\end{equation*}
\end{theorem}
Now, we present the proof of Theorem \ref{theorem:generalization_error}.

\begin{proof}[Proof of Theorem \ref{theorem:generalization_error}]
Let $\mathcal{S}=\{(x_i,y_i)\}_{i=1}^m$ be the i.i.d. sample from $\mathcal{D}$ in Theorem~\ref{theorem:bartlett}, and let $\mathcal{S}'=\{x_i\}_{i=1}^m$ be its projection onto $X$. Then, for any functions $f,\widehat f \in \mathcal{F}$, it holds that
\begin{align*}
&\mathbb{P}_{x \sim \mathcal{S}'}\left[|\widehat{f}(x) - f(x)| > \theta\right] < \epsilon \\
\iff\;&\mathbb{P}_{(x,y) \sim \mathcal{S}}\left[|\widehat{f}(x) - f(x)| > \theta\right] < \epsilon \\
\iff\;&\mathbb{P}_{(x,y) \sim \mathcal{S}}\left[|y|\cdot|\widehat{f}(x) - f(x)| > \theta\right] < \epsilon \\
\iff\;&\mathbb{P}_{(x,y) \sim \mathcal{S}}\left[|y\widehat{f}(x) - yf(x)| > \theta\right] < \epsilon \\
\iff\;&\mathbb{P}_{(x,y) \sim \mathcal{S}}\left[|\widehat{g}(x,y)-g(x,y)| > \theta\right] < \epsilon,
\end{align*}
where $g(x,y)=y f(x)$ and $\widehat g(x,y)=y\widehat f(x)$. Following \cite{bartlett1998boosting}, let $\mathcal{H}$ denote the space from which the base classifiers are chosen, where each $h\in\mathcal{H}$ is a mapping $h:X\to\{-1,+1\}$. Each $f \in \mathcal{F}$ is a voting classifier, {\it i.e.}, $f(x)=\sum_{h\in\mathcal{H}} a_h h(x)$ with $a_h\geq 0$ and $\sum_{h\in\mathcal{H}} a_h = 1$; thus, $\mathcal{F} \subseteq \mathrm{conv}(\mathcal{H})$. For convenience, define $\widetilde h(x,y):=y\,h(x)\in\{-1,+1\}$. Next, define a function class
\begin{equation*}
\mathcal{G} = \left\{ g \mid g(x,y)= yf(x) = \sum_{h \in \mathcal{H}} a_h(yh(x))=\sum_{h \in \mathcal{H}} a_h \widetilde{h}(x,y) \right\},
\end{equation*}
and we have
\begin{equation*}
|\widehat{\mathcal{G}}|= \mathcal{N}(\mathcal{G}, \theta / 2, \epsilon / 8, m)=|\widehat{\mathcal{F}}| = \mathcal{N}(\mathcal{F}, \theta / 2, \epsilon / 8, m),
\end{equation*}
where  $\widehat{\mathcal{G}}$ is an $\epsilon$-sloppy $\theta$-cover of $\mathcal{G}$ w.r.t. $\mathcal{S}$.
Then, the result of Theorem \ref{theorem:bartlett} can be equivalently rewritten as: 
\begin{quote}
\vspace{-0.9cm}
\begin{theorem}\label{theorem:equivalent}
Let $\mathcal{G}$ be a class of real-valued functions defined on the sample space $X \times \{-1, 1\}$. Let $\mathcal{S}$ be a set of $m$ i.i.d. training samples taken from the distribution $\mathcal{D}$ over the sample space. Let $\epsilon > 0$ and let $\theta > 0$. Then the probability over the random choice of the training set $\mathcal{S}$ that there exists any function $g \in \mathcal{G}$ for which
\begin{equation*}
\mathbb{P}_{(x,y) \sim \mathcal{D}}[g(x,y) \leq 0] > \mathbb{P}_{(x,y) \sim \mathcal{S}}[g(x,y) \leq \theta] + \epsilon
\end{equation*}
is at most
\begin{equation*}
2 \mathcal{N}(\mathcal{G}, \theta/2, \epsilon/8, 2m) \exp\left(-\epsilon^2 m / 32\right).
\end{equation*}
\end{theorem}
\end{quote}

Let $\mathcal{W}$ be a pre-specified, data-independent parameter set, and define the base function class
\begin{equation*}
\mathcal{H}_e := \{\, \beta(\cdot;\Theta,e):\mathcal{Z}\to\{-1,1\}\mid \Theta\in\mathcal{W}\,\}.
\end{equation*}
Next, the function class for a $\Delta_K$-convergence phase in the APW method is defined as
\begin{equation*}
\mathcal{B} = \left\{ b_{\Delta_K} \mid b_{\Delta_K}(\mathbf{z}) = \frac{g_{K+\Delta_K}(\mathbf{z})-g_K(\mathbf{z})}{A_{K+\Delta_K}-A_K}\right\}. 
\end{equation*}
By $A_K = \sum_{k=1}^K \alpha_k$ and Eq.~\eqref{eq:g_K}, $\mathcal{B}$ can be equivalently expressed as
\begin{equation*}
\mathcal{B} = \left\{  b_{\Delta_K} \mid b_{\Delta_K}(\mathbf{z}) = \frac{\sum_{k=K+1}^{K+\Delta_K} \alpha_k \beta_k(\mathbf{z}; e)}{\sum_{k=K+1}^{K+\Delta_K} \alpha_k}\right\},
\end{equation*}
where $\alpha_k>0$, $\mathcal{B}\subseteq \mathrm{conv}(\mathcal{H}_e)$, and $\beta_k(\mathbf{z}; e)$ can be written as $\beta(\mathbf{z};\Theta_k,e)$ for some $\Theta_k\in\mathcal{W}$, where $\mathcal{W}$ is fixed before seeing the training set. Although the network parameters $\Theta_k$ are learned from data, Theorem~\ref{theorem:equivalent} provides a uniform bound over $\mathcal{H}_e$, and thus it can be applied to the realized $\Theta_k$.

As shown in Tab.~\ref{tab:connections}, the notations in Theorem \ref{theorem:generalization_error} are quite parallel to those in Theorem \ref{theorem:equivalent}. Consequently, the result of Theorem \ref{theorem:generalization_error} can be directly derived from Theorem \ref{theorem:equivalent}. This completes the proof. 

\begin{table}[htbp]
\centering   
\caption{Correspondence between the notations in Theorem \ref{theorem:equivalent} and those in Theorem \ref{theorem:generalization_error}}\label{tab:connections}
\begin{tabular}{c|c}
\bottomrule
Notations in Theorem \ref{theorem:equivalent} & Notations in Theorem \ref{theorem:generalization_error}\\
\midrule
\makecell[c]{
$\mathcal{G} = \left\{ g \mid g(x,y)=\sum_{h \in \mathcal{H}} a_h \widetilde{h}(x,y) \right\}$ 
}
& \makecell[c]{
$\mathcal{B} = \left\{ b \mid b(\mathbf{z})=\sum_{j}\left[ \frac{\alpha_{j}}{\sum_{k}\alpha_{k}}\beta_{j}(\mathbf{z};e)\right] \right\}$ 
} \\
\midrule
\makecell[c]{
$a_h \geq 0; \, \sum_h a_h = 1$ 
}
& \makecell[c]{
$\frac{\alpha_{j}}{\sum_{k} \alpha_k}>0 ; \, \sum_{j}\frac{\alpha_{j}}{\sum_{k} \alpha_k} = 1$
} \\
\midrule
\makecell[c]{
$\widetilde{h}(x,y)\in \{-1, 1\}$ 
}
& \makecell[c]{
$\beta_{j}(\mathbf{z};e)\in \{-1, 1\}$ 
} \\
\midrule
\makecell[c]{$\widetilde{h}(x,y) = +1 \iff h$ gives the correct prediction \\
$\widetilde{h}(x,y) = -1 \iff h$ gives the wrong prediction
}
& \makecell[c]{
$\beta_{j}(\mathbf{z}; e) = +1 \iff \mathrm{L}_{j}(\mathbf{z}) \leq e$ \\
$\beta_{j}(\mathbf{z}; e) = -1 \iff \mathrm{L}_{j}(\mathbf{z})>e$ 
} \\
\midrule
\makecell[c]{Let $\mathcal{S}$ be a set of $m$ i.i.d. training samples  taken \\ from the distribution $\mathcal{D}$ over the sample space.}
& \makecell{Let $\mathcal{Z}$ be a set of $N$ i.i.d. training samples taken \\ from the distribution $\mathcal{D}$ over the sample space.} \\
\bottomrule
\end{tabular}
\end{table}

\end{proof}

\end{document}